\documentclass{article}

\usepackage[page,toc,titletoc,title]{appendix}

\usepackage[accepted]{icml2023}

\usepackage{amsmath,amsfonts,bm,amssymb,mathtools,amsthm}
\usepackage{color,xcolor,placeins}
\usepackage{thm-restate}

\usepackage{hyperref}       %
\hypersetup{
  colorlinks,
  linkcolor={blue!50!black},
  citecolor={blue!50!black},
  urlcolor={blue!80!black}
}
\definecolor{orange}{HTML}{ff6c0c}
\definecolor{blue}{HTML}{1f77b4}
\definecolor{Gray}{gray}{0.85}
\definecolor{LightCyan}{rgb}{0.88,1,1}

\usepackage{url}
\usepackage{xkcdcolors}
\usepackage{mdframed}

\usepackage{microtype}
\usepackage{graphicx}
\usepackage{subcaption}
\usepackage{latexsym}
\usepackage{sidecap}
\usepackage{caption}
\usepackage{booktabs} %
\usepackage{amsfonts}       %
\usepackage{nicefrac}       %
\usepackage{comment}
\usepackage{enumitem}
\usepackage{wrapfig,tikz}
\usepackage{longtable}
\usepackage{bigstrut, tabularx, multirow, makecell, diagbox}
\usepackage[export]{adjustbox}
\usepackage{float}
\usepackage{stackrel}
\usepackage{color}
\usepackage{colortbl}
\usepackage{theoremref}
\usepackage{cases}
\usepackage{stmaryrd}
\usepackage{mathabx}
\usepackage{dsfont}
\usepackage{arydshln}

\makeatletter
\usepackage{xspace}
\def\@onedot{\ifx\@let@token.\else.\null\fi\xspace}
\DeclareRobustCommand\onedot{\futurelet\@let@token\@onedot}

\definecolor{blue1}{RGB}{0,128,255}
\definecolor{blue3}{RGB}{0,0,128}
\definecolor{darkpastelgreen}{rgb}{0.01, 0.75, 0.24}
\definecolor{cerulean}{rgb}{0.0, 0.48, 0.65}

\newcommand*{\tran}{^{\mkern-1.5mu\mathsf{T}}}

\newcommand{\mbb}[1]{\mathbb{#1}}
\newcommand{\mcal}[1]{\mathcal{#1}}

\def\eg{\emph{e.g}\onedot}

\def\ie{\emph{i.e}\onedot}

\def\cf{\emph{cf}\onedot}

\def\vs{\emph{vs}\onedot}

\definecolor{darkgreen}{rgb}{0,0.6,0}

\newtheorem{theorem}{Theorem}
\newtheorem{definition}{Definition}

\newtheorem{lemma}{Lemma}
\newtheorem{remark}{Remark}

\def\eqref#1{equation~\ref{#1}}

\def\1{\bm{1}}

\def\rvw{{\mathbf{w}}}
\def\rvx{{\mathbf{x}}}
\def\rvy{{\mathbf{y}}}
\def\rvz{{\mathbf{z}}}

\def\vtheta{{\bm{\theta}}}

\def\vphi{{\bm{\phi}}}

\def\ve{{\bm{e}}}
\def\vf{{\bm{f}}}

\def\vh{{\bm{h}}}

\def\vs{{\bm{s}}}

\def\mA{{\bm{A}}}

\def\mG{{\bm{G}}}
\def\mH{{\bm{H}}}
\def\mI{{\bm{I}}}

\def\mQ{{\bm{Q}}}

\DeclareMathAlphabet{\mathsfit}{\encodingdefault}{\sfdefault}{m}{sl}
\SetMathAlphabet{\mathsfit}{bold}{\encodingdefault}{\sfdefault}{bx}{n}

\newcommand{\ud}{\mathop{}\!\mathrm{d}}

\newcommand{\norm}[1]{\left\lVert#1\right\rVert}

\newtheorem{innercustomthm}{Theorem}
\newenvironment{customthm}[1]
{\renewcommand\theinnercustomthm{#1}\innercustomthm}
{\endinnercustomthm}

\usepackage[capitalize]{cleveref}
\usepackage{tcolorbox}

\usepackage[textsize=tiny]{todonotes}

\icmltitlerunning{Consistency Models}

\begin{document}

\twocolumn[
\icmltitle{Consistency Models}

\icmlsetsymbol{equal}{*}

\begin{icmlauthorlist}
\icmlauthor{Yang Song}{a}
\icmlauthor{Prafulla Dhariwal}{a}
\icmlauthor{Mark Chen}{a}
\icmlauthor{Ilya Sutskever}{a}
\end{icmlauthorlist}

\icmlaffiliation{a}{OpenAI, San Francisco, CA 94110, USA}
\icmlcorrespondingauthor{Yang Song}{songyang@openai.com}

\icmlkeywords{generative models, diffusion, consistency models, score-based generative models, score matching}

\vskip 0.3in
]

\printAffiliationsAndNotice{}  %

\begin{abstract}
    Diffusion models have significantly advanced the fields of image, audio, and video generation, but they depend on an iterative sampling process that causes slow generation. To overcome this limitation, we propose \emph{consistency models}, a new family of models that generate high quality samples by directly mapping noise to data. They support fast one-step generation by design, while still allowing multistep sampling to trade compute for sample quality. They also support zero-shot data editing, such as image inpainting, colorization, and super-resolution, without requiring explicit training on these tasks. Consistency models can be trained either by distilling pre-trained diffusion models, or as standalone generative models altogether. Through extensive experiments, we demonstrate that they outperform existing distillation techniques for diffusion models in one- and few-step sampling, achieving the new state-of-the-art FID of 3.55 on CIFAR-10 and 6.20 on ImageNet $64\times 64$ for one-step generation. When trained in isolation, consistency models become a new family of generative models that can outperform existing one-step, non-adversarial generative models on standard benchmarks such as CIFAR-10, ImageNet $64\times 64$ and LSUN $256\times 256$. %
\end{abstract}
\section{Introduction}\label{sec:intro}
\definecolor{model}{rgb}{0.7098,0.2235,0.2549}
\definecolor{ode}{rgb}{0.7216,0.8078,0.5569}

\begin{figure}[t]
    \centering
    \includegraphics[width=\linewidth]{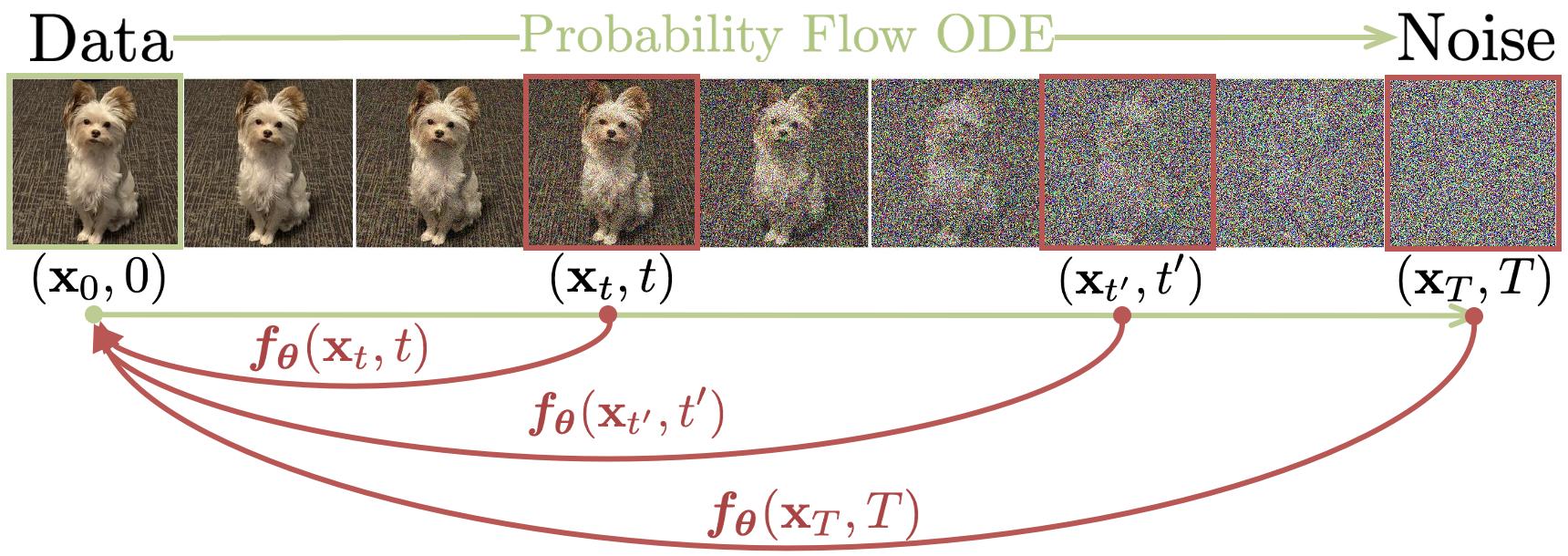}
    \caption{Given a {\color{ode} Probability Flow (PF) ODE} that smoothly converts data to noise, we learn to map any point (\eg, $\rvx_t$, $\rvx_{t'}$, and $\rvx_T$) on the ODE trajectory to its origin (\eg, $\rvx_0$) for generative modeling. Models of these mappings are called {\color{model}consistency models}, as their outputs are trained to be consistent for points on the same trajectory.}
    \label{fig:scheme}
\end{figure}

Diffusion models \cite{sohl2015deep,song2019generative,song2020improved,ho2020denoising,song2021scorebased}, also known as score-based generative models, have achieved unprecedented success across multiple fields, including image generation \cite{dhariwal2021diffusion,nichol2021glide,ramesh2022hierarchical,saharia2022photorealistic,rombach2022high}, audio synthesis \cite{kong2020diffwave,chen2021wavegrad,popov2021grad}, and video generation \cite{ho2022video,ho2022imagen}. %
A key feature of diffusion models is the iterative sampling process which progressively removes noise from random initial vectors. This iterative process provides a flexible trade-off of compute and sample quality, as using extra compute for more iterations usually yields samples of better quality. It is also the crux of many zero-shot data editing capabilities of diffusion models, enabling them to solve challenging inverse problems ranging from image inpainting, colorization, stroke-guided image editing, to Computed Tomography and Magnetic Resonance Imaging \cite{song2019generative,song2021scorebased,song2021medical,song2023pseudoinverseguided,kawar2021snips,kawar2022denoising,chung2023diffusion,meng2021sdedit}. However, compared to single-step generative models like GANs \cite{goodfellow2014generative}, VAEs \cite{kingma2013auto,rezende2014stochastic}, or normalizing flows \cite{dinh2014nice,dinh2016density,kingma2018glow}, the iterative generation procedure of diffusion models typically requires 10--2000 times more compute for sample generation \cite{song2020improved,ho2020denoising,song2021scorebased,zhang2022fast,lu2022dpm}, causing slow inference and limited real-time applications.

Our objective is to create generative models that facilitate efficient, single-step generation without sacrificing important advantages of iterative sampling, such as trading compute for sample quality when necessary, as well as performing zero-shot data editing tasks. As illustrated in \cref{fig:scheme}, we build on top of the probability flow (PF) ordinary differential equation (ODE) in continuous-time diffusion models \cite{song2021scorebased}, whose trajectories smoothly transition the data distribution into a tractable noise distribution. We propose to learn a model that maps any point at any time step to the trajectory's starting point. A notable property of our model is self-consistency: \emph{points on the same trajectory map to the same initial point}. We therefore refer to such models as \textbf{consistency models}. Consistency models allow us to generate data samples (initial points of ODE trajectories, \eg, $\rvx_0$ in \cref{fig:scheme}) by converting random noise vectors (endpoints of ODE trajectories, \eg, $\rvx_T$ in \cref{fig:scheme}) with only one network evaluation. Importantly, by chaining the outputs of consistency models at multiple time steps, we can improve sample quality and perform zero-shot data editing at the cost of more compute, similar to what iterative sampling enables for diffusion models.

To train a consistency model, we offer two methods based on enforcing the self-consistency property. The first method relies on using numerical ODE solvers and a pre-trained diffusion model to generate pairs of adjacent points on a PF ODE trajectory. By minimizing the difference between model outputs for these pairs, we can effectively distill a diffusion model into a consistency model, which allows generating high-quality samples with one network evaluation. By contrast, our second method eliminates the need for a pre-trained diffusion model altogether, allowing us to train a consistency model in isolation. This approach situates consistency models as an independent family of generative models. Importantly, neither approach necessitates adversarial training, and they both place minor constraints on the architecture, allowing the use of flexible neural networks for parameterizing consistency models.

We demonstrate the efficacy of consistency models on several image datasets, including CIFAR-10 \cite{krizhevsky2009learning}, ImageNet $64\times 64$ \cite{deng2009imagenet}, and LSUN $256\times 256$ \cite{yu2015lsun}. Empirically, we observe that as a distillation approach, consistency models outperform existing diffusion distillation methods like progressive distillation \cite{salimans2022progressive} across a variety of datasets in few-step generation: On CIFAR-10, consistency models reach new state-of-the-art FIDs of 3.55 and 2.93 for one-step and two-step generation; on ImageNet $64\times 64$, it achieves record-breaking FIDs of 6.20 and 4.70 with one and two network evaluations respectively. When trained as standalone generative models, consistency models can match or surpass the quality of one-step samples from progressive distillation, despite having no access to pre-trained diffusion models. They are also able to outperform many GANs, and existing non-adversarial, single-step generative models across multiple datasets. Furthermore, we show that consistency models can be used to perform a wide range of zero-shot data editing tasks, including image denoising, interpolation, inpainting, colorization, super-resolution, and stroke-guided image editing (SDEdit, \citet{meng2021sdedit}).

\section{Diffusion Models}\label{sec:diffusion}
Consistency models are heavily inspired by the theory of continuous-time diffusion models \citep{song2021scorebased,karras2022edm}. Diffusion models generate data by progressively perturbing data to noise via Gaussian perturbations, then creating samples from noise via sequential denoising steps. Let $p_\text{data}(\rvx)$ denote the data distribution. Diffusion models start by diffusing $p_\text{data}(\rvx)$ with a stochastic differential equation (SDE) \citep{song2021scorebased}
\begin{align}
    \ud \rvx_t = \bm{\mu}(\rvx_t, t) \ud t + \sigma(t)\ud \rvw_t,\label{eq:sde}
\end{align}
where $t\in[0, T]$, $T>0$ is a fixed constant, $\bm{\mu}(\cdot, \cdot)$ and $\sigma(\cdot)$ are the drift and diffusion coefficients respectively, and $\{\rvw_t\}_{t\in[0,T]}$ denotes the standard Brownian motion. We denote the distribution of $\rvx_t$ as $p_t(\rvx)$ and as a result $p_0(\rvx) \equiv p_\text{data}(\rvx)$. A remarkable property of this SDE is the existence of
an ordinary differential equation (ODE), dubbed the \emph{Probability Flow (PF) ODE} by \citet{song2021scorebased}, whose solution trajectories sampled at $t$ are distributed according to $p_t(\rvx)$:
\begin{align}
    \ud \rvx_t = \left[\bm{\mu}(\rvx_t, t) - \frac{1}{2} \sigma(t)^2 \nabla \log p_t(\rvx_t)\right] \ud t. \label{eq:pfode}
\end{align}
Here $\nabla \log p_t(\rvx)$ is the \emph{score function} of $p_t(\rvx)$; hence diffusion models are also known as \emph{score-based generative models} \citep{song2019generative,song2020improved,song2021scorebased}.

Typically, the SDE in \cref{eq:sde} is designed such that $p_T(\rvx)$ is close to a tractable Gaussian distribution $\pi(\rvx)$. We hereafter adopt the settings in \citet{karras2022edm}, where $\bm{\mu}(\rvx, t) = \bm{0}$ and $\sigma(t) = \sqrt{2t}$. In this case, we have $p_t(\rvx) = p_\text{data}(\rvx) \otimes \mcal{N}(\bm{0}, t^2 \mI)$, where $\otimes$ denotes the convolution operation, and $\pi(\rvx) = \mcal{N}(\bm{0}, T^2\mI)$. For sampling, we first train a \emph{score model} $\vs_\vphi(\rvx, t) \approx \nabla \log p_t(\rvx)$ via \emph{score matching} \citep{hyvarinen2005estimation,vincent2011connection,song2019sliced,song2019generative,ho2020denoising}, then plug it into \cref{eq:pfode} to obtain an empirical estimate of the PF ODE, which takes the form of
\begin{align}
    \frac{\ud \rvx_t}{\ud t} = -t \vs_\vphi(\rvx_t, t). \label{eq:e_pfode}
\end{align}
We call \cref{eq:e_pfode} the \emph{empirical PF ODE}. Next, we sample $\hat{\rvx}_T \sim \pi = \mcal{N}(\bm{0}, T^2 \mI)$ to initialize the empirical PF ODE and solve it backwards in time with any numerical ODE solver, such as Euler \citep{song2020denoising,song2021scorebased} and Heun solvers \citep{karras2022edm}, to obtain the solution trajectory $\{\hat{\rvx}_t\}_{t\in[0,T]}$. The resulting $\hat{\rvx}_0$ can then be viewed as an approximate sample from the data distribution $p_\text{data}(\rvx)$. To avoid numerical instability, one typically stops the solver at $t=\epsilon$, where $\epsilon$ is a fixed small positive number, and accepts $\hat{\rvx}_{\epsilon}$ as the approximate sample. Following \citet{karras2022edm}, we rescale image pixel values to $[-1,1]$, and set $T=80, \epsilon=0.002$.

Diffusion models are bottlenecked by their slow sampling speed. Clearly, using ODE solvers for sampling requires iterative evaluations of the score model $\vs_\vphi(\rvx, t)$, which is computationally costly. Existing methods for fast sampling include faster numerical ODE solvers \cite{song2020denoising,zhang2022fast,lu2022dpm,dockhorn2022genie}, and distillation techniques \cite{luhman2021knowledge,salimans2022progressive,meng2022distillation,zheng2022fast}. However, ODE solvers still need more than 10 evaluation steps to generate competitive samples. Most distillation methods like \citet{luhman2021knowledge} and \citet{zheng2022fast} rely on collecting a large dataset of samples from the diffusion model prior to distillation, which itself is computationally expensive. To our best knowledge, the only distillation approach that does not suffer from this drawback is progressive distillation (PD, \citet{salimans2022progressive}), with which we compare consistency models extensively in our experiments.

\section{Consistency Models}\label{sec:consistency}

\begin{figure}[t]
    \centering
    \includegraphics[width=\linewidth]{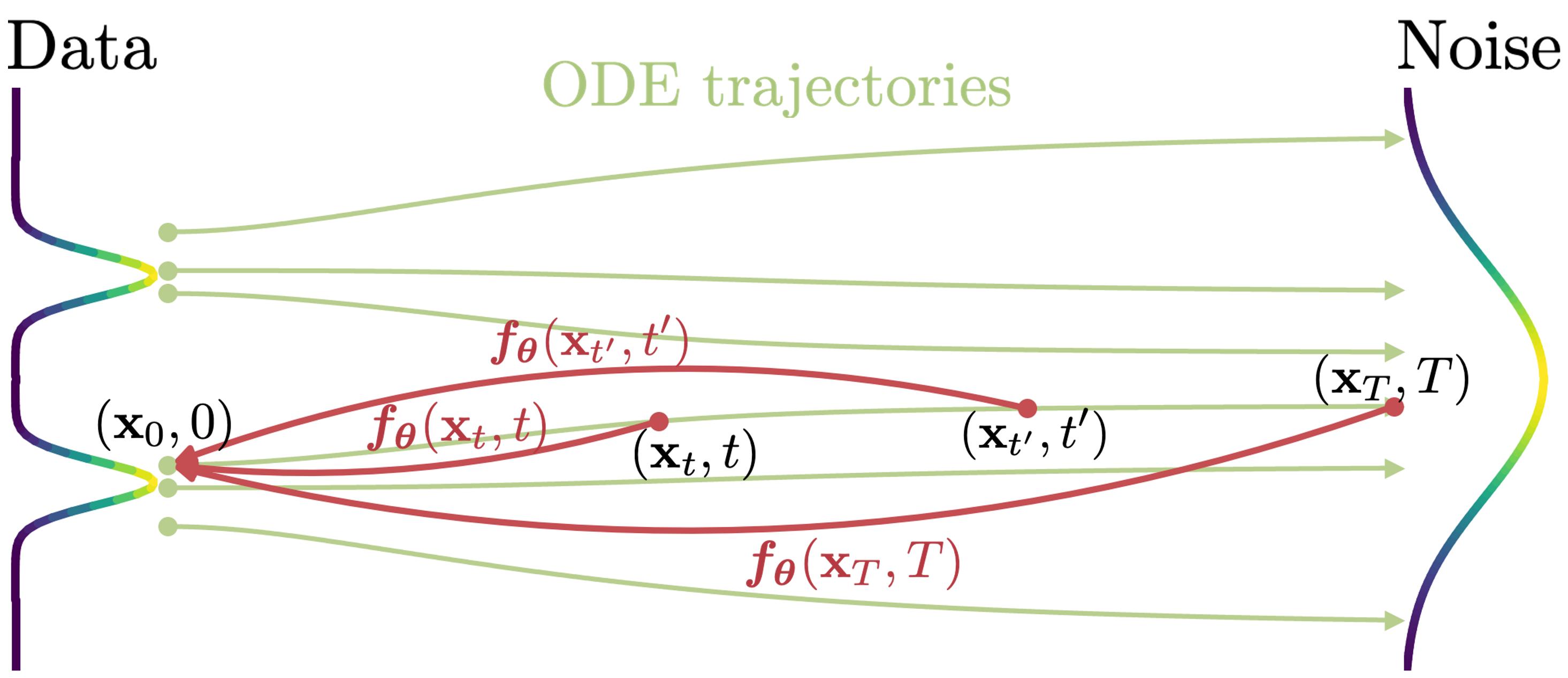}
    \caption{{\color{model}Consistency models} are trained to map points on any trajectory of the {\color{ode}PF ODE} to the trajectory's origin.}
    \label{fig:teaser}
\end{figure}
We propose consistency models, a new type of models that support single-step generation at the core of its design, while still allowing iterative generation for trade-offs between sample quality and compute, and zero-shot data editing. Consistency models can be trained in either the distillation mode or the isolation mode. In the former case, consistency models distill the knowledge of pre-trained diffusion models into a single-step sampler, significantly improving other distillation approaches in sample quality, while allowing zero-shot image editing applications. In the latter case, consistency models are trained in isolation, with no dependence on pre-trained diffusion models. This makes them an independent new class of generative models.

Below we introduce the definition, parameterization, and sampling of consistency models, plus a brief discussion on their applications to zero-shot data editing.

\textbf{Definition}~ Given a solution trajectory $\{ \rvx_t \}_{t\in[\epsilon,T]}$ of the PF ODE in \cref{eq:pfode}, we define the \emph{consistency function} as $\vf: (\rvx_t, t) \mapsto \rvx_{\epsilon}$. A consistency function has the property of \emph{self-consistency}: its outputs are consistent for arbitrary pairs of $(\rvx_t, t)$ that belong to the same PF ODE trajectory, \ie, $\vf(\rvx_t, t) = \vf(\rvx_{t'}, t')$ for all $t, t' \in [\epsilon,T]$. %
As illustrated in \cref{fig:teaser}, the goal of a \emph{consistency model}, symbolized as $\vf_\vtheta$, is to estimate this consistency function $\vf$ from data by learning to enforce the self-consistency property (details in \cref{sec:distillation,sec:generation}). Note that a similar definition is used for neural flows \cite{bilovs2021neural} in the context of neural ODEs \cite{chen2018neural}. Compared to neural flows, however, we do not enforce consistency models to be invertible.

\textbf{Parameterization}~ For any consistency function $\vf(\cdot, \cdot)$, we have $\vf(\rvx_\epsilon, \epsilon) = \rvx_\epsilon$, \ie, $\vf(\cdot, \epsilon)$ is an identity function. We call this constraint the \emph{boundary condition}. All consistency models have to meet this boundary condition, as it plays a crucial role in the successful training of consistency models. This boundary condition is also the most confining architectural constraint on consistency models. For consistency models based on deep neural networks, we discuss two ways to implement this boundary condition \emph{almost for free}. Suppose we have a free-form deep neural network $F_\vtheta(\rvx, t)$ whose output has the same dimensionality as $\rvx$. The first way is to simply parameterize the consistency model as
\begin{align}
    \vf_\vtheta(\rvx, t) = \begin{cases}
        \rvx &\quad t = \epsilon\\
        F_\vtheta(\rvx, t) &\quad t \in (\epsilon, T]
    \end{cases}\label{eq:param1}.
\end{align}
The second method is to parameterize the consistency model using skip connections, that is,
\begin{align}
    \vf_\vtheta(\rvx, t) = c_\text{skip}(t) \rvx + c_\text{out}(t) F_\vtheta(\rvx, t)\label{eq:param2},
\end{align}
where $c_\text{skip}(t)$ and $c_\text{out}(t)$ are differentiable functions such that $c_\text{skip}(\epsilon) = 1$, and $c_\text{out}(\epsilon) = 0$. This way, the consistency model is differentiable at $t = \epsilon$ if $F_\vtheta(\rvx, t), c_\text{skip}(t), c_\text{out}(t)$ are all differentiable, which is critical for training continuous-time consistency models (\cref{sec:ctcd,sec:ctct}). %
The parameterization in \cref{eq:param2} bears strong resemblance to many successful diffusion models \cite{karras2022edm,balaji2022eDiff-I}, making it easier to borrow powerful diffusion model architectures for constructing consistency models. We therefore follow the second parameterization in all experiments.

\textbf{Sampling}~ With a well-trained consistency model $\vf_\vtheta(\cdot, \cdot)$, we can generate samples by sampling from the initial distribution $\hat{\rvx}_T \sim \mcal{N}(\bm{0}, T^2 \mI)$ and then evaluating the consistency model for $\hat{\rvx}_\epsilon = \vf_\vtheta(\hat{\rvx}_T, T)$. This involves only one forward pass through the consistency model and therefore \emph{generates samples in a single step}. Importantly, one can also evaluate the consistency model multiple times by alternating denoising and noise injection steps for improved sample quality. Summarized in \cref{alg:sampling}, this \emph{multistep} sampling procedure provides the flexibility to trade compute for sample quality. It also has important applications in zero-shot data editing. In practice, we find time points $\{\tau_1, \tau_2, \cdots, \tau_{N-1}\}$ in \cref{alg:sampling} with a greedy algorithm, where the time points are pinpointed one at a time using ternary search to optimize the FID of samples obtained from \cref{alg:sampling}. This assumes that given prior time points, the FID is a unimodal function of the next time point. We find this assumption to hold empirically in our experiments, and leave the exploration of better strategies as future work.

\begin{algorithm}[tb]
    \caption{Multistep Consistency Sampling}
    \label{alg:sampling}
 \begin{algorithmic}
    \STATE {\bfseries Input:} Consistency model $\vf_\vtheta(\cdot, \cdot)$, sequence of time points $\tau_1 > \tau_2 > \cdots > \tau_{N-1}$, initial noise $\hat{\rvx}_T$
    \STATE $\rvx \gets \vf_\vtheta(\hat{\rvx}_T, T)$
    \FOR{$n=1$ {\bfseries to} $N-1$}
        \STATE Sample $\rvz \sim \mcal{N}(\bm{0}, \mI)$
        \STATE $\hat{\rvx}_{\tau_n} \gets \rvx + \sqrt{\tau_n^2 - \epsilon^2} \rvz$
        \STATE $\rvx \gets \vf_\vtheta(\hat{\rvx}_{\tau_n}, \tau_n)$
    \ENDFOR
    \STATE {\bfseries Output:} $\rvx$
 \end{algorithmic}
\end{algorithm}

\textbf{Zero-Shot Data Editing}~ Similar to diffusion models, consistency models enable various data editing and manipulation applications in zero shot; they do not require explicit training to perform these tasks. For example, consistency models define a one-to-one mapping from a Gaussian noise vector to a data sample. Similar to latent variable models like GANs, VAEs, and normalizing flows, consistency models can easily interpolate between samples by traversing the latent space (\cref{fig:bedroom_interp}). As consistency models are trained to recover $\rvx_\epsilon$ from any noisy input $\rvx_t$ where $t \in [\epsilon, T]$, they can perform denoising for various noise levels (\cref{fig:bedroom_denoising}). Moreover, the multistep generation procedure in \cref{alg:sampling} is useful for solving certain inverse problems in zero shot by using an iterative replacement procedure similar to that of diffusion models \cite{song2019generative,song2021scorebased,ho2022video}. This enables many applications in the context of image editing, including inpainting (\cref{fig:bedroom_inpainting}), colorization (\cref{fig:bedroom_colorization}), super-resolution (\cref{fig:bedroom_superres_lite}) and stroke-guided image editing (\cref{fig:bedroom_sdedit}) as in SDEdit \cite{meng2021sdedit}. In \cref{sec:zeroshot}, we empirically demonstrate the power of consistency models on many zero-shot image editing tasks.

\section{Training Consistency Models via Distillation}\label{sec:distillation}

We present our first method for training consistency models based on distilling a pre-trained score model $\vs_\vphi(\rvx, t)$. Our discussion revolves around the empirical PF ODE in \cref{eq:e_pfode}, obtained by plugging the score model $\vs_\vphi(\rvx, t)$ into the PF ODE. Consider discretizing the time horizon $[\epsilon, T]$ into $N-1$ sub-intervals, with boundaries $t_1=\epsilon < t_2 < \cdots < t_{N}=T$. In practice, we follow \citet{karras2022edm} to determine the boundaries with the formula $t_i = (\epsilon^{1/\rho} + \nicefrac{i-1}{N-1} (T^{1/\rho} - \epsilon^{1/\rho}))^\rho$, where $\rho=7$. When $N$ is sufficiently large, we can obtain an accurate estimate of $\rvx_{t_n}$ from $\rvx_{t_{n+1}}$ by running one discretization step of a numerical ODE solver. This estimate, which we denote as $\hat{\rvx}_{t_n}^\vphi$, is defined by
\begin{align}
   \hat{\rvx}_{t_n}^\vphi \coloneqq \rvx_{t_{n+1}} + (t_n - t_{n+1})\Phi(\rvx_{t_{n+1}}, t_{n+1}; \vphi),\label{eq:odesolve}
\end{align}
where $\Phi(\cdots; \vphi)$ represents the update function of a one-step ODE solver applied to the empirical PF ODE. For example, when using the Euler solver, we have $\Phi(\rvx, t; \vphi) = -t \vs_\vphi(\rvx, t)$ which corresponds to the following update rule
\begin{align*}
    \hat{\rvx}_{t_n}^\vphi = \rvx_{t_{n+1}} - (t_n - t_{n+1}) t_{n+1} \vs_\vphi(\rvx_{t_{n+1}}, t_{n+1}).
\end{align*}
For simplicity, we only consider one-step ODE solvers in this work. It is straightforward to generalize our framework to multistep ODE solvers and we leave it as future work.

Due to the connection between the PF ODE in \cref{eq:pfode} and the SDE in \cref{eq:sde} (see \cref{sec:diffusion}), one can sample along the distribution of ODE trajectories by first sampling $\rvx \sim p_\text{data}$, then adding Gaussian noise to $\rvx$. Specifically, given a data point $\rvx$, we can generate a pair of adjacent data points $(\hat{\rvx}_{t_n}^\vphi, \rvx_{t_{n+1}})$ on the PF ODE trajectory efficiently by sampling $\rvx$ from the dataset, followed by sampling $\rvx_{t_{n+1}}$ from the transition density of the SDE $\mcal{N}(\rvx, t_{n+1}^2 \mI)$, and then computing $\hat{\rvx}_{t_n}^\vphi$ using one discretization step of the numerical ODE solver according to \cref{eq:odesolve}. Afterwards, we train the consistency model by minimizing its output differences on the pair $(\hat{\rvx}_{t_n}^\vphi, \rvx_{t_{n+1}})$. This motivates our following \emph{consistency distillation} loss for training consistency models.

\begin{definition}\label{def:loss}
The consistency distillation loss is defined as
\begin{multline}
    \mcal{L}_\text{CD}^N(\vtheta, \vtheta^{-}; \vphi) \coloneqq \\ \mbb{E}[\lambda(t_n) d(\vf_\vtheta({\rvx}_{t_{n+1}}, t_{n+1}), \vf_{\vtheta^-}(\hat{\rvx}_{t_n}^\vphi, t_n))], \label{eq:distill_obj}
\end{multline}
where the expectation is taken with respect to $\rvx \sim p_\text{data}$, $n \sim \mcal{U}\llbracket 1,N-1 \rrbracket$, and $\rvx_{t_{n+1}} \sim \mcal{N}(\rvx; t_{n+1}^2 \mI)$. Here $\mcal{U}\llbracket 1, N-1 \rrbracket$ denotes the uniform distribution over $\{1,2,\cdots, N-1\}$, $\lambda(\cdot) \in \mbb{R}^+$ is a positive weighting function, $\hat{\rvx}_{t_n}^\vphi$ is given by \cref{eq:odesolve}, $\vtheta^-$ denotes a running average of the past values of $\vtheta$ during the course of optimization, and $d(\cdot, \cdot)$ is a metric function that satisfies $\forall \rvx, \rvy: d(\rvx, \rvy) \geq 0$ and $d(\rvx, \rvy) = 0$ if and only if $\rvx = \rvy$.
\end{definition}

Unless otherwise stated, we adopt the notations in \cref{def:loss} throughout this paper, and use $\mbb{E}[\cdot]$ to denote the expectation over all random variables. In our experiments, we consider the squared $\ell_2$ distance $d(\rvx, \rvy) = \|\rvx - \rvy\|^2_2$, $\ell_1$ distance $d(\rvx, \rvy) = \|\rvx-\rvy\|_1$, and the Learned Perceptual Image Patch Similarity (LPIPS, \citet{zhang2018perceptual}). We find $\lambda(t_n) \equiv 1$ performs well across all tasks and datasets. In practice, we minimize the objective by stochastic gradient descent on the model parameters $\vtheta$, while updating $\vtheta^-$ with exponential moving average (EMA). That is, given a decay rate $0 \leq \mu < 1$, we perform the following update after each optimization step:
\begin{align}
    \vtheta^- \leftarrow \operatorname{stopgrad}(\mu \vtheta^- + (1-\mu) \vtheta).\label{eq:ema}
\end{align}
The overall training procedure is summarized in \cref{alg:distillation}. In alignment with the convention in deep reinforcement learning \cite{mnih2013playing,mnih2015human,lillicrap2015continuous} and momentum based contrastive learning \cite{grill2020bootstrap,he2020momentum}, we refer to $\vf_{\vtheta^-}$ as the ``target network'', and $\vf_\vtheta$ as the ``online network''. We find that compared to simply setting $\vtheta^- = \vtheta$, the EMA update and ``stopgrad'' operator in \cref{eq:ema} can greatly stabilize the training process and improve the final performance of the consistency model.

\begin{algorithm}[tb]
    \caption{Consistency Distillation (CD)}\label{alg:distillation}
 \begin{algorithmic}
    \STATE {\bfseries Input:} dataset $\mcal{D}$, initial model parameter $\vtheta$, learning rate $\eta$, ODE solver $\Phi(\cdot, \cdot; \vphi)$, $d(\cdot, \cdot)$, $\lambda(\cdot)$, and $\mu$
    \STATE $\vtheta^- \gets \vtheta$
    \REPEAT
    \STATE Sample $\rvx \sim \mcal{D}$ and $n \sim \mcal{U}\llbracket 1,N-1 \rrbracket$
    \STATE Sample $\rvx_{t_{n+1}} \sim \mcal{N}(\rvx; t_{n+1}^2 \mI)$
    \STATE $\hat{\rvx}_{t_n}^\vphi \gets \rvx_{t_{n+1}} + (t_n - t_{n+1})\Phi(\rvx_{t_{n+1}}, t_{n+1}; \vphi)$
    \vspace{0.33em}
    \STATE $\begin{multlined}
        \mcal{L}(\vtheta, \vtheta^{-}; \vphi) \gets \\ \lambda(t_n) d(\vf_\vtheta(\rvx_{t_{n+1}}, t_{n+1}), \vf_{\vtheta^-}(\hat{\rvx}_{t_n}^\vphi, t_n))
    \end{multlined}
    $
    \vspace{0.33em}
    \STATE $\vtheta \gets \vtheta - \eta \nabla_\vtheta \mcal{L}(\vtheta, \vtheta^{-}; \vphi)$
    \STATE $\vtheta^- \gets \operatorname{stopgrad}(\mu \vtheta^- + (1-\mu) \vtheta$)
    \UNTIL{convergence}
 \end{algorithmic}
 \end{algorithm}

Below we provide a theoretical justification for consistency distillation based on asymptotic analysis.

\begin{theorem}\label{thm:convergence}
    Let $\Delta t \coloneqq \max_{n \in \llbracket 1, N-1\rrbracket}\{|t_{n+1} - t_{n}|\}$, and $\vf(\cdot,\cdot;\vphi)$ be the consistency function of the empirical PF ODE in \cref{eq:e_pfode}. Assume $\vf_\vtheta$ satisfies the Lipschitz condition: there exists $L > 0$ such that for all $t \in [\epsilon, T]$, $\rvx$, and $\rvy$, we have $\norm{\vf_\vtheta(\rvx, t) - \vf_\vtheta(\rvy, t)}_2 \leq L \norm{\rvx - \rvy}_2$. Assume further that for all $n \in \llbracket 1, N-1 \rrbracket$, the ODE solver called at $t_{n+1}$ has local error uniformly bounded by $O((t_{n+1} - t_n)^{p+1})$ with $p\geq 1$. Then, if $\mcal{L}_\text{CD}^N(\vtheta, \vtheta; \vphi) = 0$, we have
    \begin{align*}
        \sup_{n, \rvx}\|\vf_{\vtheta}(\rvx, t_n) - \vf(\rvx, t_n; \vphi)\|_2 = O((\Delta t)^p).
    \end{align*}
\end{theorem}
\begin{proof}
    The proof is based on induction and parallels the classic proof of global error bounds for numerical ODE solvers \cite{suli2003introduction}. We provide the full proof in \cref{app:proof_cd}.
\end{proof}

Since $\vtheta^{-}$ is a running average of the history of $\vtheta$, we have $\vtheta^{-} = \vtheta$ when the optimization of \cref{alg:distillation} converges. That is, the target and online consistency models will eventually match each other. If the consistency model additionally achieves zero consistency distillation loss, then \cref{thm:convergence} implies that, under some regularity conditions, the estimated consistency model can become arbitrarily accurate, as long as the step size of the ODE solver is sufficiently small. Importantly, our boundary condition $\vf_\vtheta(\rvx, \epsilon) \equiv \rvx$ precludes the trivial solution $\vf_\vtheta(\rvx, t) \equiv \bm{0}$ from arising in consistency model training.

The consistency distillation loss $\mcal{L}_\text{CD}^N(\vtheta, \vtheta^{-}; \vphi)$ can be extended to hold for infinitely many time steps ($N \to \infty$) if $\vtheta^{-} = \vtheta$ or $\vtheta^{-} = \operatorname{stopgrad}(\vtheta)$. The resulting continuous-time loss functions do not require specifying $N$ nor the time steps $\{t_1, t_2, \cdots, t_N\}$. Nonetheless, they involve Jacobian-vector products and require forward-mode automatic differentiation for efficient implementation, which may not be well-supported in some deep learning frameworks. We provide these continuous-time distillation loss functions in \cref{thm:ctcd1,thm:ctcd_l1,thm:ctcd2}, and relegate details to \cref{sec:ctcd}.

\section{Training Consistency Models in Isolation}\label{sec:generation}
Consistency models can be trained without relying on any pre-trained diffusion models. This differs from existing diffusion distillation techniques, making consistency models a new independent family of generative models.
\begin{algorithm}[tb]
    \caption{Consistency Training (CT)}
        \label{alg:dtct}
        \begin{algorithmic}
        \STATE {\bfseries Input:} dataset $\mcal{D}$, initial model parameter $\vtheta$, learning rate $\eta$, step schedule $N(\cdot)$, EMA decay rate schedule $\mu(\cdot)$, $d(\cdot, \cdot)$, and $\lambda(\cdot)$
        \STATE $\vtheta^- \gets \vtheta$ and $k \gets 0$
        \REPEAT
        \STATE Sample $\rvx \sim \mcal{D}$, and $n \sim \mcal{U}\llbracket 1,N(k)-1 \rrbracket$
        \STATE Sample $\rvz \sim \mcal{N}(\bm{0}, \mI)$
        \STATE $\begin{multlined}
            \mcal{L}(\vtheta, \vtheta^{-}) \gets \\
            \lambda(t_n) d(\vf_\vtheta(\rvx + t_{n+1} \rvz, t_{n+1}),\vf_{\vtheta^-}(\rvx + t_n \rvz, t_n))
        \end{multlined}$
        \STATE $\vtheta \gets \vtheta - \eta \nabla_\vtheta \mcal{L}(\vtheta, \vtheta^{-})$
        \STATE $\vtheta^- \gets \operatorname{stopgrad}(\mu(k) \vtheta^- + (1-\mu(k)) \vtheta)$
        \STATE $k \gets k + 1$
        \UNTIL{convergence}
    \end{algorithmic}
\end{algorithm}

Recall that in consistency distillation, we rely on a pre-trained score model $\vs_\vphi(\rvx, t)$ to approximate the ground truth score function $\nabla \log p_t(\rvx)$. It turns out that we can avoid this pre-trained score model altogether by leveraging the following unbiased estimator (\cref{lem:grad_log_p_t} in \cref{app:proof}):
\begin{align*}
    \nabla \log p_t(\rvx_t) = -\mbb{E}\left[ \frac{\rvx_t - \rvx}{t^2} \mathrel\bigg| \rvx_t \right],
\end{align*}
where $\rvx\sim p_\text{data}$ and $\rvx_t \sim \mcal{N}(\rvx; t^2 \mI)$. That is, given $\rvx$ and $\rvx_t$, we can estimate $\nabla \log p_t(\rvx_t)$ with $-(\rvx_t-\rvx)/t^2$.

This unbiased estimate suffices to replace the pre-trained diffusion model in consistency distillation when using the Euler method as the ODE solver in the limit of $N\to\infty$, as justified by the following result.
\begin{theorem}\label{thm:ct}
    Let $\Delta t \coloneqq \max_{n \in \llbracket 1, N-1\rrbracket}\{|t_{n+1} - t_{n}|\}$. Assume $d$ and $\vf_{\vtheta^{-}}$ are both twice continuously differentiable with bounded second derivatives, the weighting function $\lambda(\cdot)$ is bounded, and $\mbb{E}[\norm{\nabla \log p_{t_n}(\rvx_{t_{n}})}_2^2] < \infty$. Assume further that we use the Euler ODE solver, and the pre-trained score model matches the ground truth, \ie, $\forall t\in[\epsilon, T]: \vs_{\vphi}(\rvx, t) \equiv \nabla \log p_t(\rvx)$. Then,
    \begin{align}
       \mcal{L}_\text{CD}^N(\vtheta, \vtheta^{-}; \vphi) = \mcal{L}_\text{CT}^N(\vtheta, \vtheta^{-}) + o(\Delta t),\label{eq:cd2}
    \end{align}
    where the expectation is taken with respect to $\rvx \sim p_\text{data}$, $n \sim \mcal{U}\llbracket 1,N-1 \rrbracket$, and $\rvx_{t_{n+1}} \sim \mcal{N}(\rvx; t_{n+1}^2 \mI)$. The consistency training objective, denoted by $\mcal{L}_\text{CT}^N(\vtheta, \vtheta^{-})$, is defined as
    \begin{align}
        \resizebox{0.87\linewidth}{!}{$\displaystyle
        \mbb{E}[\lambda(t_n) d(\vf_\vtheta(\rvx + t_{n+1}\rvz, t_{n+1}), \vf_{\vtheta^{-}}(\rvx + t_n\rvz, t_n))]$},\label{eq:dtct}
    \end{align}
    where $\rvz \sim \mcal{N}(\bf{0}, \mI)$. Moreover, $\mcal{L}_\text{CT}^N(\vtheta, \vtheta^{-}) \geq O(\Delta t)$ if $\inf_N \mcal{L}_\text{CD}^N(\vtheta, \vtheta^{-}; \vphi) > 0$.
\end{theorem}
\begin{proof}
    The proof is based on Taylor series expansion and properties of score functions (\cref{lem:grad_log_p_t}). A complete proof is provided in \cref{app:proof_ct}.
\end{proof}

We refer to \cref{eq:dtct} as the \emph{consistency training} (CT) loss. Crucially, $\mcal{L}(\vtheta, \vtheta^{-})$ only depends on the online network $\vf_\vtheta$, and the target network $\vf_{\vtheta^{-}}$, while being completely agnostic to diffusion model parameters $\vphi$. The loss function $\mcal{L}(\vtheta, \vtheta^{-}) \geq O(\Delta t)$ decreases at a slower rate than the remainder $o(\Delta t)$ and thus will dominate the loss in \cref{eq:cd2} as $N\to\infty$ and $\Delta t \to 0$.

For improved practical performance, we propose to progressively increase $N$ during training according to a schedule function $N(\cdot)$. The intuition (\cf, \cref{fig:ct_adaptive}) is that the consistency training loss has less ``variance'' but more ``bias'' with respect to the underlying consistency distillation loss (\ie, the left-hand side of \cref{eq:cd2}) when $N$ is small (\ie, $\Delta t$ is large), which facilitates faster convergence at the beginning of training. On the contrary, it has more ``variance'' but less ``bias'' when $N$ is large (\ie, $\Delta t$ is small), which is desirable when closer to the end of training. For best performance, we also find that $\mu$ should change along with $N$, according to a schedule function $\mu(\cdot)$. The full algorithm of consistency training is provided in \cref{alg:dtct}, and the schedule functions used in our experiments are given in \cref{app:exp}.

Similar to consistency distillation, the consistency training loss $\mcal{L}_\text{CT}^N (\vtheta, \vtheta^{-})$ can be extended to hold in continuous time (\ie, $N \to \infty$) if $\vtheta^{-} = \operatorname{stopgrad}(\vtheta)$, as shown in \cref{thm:ctct}. This continuous-time loss function does not require schedule functions for $N$ or $\mu$, but requires forward-mode automatic differentiation for efficient implementation. %
Unlike the discrete-time CT loss, there is no undesirable ``bias'' associated with the continuous-time objective, as we effectively take $\Delta t \to 0$ in \cref{thm:ct}. We relegate more details to \cref{sec:ctct}.

\section{Experiments}\label{sec:experiments}
We employ consistency distillation and consistency training to learn consistency models on real image datasets, including CIFAR-10 \cite{krizhevsky2009learning}, ImageNet $64\times 64$ \cite{deng2009imagenet}, LSUN Bedroom $256\times 256$, and LSUN Cat $256\times 256$ \cite{yu2015lsun}. Results are compared according to Fr\'echet Inception Distance (FID, \citet{heusel2017gans}, lower is better), Inception Score (IS, \citet{salimans2016improved}, higher is better), Precision (Prec., \citet{kynkaanniemi2019improved}, higher is better), and Recall (Rec., \citet{kynkaanniemi2019improved}, higher is better). Additional experimental details are provided in \cref{app:exp}.

\begin{figure*}
    \centering
    \begin{subfigure}[b]{0.25\textwidth}
        \includegraphics[width=\textwidth]{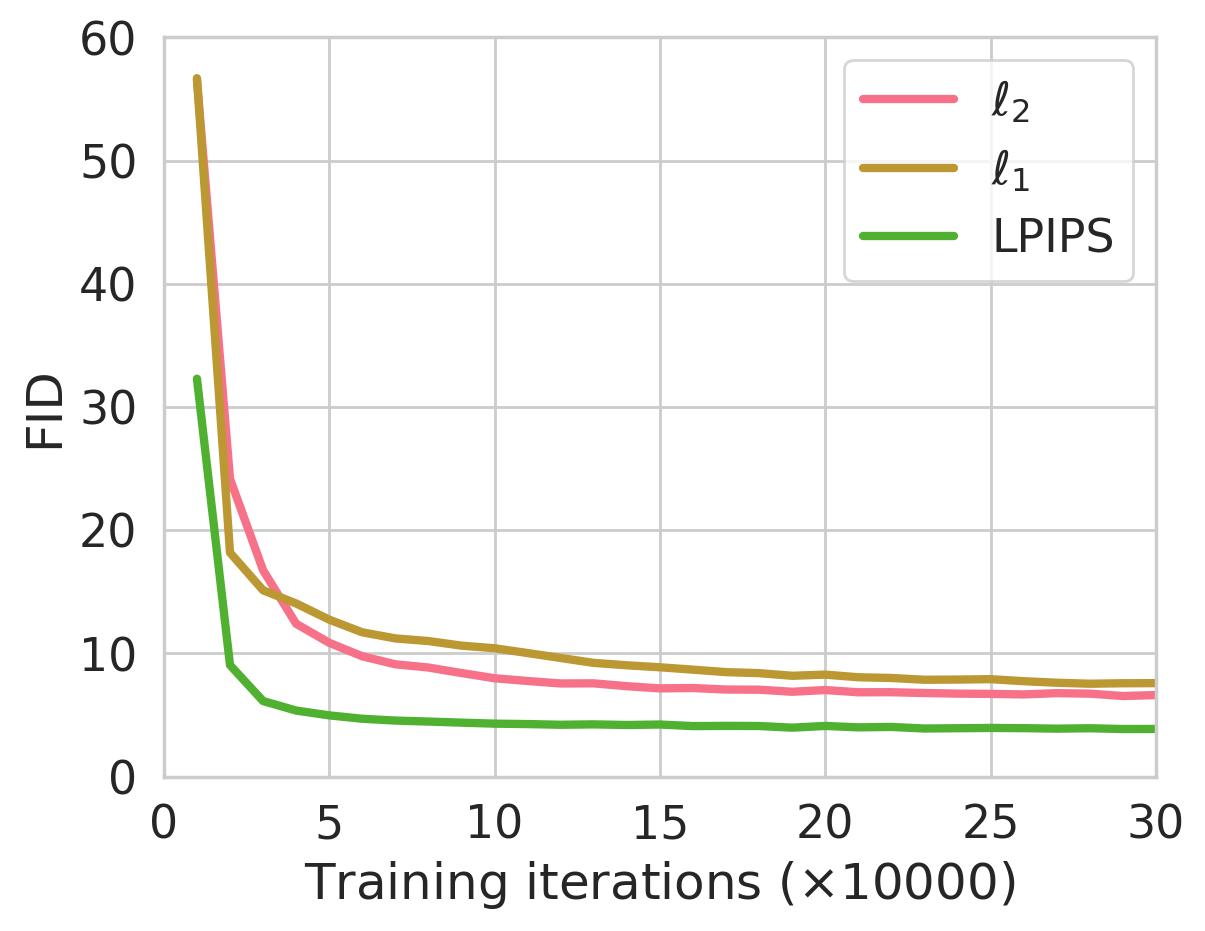}
        \caption{Metric functions in CD.}\label{fig:cd_compare}
    \end{subfigure}%
    \begin{subfigure}[b]{0.25\textwidth}
        \includegraphics[width=\textwidth]{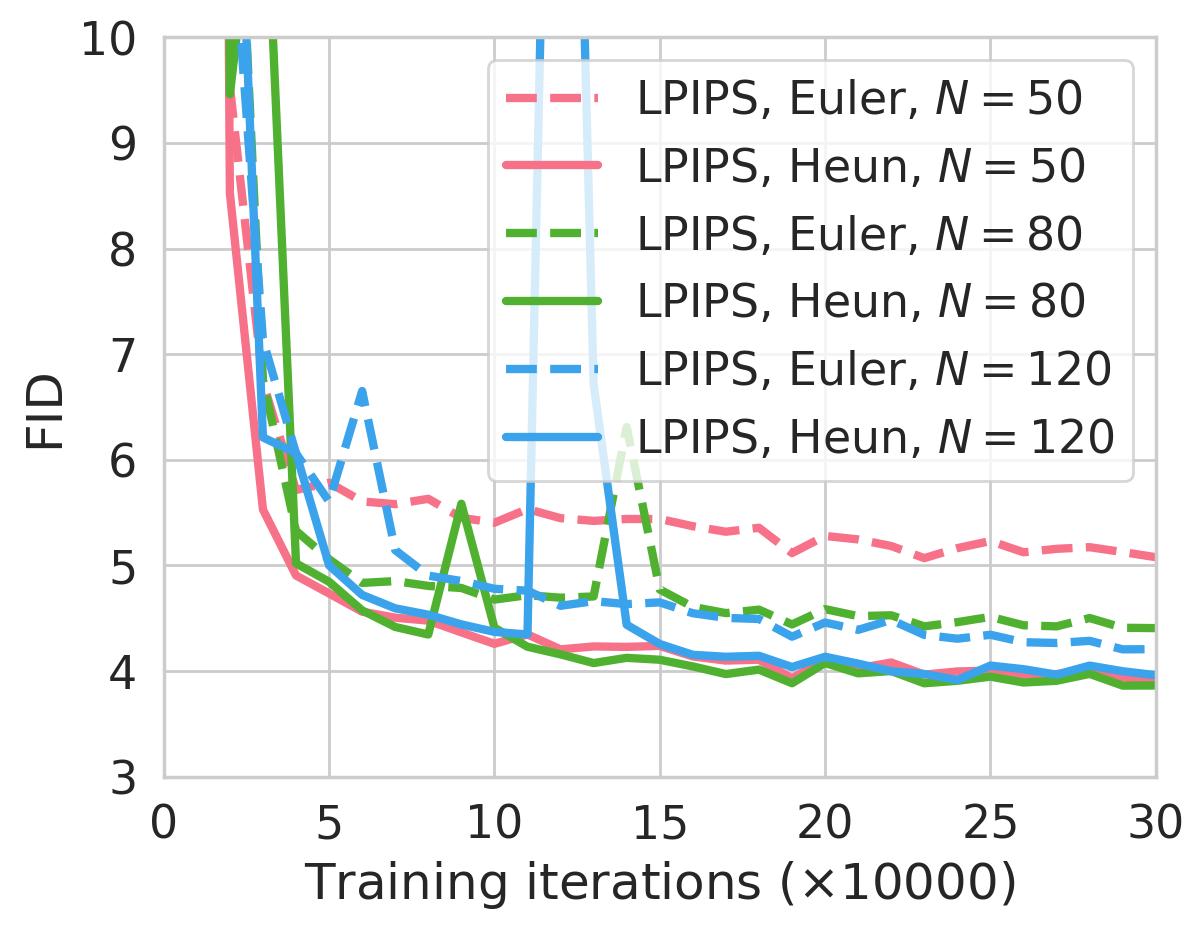}
        \caption{Solvers and $N$ in CD.}\label{fig:cd_solver}
    \end{subfigure}%
    \begin{subfigure}[b]{0.25\textwidth}
        \includegraphics[width=\textwidth]{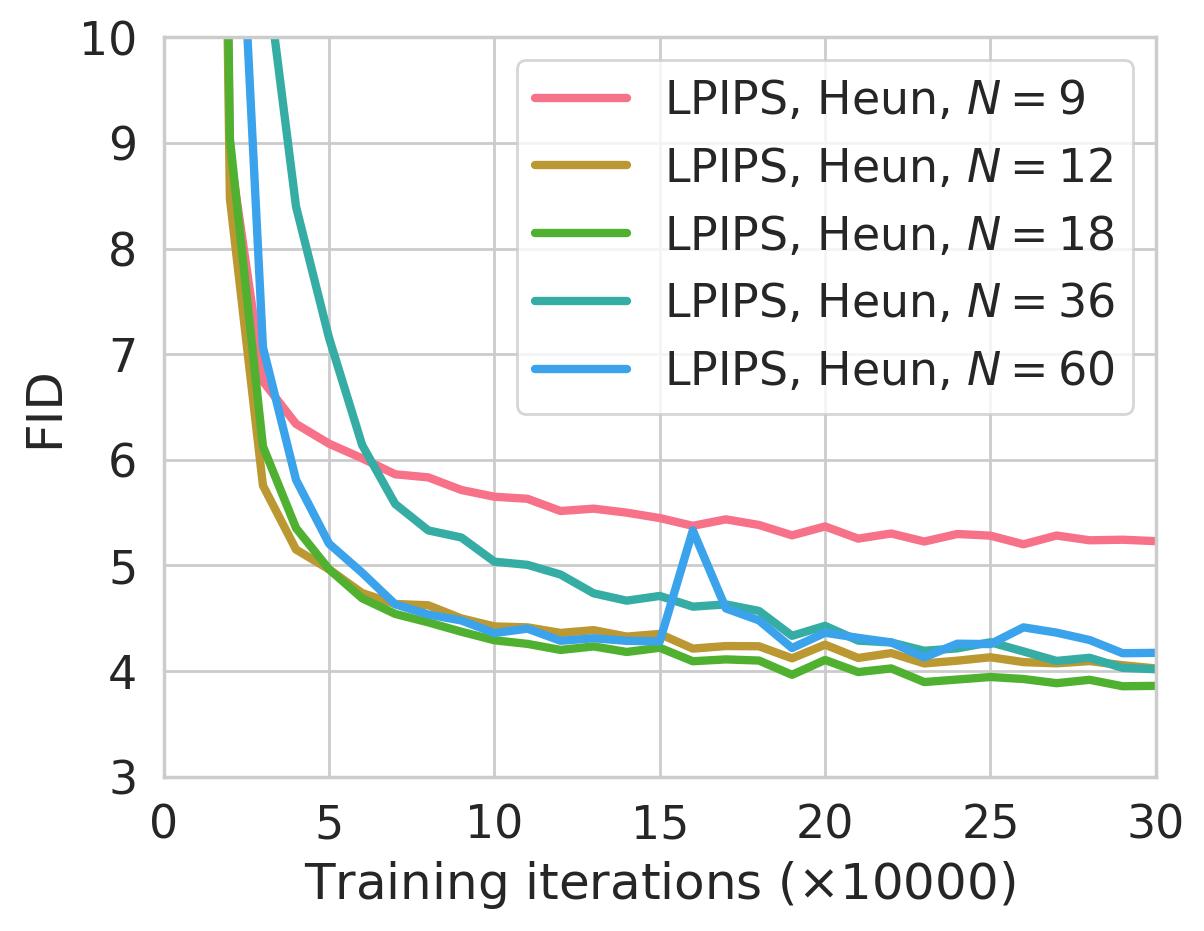}
        \caption{$N$ with Heun solver in CD.}\label{fig:cd_n}
    \end{subfigure}%
    \begin{subfigure}[b]{0.25\textwidth}
        \includegraphics[width=\textwidth]{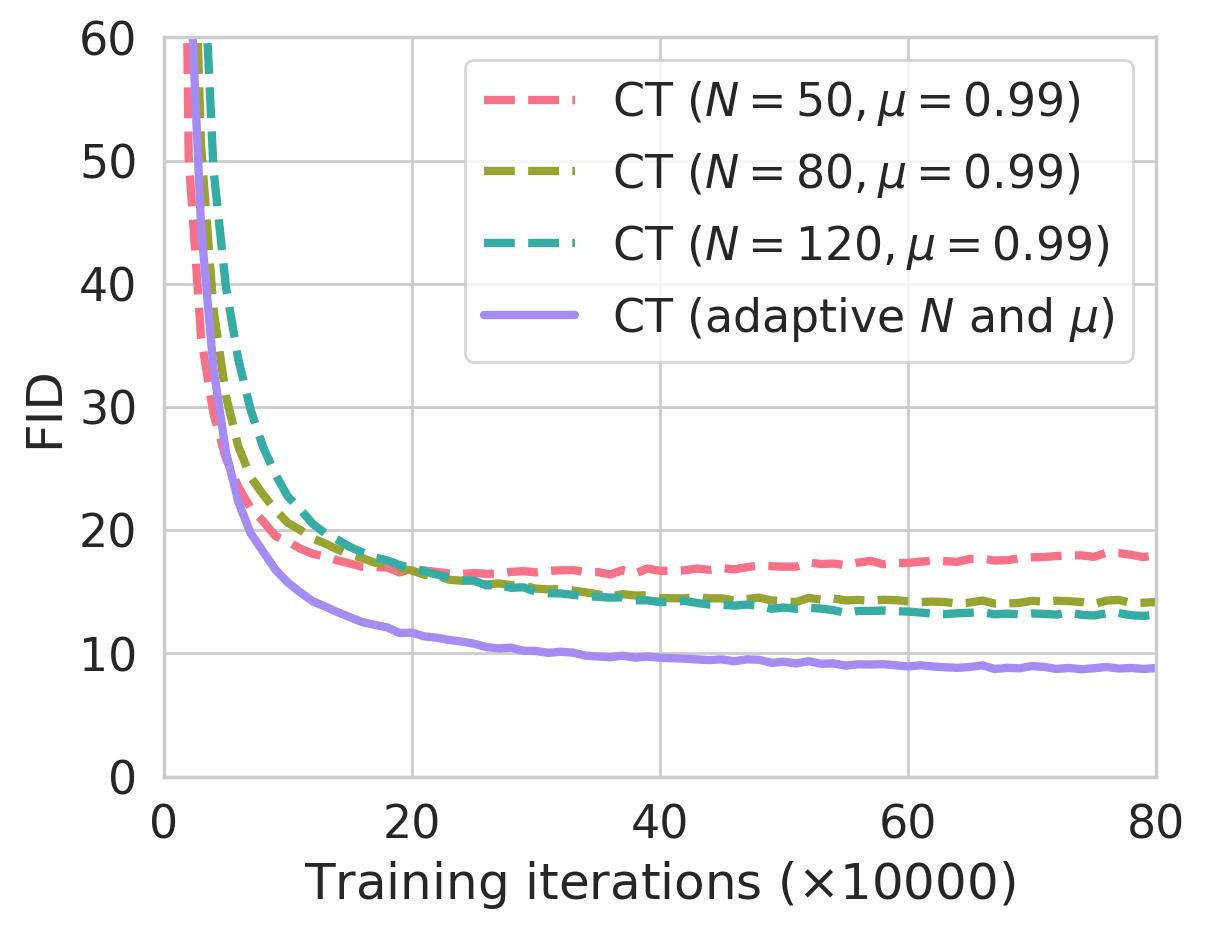}
        \caption{Adaptive $N$ and $\mu$ in CT.}\label{fig:ct_adaptive}
    \end{subfigure}
    \caption{Various factors that affect consistency distillation (CD) and consistency training (CT) on CIFAR-10. The best configuration for CD is LPIPS, Heun ODE solver, and $N=18$. Our adaptive schedule functions for $N$ and $\mu$ make CT converge significantly faster than fixing them to be constants during the course of optimization.}
    \label{fig:cd_ablation}
\end{figure*}

\begin{figure*}
    \centering
    \begin{subfigure}[b]{0.25\textwidth}
        \includegraphics[width=\textwidth]{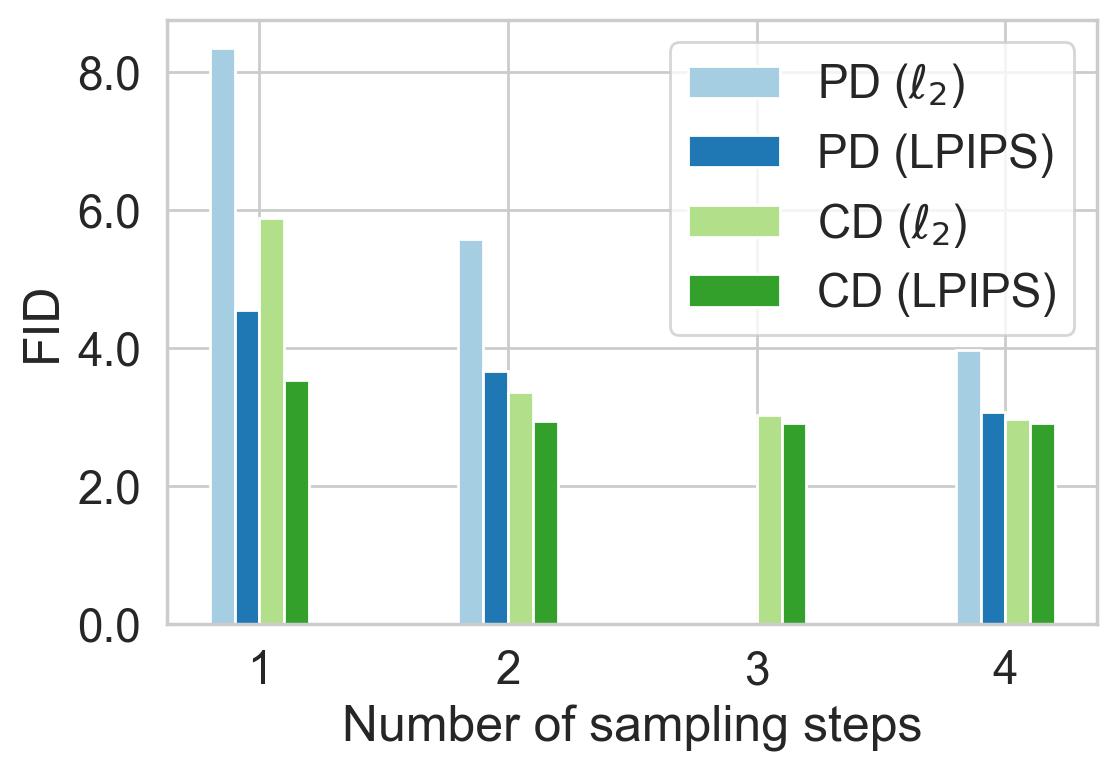}
        \caption{CIFAR-10}
    \end{subfigure}%
    \begin{subfigure}[b]{0.25\textwidth}
        \includegraphics[width=\textwidth]{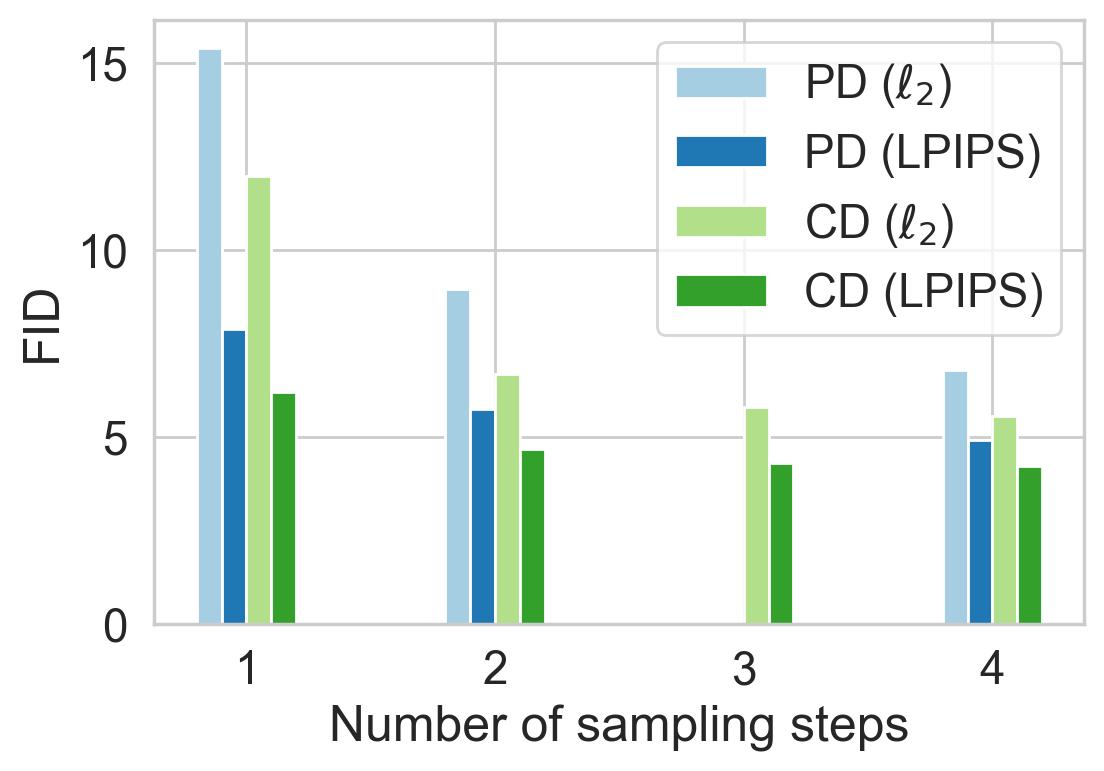}
        \caption{ImageNet $64\times 64$}
    \end{subfigure}%
    \begin{subfigure}[b]{0.25\textwidth}
        \includegraphics[width=\textwidth]{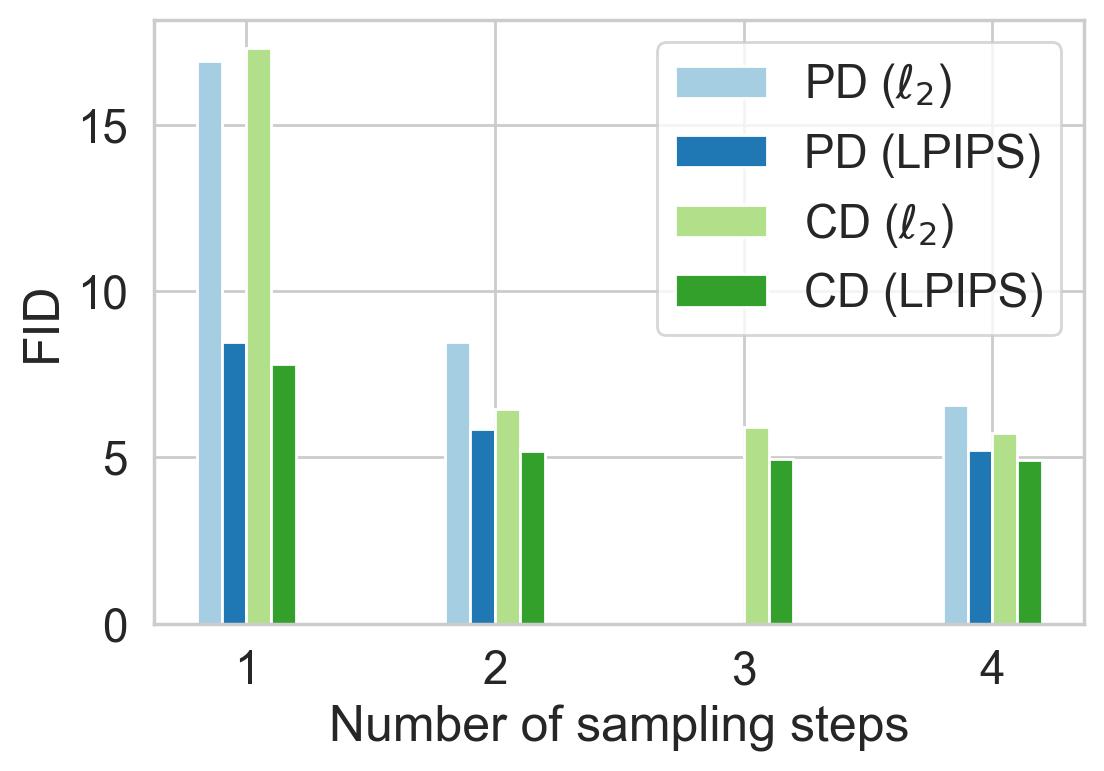}
        \caption{Bedroom $256\times 256$}
    \end{subfigure}%
    \begin{subfigure}[b]{0.25\textwidth}
        \includegraphics[width=\textwidth]{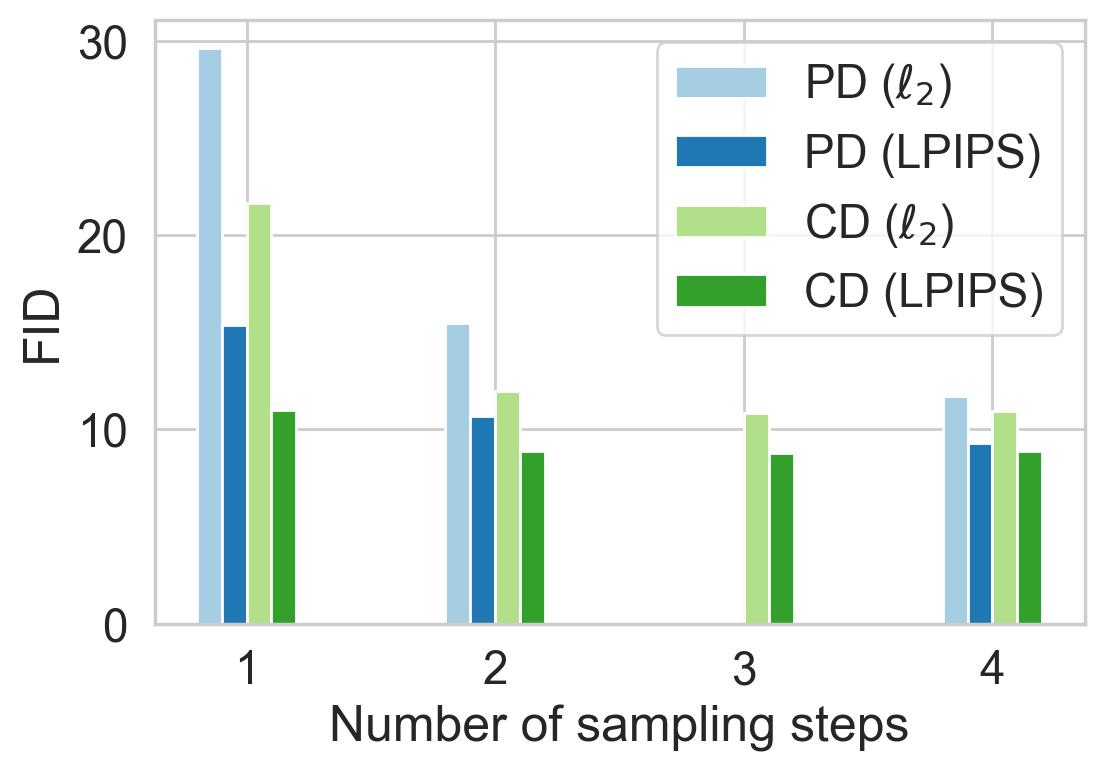}
        \caption{Cat $256\times 256$}
    \end{subfigure}
    \caption{Multistep image generation with consistency distillation (CD). CD outperforms progressive distillation (PD) across all datasets and sampling steps. The only exception is single-step generation on Bedroom $256\times 256$.}
    \label{fig:iterative_sampling}
\end{figure*}

\subsection{Training Consistency Models}
We perform a series of experiments on CIFAR-10 to understand the effect of various hyperparameters on the performance of consistency models trained by consistency distillation (CD) and consistency training (CT). We first focus on the effect of the metric function $d(\cdot, \cdot)$, the ODE solver, and the number of discretization steps $N$ in CD, then investigate the effect of the schedule functions $N(\cdot)$ and $\mu(\cdot)$ in CT.

To set up our experiments for CD, we consider the squared $\ell_2$ distance $d(\rvx, \rvy) = \|\rvx - \rvy\|^2_2$, $\ell_1$ distance $d(\rvx, \rvy) = \|\rvx-\rvy\|_1$, and the Learned Perceptual Image Patch Similarity (LPIPS, \citet{zhang2018perceptual}) as the metric function. For the ODE solver, we compare Euler's forward method and Heun's second order method as detailed in \citet{karras2022edm}. For the number of discretization steps $N$, we compare $N \in \{9, 12, 18, 36, 50, 60, 80, 120\}$. All consistency models trained by CD in our experiments are initialized with the corresponding pre-trained diffusion models, whereas models trained by CT are randomly initialized.

As visualized in \cref{fig:cd_compare}, the optimal metric for CD is LPIPS, which outperforms both $\ell_1$ and $\ell_2$ by a large margin over all training iterations. This is expected as the outputs of consistency models are images on CIFAR-10, and LPIPS is specifically designed for measuring the similarity between natural images. Next, we investigate which ODE solver and which discretization step $N$ work the best for CD. As shown in \cref{fig:cd_solver,fig:cd_n}, Heun ODE solver and $N=18$ are the best choices. Both are in line with the recommendation of \citet{karras2022edm} despite the fact that we are training consistency models, not diffusion models. Moreover, \cref{fig:cd_solver} shows that with the same $N$, Heun's second order solver uniformly outperforms Euler's first order solver. This corroborates with \cref{thm:convergence}, which states that the optimal consistency models trained by higher order ODE solvers have smaller estimation errors with the same $N$. The results of \cref{fig:cd_n} also indicate that once $N$ is sufficiently large, the performance of CD becomes insensitive to $N$. Given these insights, we hereafter use LPIPS and Heun ODE solver for CD unless otherwise stated. For $N$ in CD, we follow the suggestions in \citet{karras2022edm} on CIFAR-10 and ImageNet $64\times 64$. We tune $N$ separately on other datasets (details in \cref{app:exp}).

Due to the strong connection between CD and CT, we adopt LPIPS for our CT experiments throughout this paper. Unlike CD, there is no need for using Heun's second order solver in CT as the loss function does not rely on any particular numerical ODE solver. As demonstrated in \cref{fig:ct_adaptive}, the convergence of CT is highly sensitive to $N$---smaller $N$ leads to faster convergence but worse samples, whereas larger $N$ leads to slower convergence but better samples upon convergence. This matches our analysis in \cref{sec:generation}, and motivates our practical choice of progressively growing $N$ and $\mu$ for CT to balance the trade-off between convergence speed and sample quality. As shown in \cref{fig:ct_adaptive}, adaptive schedules of $N$ and $\mu$ significantly improve the convergence speed and sample quality of CT. In our experiments, we tune the schedules $N(\cdot)$ and $\mu(\cdot)$ separately for images of different resolutions, with more details in \cref{app:exp}.

\begin{table*}
    \begin{minipage}[t]{0.49\linewidth}
	\caption{Sample quality on CIFAR-10. $^\ast$Methods that require synthetic data construction for distillation.}\label{tab:results}
	\centering
	{\setlength{\extrarowheight}{1.5pt}
	\begin{adjustbox}{max width=\linewidth}
	\begin{tabular}{lccc}
        \Xhline{3\arrayrulewidth}
	    METHOD & NFE ($\downarrow$) & FID ($\downarrow$) & IS ($\uparrow$) \\
        \\[-2ex]
        \multicolumn{4}{l}{\textbf{Diffusion + Samplers}}\\\Xhline{3\arrayrulewidth}
        DDIM \cite{song2020denoising}
        & 50 & 4.67\\
        DDIM \cite{song2020denoising}
        & 20 & 6.84\\
        DDIM \cite{song2020denoising}
        & 10 & 8.23\\
        DPM-solver-2 \cite{lu2022dpm}
        & 10 & 5.94\\
        DPM-solver-fast \cite{lu2022dpm}
        & 10 & 4.70 \\
        3-DEIS \cite{zhang2022fast}
        & 10 & \textbf{4.17}\\
        \\[-2ex]
        \multicolumn{4}{l}{\textbf{Diffusion + Distillation}}\\\Xhline{3\arrayrulewidth}
        Knowledge Distillation$^\ast$ \cite{luhman2021knowledge}
         & 1 & 9.36 &  \\
        DFNO$^\ast$ \cite{zheng2022fast}
        & 1 & 4.12 & \\
        1-Rectified Flow (+distill)$^\ast$ \cite{liu2022flow}
         & 1 & 6.18 & 9.08\\
        2-Rectified Flow (+distill)$^\ast$ \cite{liu2022flow}
         & 1 & 4.85 & 9.01\\
        3-Rectified Flow (+distill)$^\ast$ \cite{liu2022flow}
         & 1 & 5.21 & 8.79\\
        PD \cite{salimans2022progressive} & 1 & 8.34 & 8.69 \\
        \textbf{CD} & 1 & \textbf{3.55} & \textbf{9.48} \\
        \hline
        PD \cite{salimans2022progressive}  & 2 & 5.58 & 9.05 \\
        \textbf{CD}  & 2 & \textbf{2.93} & \textbf{9.75} \\
        \\[-3ex]
        \multicolumn{4}{l}{\textbf{Direct Generation}}\\\Xhline{3\arrayrulewidth}
        BigGAN \cite{brock2018large} & 1 & 14.7 & 9.22\\
         Diffusion GAN \cite{xiao2022tackling} & 1 & 14.6 & 8.93\\
         AutoGAN \cite{gong2019autogan} & 1 & 12.4 & 8.55\\
         E2GAN \cite{tian2020off} & 1 & 11.3 & 8.51\\
        ViTGAN \cite{lee2021vitgan}
         & 1 & 6.66 & 9.30 \\
         TransGAN \cite{jiang2021transgan} & 1 & 9.26 & 9.05 \\
        StyleGAN2-ADA \cite{karras2020analyzing}
         & 1 & 2.92 & \textbf{9.83}\\
        StyleGAN-XL \cite{sauer2022stylegan} & 1 & \textbf{1.85} & \\
        \hline
        Score SDE \cite{song2021scorebased} & 2000 & 2.20 & \textbf{9.89}\\
        DDPM \cite{ho2020denoising} & 1000 & 3.17 & 9.46\\
        LSGM \cite{vahdat2021score} & 147 & 2.10 & \\
        PFGM \cite{xu2022poisson} & 110 & 2.35 & 9.68\\
        EDM \cite{karras2022edm}
         & 35 & \textbf{2.04} & 9.84 \\
         \hline
        1-Rectified Flow \cite{liu2022flow}
         & 1 & 378 & 1.13\\
        Glow \cite{kingma2018glow}
         & 1 & 48.9 & 3.92 \\
        Residual Flow \cite{chen2019residual} & 1 & 46.4\\
        GLFlow \cite{xiao2019generative} & 1 & 44.6 & \\
        DenseFlow \cite{grcic2021densely} & 1 & 34.9 & \\
        DC-VAE \cite{parmar2021dual} & 1 & 17.9 & 8.20 \\
        \textbf{CT} & 1 & \textbf{8.70} & \textbf{8.49} \\
        \hline
        \textbf{CT}  & 2 & \textbf{5.83} & \textbf{8.85} \\
	\end{tabular}
    \end{adjustbox}
	}
\end{minipage}
\hfill
\begin{minipage}[t]{0.49\linewidth}
	\caption{Sample quality on ImageNet $64\times 64$, and LSUN Bedroom \& Cat $256\times 256$. $^\dagger$Distillation techniques. }\label{tab:results2}
	\centering
	{\setlength{\extrarowheight}{1.5pt}
	\begin{adjustbox}{max width=\linewidth}
	\begin{tabular}{lcccc}
	    \Xhline{3\arrayrulewidth}
        METHOD & NFE ($\downarrow$) & FID ($\downarrow$) & Prec. ($\uparrow$) & Rec. ($\uparrow$) \\
        \\[-2ex]
        \multicolumn{5}{l}{\textbf{ImageNet $\bm{64\times 64}$}}\\
        \Xhline{3\arrayrulewidth}
        PD$^\dagger$ \cite{salimans2022progressive} & 1 & 15.39 & 0.59 & 0.62\\
        DFNO$^{\dagger}$ \cite{zheng2022fast} & 1 & 8.35 & & \\
        \textbf{CD}$^\dagger$ & 1 & 6.20 & 0.68 & 0.63\\
        PD$^\dagger$ \cite{salimans2022progressive} & 2 & 8.95  & 0.63 & \textbf{0.65}\\
        \textbf{CD}$^\dagger$ & 2 & \textbf{4.70} & \textbf{0.69} & 0.64\\
        \hline
        ADM
        \cite{dhariwal2021diffusion}
        & 250 & \textbf{2.07} & 0.74 & 0.63 \\
        EDM
        \cite{karras2022edm}
        & 79 & 2.44 & 0.71 & \textbf{0.67} \\
        BigGAN-deep
        \cite{brock2018large}
        & 1 & 4.06 & \textbf{0.79} & 0.48 \\
        \textbf{CT} & 1 & 13.0 & 0.71 & 0.47\\
        \textbf{CT} & 2 & 11.1 & 0.69 & 0.56\\
        \\[-2ex]
        \multicolumn{5}{l}{\textbf{LSUN Bedroom $\bm{256\times 256}$}}\\
        \Xhline{3\arrayrulewidth}
        PD$^\dagger$ \cite{salimans2022progressive} & 1 & 16.92 & 0.47 & 0.27\\
        PD$^\dagger$ \cite{salimans2022progressive} & 2 & 8.47 & 0.56 & \textbf{0.39} \\
        \textbf{CD}$^\dagger$ & 1 & 7.80 & 0.66 & 0.34\\
        \textbf{CD}$^\dagger$ & 2 & \textbf{5.22} & \textbf{0.68} & \textbf{0.39}\\
        \hline
        DDPM
        \cite{ho2020denoising}
        & 1000 & 4.89 & 0.60 & 0.45\\
        ADM
        \cite{dhariwal2021diffusion}
        & 1000 & \textbf{1.90} & 0.66 &\textbf{0.51} \\
        EDM
        \cite{karras2022edm}
        & 79 & 3.57 & 0.66 & 0.45 \\
        PGGAN \cite{karras2018progressive} & 1 & 8.34 & &  \\
        PG-SWGAN \cite{wu2019sliced} & 1 & 8.0 & &  \\
        TDPM (GAN) \cite{zheng2023truncated} & 1 & 5.24 & & \\
        StyleGAN2 \cite{karras2020analyzing} & 1 & 2.35 & 0.59 & 0.48 \\
        \textbf{CT} & 1 & 16.0 & 0.60 & 0.17\\
        \textbf{CT} & 2 & 7.85 & \textbf{0.68} & 0.33\\
        \\[-2ex]
        \multicolumn{5}{l}{\textbf{LSUN Cat $\bm{256\times 256}$}}\\
        \Xhline{3\arrayrulewidth}
        PD$^\dagger$ \cite{salimans2022progressive} & 1 & 29.6 & 0.51 & 0.25 \\
        PD$^\dagger$ \cite{salimans2022progressive} & 2 & 15.5 & 0.59 & 0.36 \\
        \textbf{CD}$^\dagger$ & 1 & 11.0 & 0.65 & 0.36 \\
        \textbf{CD}$^\dagger$ & 2 & \textbf{8.84} & \textbf{0.66} & \textbf{0.40} \\
        \hline
        DDPM
        \cite{ho2020denoising}
        & 1000 & 17.1 & 0.53 & 0.48 \\
        ADM
        \cite{dhariwal2021diffusion}
        & 1000 & \textbf{5.57} & 0.63 & \textbf{0.52} \\
        EDM
        \cite{karras2022edm}
        & 79 & 6.69 & \textbf{0.70} & 0.43 \\
        PGGAN \cite{karras2018progressive} & 1 & 37.5 & & \\
        StyleGAN2 \cite{karras2020analyzing} & 1 & 7.25 & 0.58 & 0.43\\
        \textbf{CT} & 1 & 20.7 & 0.56 & 0.23 \\
        \textbf{CT} & 2 & 11.7 & 0.63 & 0.36
	\end{tabular}
    \end{adjustbox}
    }
\end{minipage}
\end{table*}

\subsection{Few-Step Image Generation}
\textbf{Distillation}~ In current literature, the most directly comparable approach to our consistency distillation (CD) is progressive distillation (PD, \citet{salimans2022progressive}); both are thus far the only distillation approaches that \emph{do not construct synthetic data before distillation}. In stark contrast, other distillation techniques, such as knowledge distillation \cite{luhman2021knowledge} and DFNO \cite{zheng2022fast}, have to prepare a large synthetic dataset by generating numerous samples from the diffusion model with expensive numerical ODE/SDE solvers. We perform comprehensive comparison for PD and CD on CIFAR-10, ImageNet $64\times 64$, and LSUN $256\times 256$, with all results reported in \cref{fig:iterative_sampling}. All methods distill from an EDM \cite{karras2022edm} model that we pre-trained in-house. We note that across all sampling iterations, \emph{using the LPIPS metric uniformly improves PD compared to the squared $\ell_2$ distance in the original paper of \citet{salimans2022progressive}}. Both PD and CD improve as we take more sampling steps. We find that CD uniformly outperforms PD across all datasets, sampling steps, and metric functions considered, except for single-step generation on Bedroom $256\times 256$, where CD with $\ell_2$ slightly underperforms PD with $\ell_2$. As shown in \cref{tab:results}, CD even outperforms distillation approaches that require synthetic dataset construction, such as Knowledge Distillation \cite{luhman2021knowledge} and DFNO \cite{zheng2022fast}.

\begin{figure*}
    \centering
    \begin{subfigure}[b]{0.33\textwidth}
        \includegraphics[width=\textwidth]{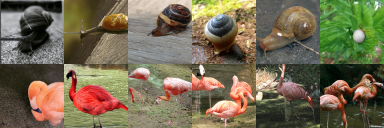}
    \end{subfigure}\hfill
    \begin{subfigure}[b]{0.33\textwidth}
        \includegraphics[width=\textwidth]{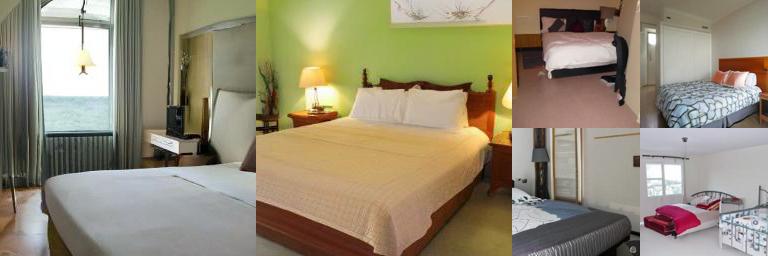}
    \end{subfigure}\hfill
    \begin{subfigure}[b]{0.33\textwidth}
        \includegraphics[width=\textwidth]{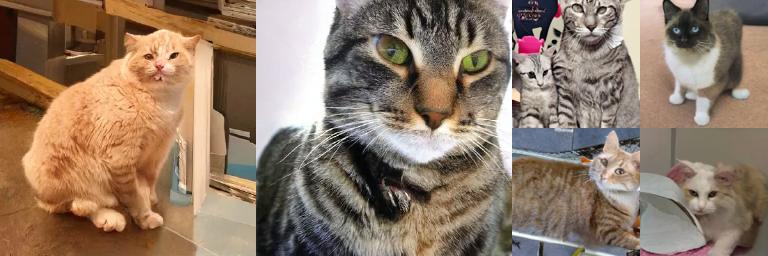}
    \end{subfigure}
    \begin{subfigure}[b]{0.33\textwidth}
        \includegraphics[width=\textwidth]{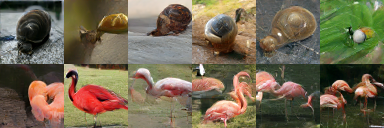}
    \end{subfigure}\hfill
    \begin{subfigure}[b]{0.33\textwidth}
        \includegraphics[width=\textwidth]{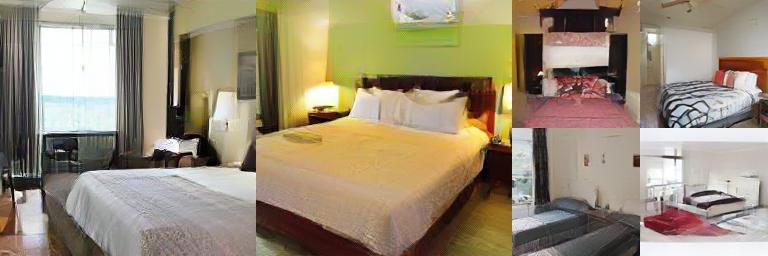}
    \end{subfigure}\hfill
    \begin{subfigure}[b]{0.33\textwidth}
        \includegraphics[width=\textwidth]{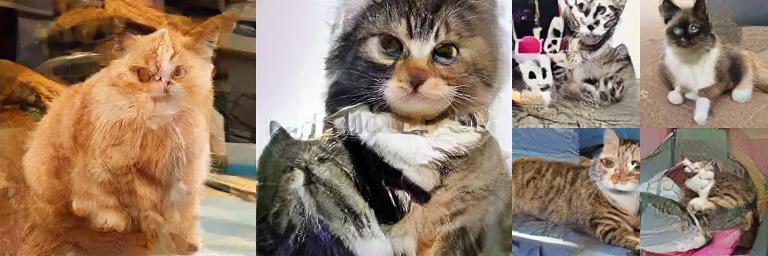}
    \end{subfigure}
    \begin{subfigure}[b]{0.33\textwidth}
        \includegraphics[width=\textwidth]{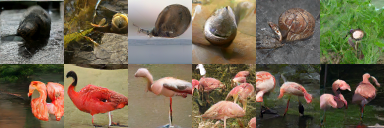}
    \end{subfigure}\hfill
    \begin{subfigure}[b]{0.33\textwidth}
        \includegraphics[width=\textwidth]{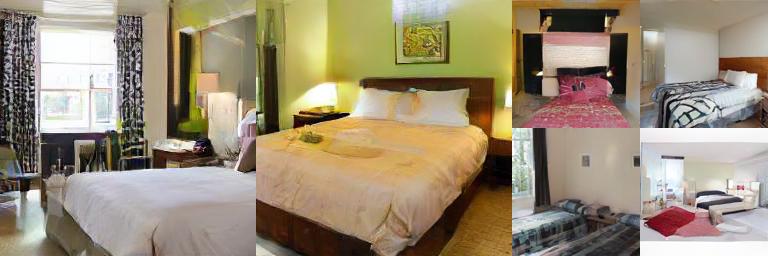}
    \end{subfigure}\hfill
    \begin{subfigure}[b]{0.33\textwidth}
        \includegraphics[width=\textwidth]{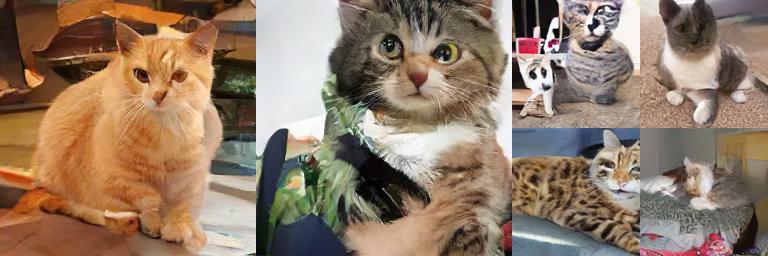}
    \end{subfigure}
    \caption{Samples generated by EDM (\emph{top}), CT + single-step generation (\emph{middle}), and CT + 2-step generation (\emph{Bottom}). All corresponding images are generated from the same initial noise.}
    \label{fig:samples}
\end{figure*}

\begin{figure*}
    \centering
    \begin{subfigure}[b]{\textwidth}
        \includegraphics[width=0.11\textwidth]{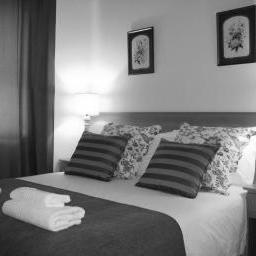}\hfill
        \includegraphics[width=0.77\textwidth]{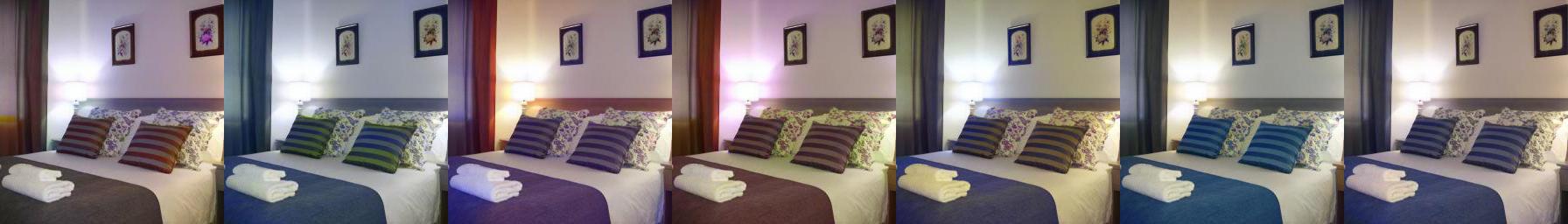}\hfill
        \includegraphics[width=0.11\textwidth]{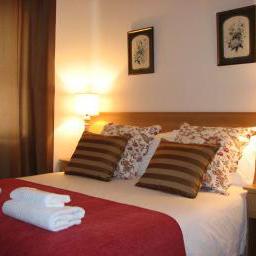}
        \caption{\emph{Left}: The gray-scale image. \emph{Middle}: Colorized images. \emph{Right}: The ground-truth image.}
        \label{fig:bedroom_colorization_lite}
    \end{subfigure}
    \begin{subfigure}[b]{\textwidth}
        \includegraphics[width=0.11\textwidth]{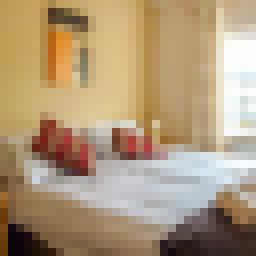}\hfill
        \includegraphics[width=0.77\textwidth]{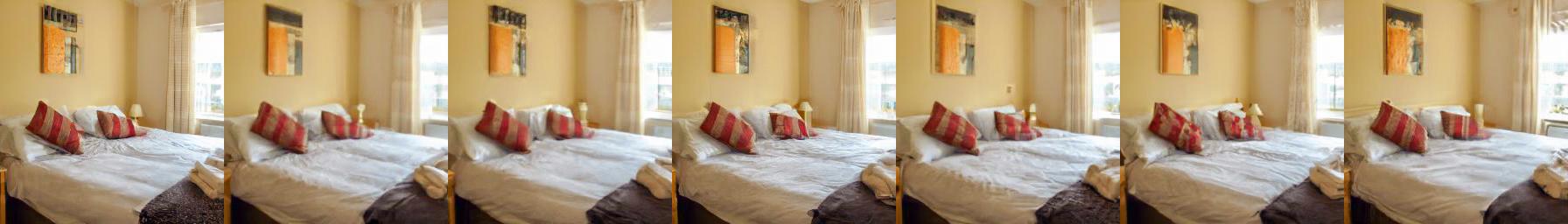}\hfill
        \includegraphics[width=0.11\textwidth]{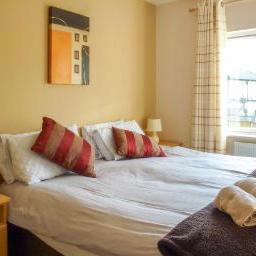}
        \caption{\emph{Left}: The downsampled image ($32\times 32$). \emph{Middle}: Full resolution images ($256\times 256$). \emph{Right}: The ground-truth image ($256\times 256$).}
        \label{fig:bedroom_superres_lite}
    \end{subfigure}
    \begin{subfigure}[b]{\textwidth}
        \includegraphics[width=0.1105\textwidth]{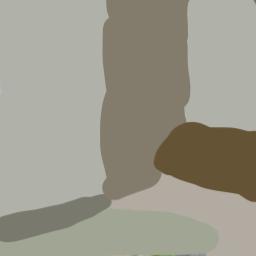}\hfill
        \includegraphics[width=0.884\textwidth]{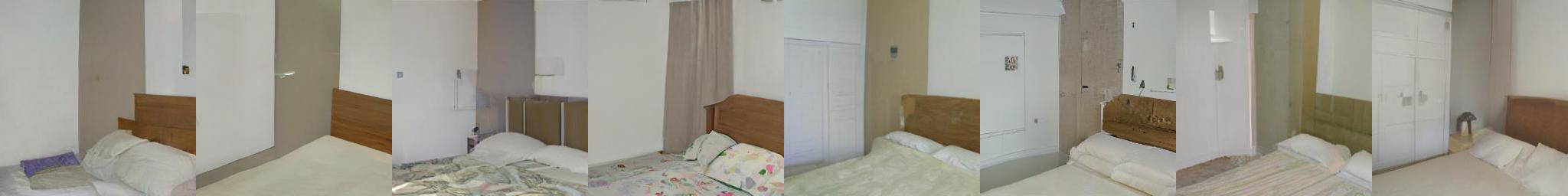}
        \caption{\emph{Left}: A stroke input provided by users. \emph{Right}: Stroke-guided image generation.}
        \label{fig:bedroom_sdedit_lite}
    \end{subfigure}
    \caption{Zero-shot image editing with a consistency model trained by consistency distillation on LSUN Bedroom $256\times 256$.}
\end{figure*}

\textbf{Direct Generation}~ In \cref{tab:results,tab:results2}, we compare the sample quality of consistency training (CT) with other generative models using one-step and two-step generation. We also include PD and CD results for reference. Both tables report PD results obtained from the $\ell_2$ metric function, as this is the default setting used in the original paper of \citet{salimans2022progressive}. For fair comparison, we ensure PD and CD distill the same EDM models. In \cref{tab:results,tab:results2}, we observe that CT outperforms existing single-step, non-adversarial generative models, \ie, VAEs and normalizing flows, by a significant margin on CIFAR-10.
Moreover, \emph{CT achieves comparable quality to one-step samples from PD without relying on distillation}. In \cref{fig:samples}, we provide EDM samples (top), single-step CT samples (middle), and two-step CT samples (bottom). In \cref{app:samples}, we show additional samples for both CD and CT in \cref{fig:cifar10_full_cd,fig:imagenet64_full_cd,fig:bedroom_full_cd,fig:cat_full_cd,fig:cifar10_full,fig:imagenet64_full,fig:bedroom_full,fig:cat_full}. Importantly, \emph{all samples obtained from the same initial noise vector share significant structural similarity}, even though CT and EDM models are trained independently from one another. This indicates that CT is less likely to suffer from mode collapse, as EDMs do not.

\subsection{Zero-Shot Image Editing}\label{sec:zeroshot}
Similar to diffusion models, consistency models allow zero-shot image editing by modifying the multistep sampling process in \cref{alg:sampling}. We demonstrate this capability with a consistency model trained on the LSUN bedroom dataset using consistency distillation. In \cref{fig:bedroom_colorization_lite}, we show such a consistency model can colorize gray-scale bedroom images at test time, even though it has never been trained on colorization tasks. In \cref{fig:bedroom_superres_lite}, we show the same consistency model can generate high-resolution images from low-resolution inputs. In \cref{fig:bedroom_sdedit_lite}, we additionally demonstrate that it can generate images based on stroke inputs created by humans, as in SDEdit for diffusion models \cite{meng2021sdedit}. Again, this editing capability is zero-shot, as the model has not been trained on stroke inputs. In \cref{app:editing}, we additionally demonstrate the zero-shot capability of consistency models on inpainting (\cref{fig:bedroom_inpainting}), interpolation (\cref{fig:bedroom_interp}) and denoising (\cref{fig:bedroom_denoising}), with more examples on colorization (\cref{fig:bedroom_colorization}), super-resolution (\cref{fig:bedroom_superres}) and stroke-guided image generation (\cref{fig:bedroom_sdedit}).

\section{Conclusion}\label{sec:conclusion}
We have introduced consistency models, a type of generative models that are specifically designed to support one-step and few-step generation. We have empirically demonstrated that our consistency distillation method outshines the existing distillation techniques for diffusion models on multiple image benchmarks and small sampling iterations. Furthermore, as a standalone generative model, consistency models generate better samples than existing single-step generation models except for GANs. Similar to diffusion models, they also allow zero-shot image editing applications such as inpainting, colorization, super-resolution, denoising, interpolation, and stroke-guided image generation.

In addition, consistency models share striking similarities with techniques employed in other fields, including deep Q-learning \cite{mnih2015human} and momentum-based contrastive learning \cite{grill2020bootstrap,he2020momentum}. This offers exciting prospects for cross-pollination of ideas and methods among these diverse fields.

\section*{Acknowledgements}
We thank Alex Nichol for reviewing the manuscript and providing valuable feedback, Chenlin Meng for providing stroke inputs needed in our stroke-guided image generation experiments, and the OpenAI Algorithms team.

\bibliography{bibliography}

\begin{thebibliography}{74}
\providecommand{\natexlab}[1]{#1}
\providecommand{\url}[1]{\texttt{#1}}
\expandafter\ifx\csname urlstyle\endcsname\relax
  \providecommand{\doi}[1]{doi: #1}\else
  \providecommand{\doi}{doi: \begingroup \urlstyle{rm}\Url}\fi

\bibitem[Balaji et~al.(2022)Balaji, Nah, Huang, Vahdat, Song, Kreis, Aittala,
  Aila, Laine, Catanzaro, Karras, and Liu]{balaji2022eDiff-I}
Balaji, Y., Nah, S., Huang, X., Vahdat, A., Song, J., Kreis, K., Aittala, M.,
  Aila, T., Laine, S., Catanzaro, B., Karras, T., and Liu, M.-Y.
\newblock ediff-i: Text-to-image diffusion models with ensemble of expert
  denoisers.
\newblock \emph{arXiv preprint arXiv:2211.01324}, 2022.

\bibitem[Bilo{\v{s}} et~al.(2021)Bilo{\v{s}}, Sommer, Rangapuram, Januschowski,
  and G{\"u}nnemann]{bilovs2021neural}
Bilo{\v{s}}, M., Sommer, J., Rangapuram, S.~S., Januschowski, T., and
  G{\"u}nnemann, S.
\newblock Neural flows: Efficient alternative to neural odes.
\newblock \emph{Advances in Neural Information Processing Systems},
  34:\penalty0 21325--21337, 2021.

\bibitem[Brock et~al.(2019)Brock, Donahue, and Simonyan]{brock2018large}
Brock, A., Donahue, J., and Simonyan, K.
\newblock Large scale {GAN} training for high fidelity natural image synthesis.
\newblock In \emph{International Conference on Learning Representations}, 2019.
\newblock URL \url{https://openreview.net/forum?id=B1xsqj09Fm}.

\bibitem[Chen et~al.(2021)Chen, Zhang, Zen, Weiss, Norouzi, and
  Chan]{chen2021wavegrad}
Chen, N., Zhang, Y., Zen, H., Weiss, R.~J., Norouzi, M., and Chan, W.
\newblock Wavegrad: Estimating gradients for waveform generation.
\newblock In \emph{International Conference on Learning Representations
  (ICLR)}, 2021.

\bibitem[Chen et~al.(2018)Chen, Rubanova, Bettencourt, and
  Duvenaud]{chen2018neural}
Chen, R.~T., Rubanova, Y., Bettencourt, J., and Duvenaud, D.~K.
\newblock Neural {O}rdinary {D}ifferential {E}quations.
\newblock In \emph{Advances in neural information processing systems}, pp.\
  6571--6583, 2018.

\bibitem[Chen et~al.(2019)Chen, Behrmann, Duvenaud, and
  Jacobsen]{chen2019residual}
Chen, R.~T., Behrmann, J., Duvenaud, D.~K., and Jacobsen, J.-H.
\newblock Residual flows for invertible generative modeling.
\newblock In \emph{Advances in Neural Information Processing Systems}, pp.\
  9916--9926, 2019.

\bibitem[Chung et~al.(2023)Chung, Kim, Mccann, Klasky, and
  Ye]{chung2023diffusion}
Chung, H., Kim, J., Mccann, M.~T., Klasky, M.~L., and Ye, J.~C.
\newblock Diffusion posterior sampling for general noisy inverse problems.
\newblock In \emph{International Conference on Learning Representations}, 2023.
\newblock URL \url{https://openreview.net/forum?id=OnD9zGAGT0k}.

\bibitem[Deng et~al.(2009)Deng, Dong, Socher, Li, Li, and
  Fei-Fei]{deng2009imagenet}
Deng, J., Dong, W., Socher, R., Li, L.-J., Li, K., and Fei-Fei, L.
\newblock Imagenet: A large-scale hierarchical image database.
\newblock In \emph{2009 IEEE conference on computer vision and pattern
  recognition}, pp.\  248--255. Ieee, 2009.

\bibitem[Dhariwal \& Nichol(2021)Dhariwal and Nichol]{dhariwal2021diffusion}
Dhariwal, P. and Nichol, A.
\newblock Diffusion models beat gans on image synthesis.
\newblock \emph{Advances in Neural Information Processing Systems (NeurIPS)},
  2021.

\bibitem[Dinh et~al.(2015)Dinh, Krueger, and Bengio]{dinh2014nice}
Dinh, L., Krueger, D., and Bengio, Y.
\newblock {NICE}: Non-linear independent components estimation.
\newblock \emph{International Conference in Learning Representations Workshop
  Track}, 2015.

\bibitem[Dinh et~al.(2017)Dinh, Sohl{-}Dickstein, and Bengio]{dinh2016density}
Dinh, L., Sohl{-}Dickstein, J., and Bengio, S.
\newblock Density estimation using real {NVP}.
\newblock In \emph{5th International Conference on Learning Representations,
  {ICLR} 2017, Toulon, France, April 24-26, 2017, Conference Track
  Proceedings}. OpenReview.net, 2017.
\newblock URL \url{https://openreview.net/forum?id=HkpbnH9lx}.

\bibitem[Dockhorn et~al.(2022)Dockhorn, Vahdat, and Kreis]{dockhorn2022genie}
Dockhorn, T., Vahdat, A., and Kreis, K.
\newblock Genie: Higher-order denoising diffusion solvers.
\newblock \emph{arXiv preprint arXiv:2210.05475}, 2022.

\bibitem[Gong et~al.(2019)Gong, Chang, Jiang, and Wang]{gong2019autogan}
Gong, X., Chang, S., Jiang, Y., and Wang, Z.
\newblock Autogan: Neural architecture search for generative adversarial
  networks.
\newblock In \emph{Proceedings of the IEEE/CVF International Conference on
  Computer Vision}, pp.\  3224--3234, 2019.

\bibitem[Goodfellow et~al.(2014)Goodfellow, Pouget-Abadie, Mirza, Xu,
  Warde-Farley, Ozair, Courville, and Bengio]{goodfellow2014generative}
Goodfellow, I., Pouget-Abadie, J., Mirza, M., Xu, B., Warde-Farley, D., Ozair,
  S., Courville, A., and Bengio, Y.
\newblock Generative adversarial nets.
\newblock In \emph{Advances in neural information processing systems}, pp.\
  2672--2680, 2014.

\bibitem[Grci{\'c} et~al.(2021)Grci{\'c}, Grubi{\v{s}}i{\'c}, and
  {\v{S}}egvi{\'c}]{grcic2021densely}
Grci{\'c}, M., Grubi{\v{s}}i{\'c}, I., and {\v{S}}egvi{\'c}, S.
\newblock Densely connected normalizing flows.
\newblock \emph{Advances in Neural Information Processing Systems},
  34:\penalty0 23968--23982, 2021.

\bibitem[Grill et~al.(2020)Grill, Strub, Altch{\'e}, Tallec, Richemond,
  Buchatskaya, Doersch, Avila~Pires, Guo, Gheshlaghi~Azar,
  et~al.]{grill2020bootstrap}
Grill, J.-B., Strub, F., Altch{\'e}, F., Tallec, C., Richemond, P.,
  Buchatskaya, E., Doersch, C., Avila~Pires, B., Guo, Z., Gheshlaghi~Azar, M.,
  et~al.
\newblock Bootstrap your own latent-a new approach to self-supervised learning.
\newblock \emph{Advances in neural information processing systems},
  33:\penalty0 21271--21284, 2020.

\bibitem[He et~al.(2020)He, Fan, Wu, Xie, and Girshick]{he2020momentum}
He, K., Fan, H., Wu, Y., Xie, S., and Girshick, R.
\newblock Momentum contrast for unsupervised visual representation learning.
\newblock In \emph{Proceedings of the IEEE/CVF conference on computer vision
  and pattern recognition}, pp.\  9729--9738, 2020.

\bibitem[Heusel et~al.(2017)Heusel, Ramsauer, Unterthiner, Nessler, and
  Hochreiter]{heusel2017gans}
Heusel, M., Ramsauer, H., Unterthiner, T., Nessler, B., and Hochreiter, S.
\newblock {GANs} trained by a two time-scale update rule converge to a local
  {Nash} equilibrium.
\newblock In \emph{Advances in Neural Information Processing Systems}, pp.\
  6626--6637, 2017.

\bibitem[Ho et~al.(2020)Ho, Jain, and Abbeel]{ho2020denoising}
Ho, J., Jain, A., and Abbeel, P.
\newblock Denoising {D}iffusion {P}robabilistic {M}odels.
\newblock \emph{Advances in Neural Information Processing Systems}, 33, 2020.

\bibitem[Ho et~al.(2022{\natexlab{a}})Ho, Chan, Saharia, Whang, Gao, Gritsenko,
  Kingma, Poole, Norouzi, Fleet, et~al.]{ho2022imagen}
Ho, J., Chan, W., Saharia, C., Whang, J., Gao, R., Gritsenko, A., Kingma,
  D.~P., Poole, B., Norouzi, M., Fleet, D.~J., et~al.
\newblock Imagen video: High definition video generation with diffusion models.
\newblock \emph{arXiv preprint arXiv:2210.02303}, 2022{\natexlab{a}}.

\bibitem[Ho et~al.(2022{\natexlab{b}})Ho, Salimans, Gritsenko, Chan, Norouzi,
  and Fleet]{ho2022video}
Ho, J., Salimans, T., Gritsenko, A.~A., Chan, W., Norouzi, M., and Fleet, D.~J.
\newblock Video diffusion models.
\newblock In \emph{ICLR Workshop on Deep Generative Models for Highly
  Structured Data}, 2022{\natexlab{b}}.
\newblock URL \url{https://openreview.net/forum?id=BBelR2NdDZ5}.

\bibitem[Hyv{\"a}rinen \& Dayan(2005)Hyv{\"a}rinen and
  Dayan]{hyvarinen2005estimation}
Hyv{\"a}rinen, A. and Dayan, P.
\newblock Estimation of non-normalized statistical models by score matching.
\newblock \emph{Journal of Machine Learning Research (JMLR)}, 6\penalty0 (4),
  2005.

\bibitem[Jiang et~al.(2021)Jiang, Chang, and Wang]{jiang2021transgan}
Jiang, Y., Chang, S., and Wang, Z.
\newblock Transgan: Two pure transformers can make one strong gan, and that can
  scale up.
\newblock \emph{Advances in Neural Information Processing Systems},
  34:\penalty0 14745--14758, 2021.

\bibitem[Karras et~al.(2018)Karras, Aila, Laine, and
  Lehtinen]{karras2018progressive}
Karras, T., Aila, T., Laine, S., and Lehtinen, J.
\newblock Progressive growing of {GAN}s for improved quality, stability, and
  variation.
\newblock In \emph{International Conference on Learning Representations}, 2018.
\newblock URL \url{https://openreview.net/forum?id=Hk99zCeAb}.

\bibitem[Karras et~al.(2020)Karras, Laine, Aittala, Hellsten, Lehtinen, and
  Aila]{karras2020analyzing}
Karras, T., Laine, S., Aittala, M., Hellsten, J., Lehtinen, J., and Aila, T.
\newblock Analyzing and improving the image quality of stylegan.
\newblock 2020.

\bibitem[Karras et~al.(2022)Karras, Aittala, Aila, and Laine]{karras2022edm}
Karras, T., Aittala, M., Aila, T., and Laine, S.
\newblock Elucidating the design space of diffusion-based generative models.
\newblock In \emph{Proc. NeurIPS}, 2022.

\bibitem[Kawar et~al.(2021)Kawar, Vaksman, and Elad]{kawar2021snips}
Kawar, B., Vaksman, G., and Elad, M.
\newblock Snips: Solving noisy inverse problems stochastically.
\newblock \emph{arXiv preprint arXiv:2105.14951}, 2021.

\bibitem[Kawar et~al.(2022)Kawar, Elad, Ermon, and Song]{kawar2022denoising}
Kawar, B., Elad, M., Ermon, S., and Song, J.
\newblock Denoising diffusion restoration models.
\newblock In \emph{Advances in Neural Information Processing Systems}, 2022.

\bibitem[Kingma \& Dhariwal(2018)Kingma and Dhariwal]{kingma2018glow}
Kingma, D.~P. and Dhariwal, P.
\newblock Glow: Generative flow with invertible 1x1 convolutions.
\newblock In Bengio, S., Wallach, H., Larochelle, H., Grauman, K.,
  Cesa-Bianchi, N., and Garnett, R. (eds.), \emph{Advances in Neural
  Information Processing Systems 31}, pp.\  10215--10224. 2018.

\bibitem[Kingma \& Welling(2014)Kingma and Welling]{kingma2013auto}
Kingma, D.~P. and Welling, M.
\newblock Auto-encoding variational bayes.
\newblock In \emph{International Conference on Learning Representations}, 2014.

\bibitem[Kong et~al.(2020)Kong, Ping, Huang, Zhao, and
  Catanzaro]{kong2020diffwave}
Kong, Z., Ping, W., Huang, J., Zhao, K., and Catanzaro, B.
\newblock Diff{W}ave: {A} {V}ersatile {D}iffusion {M}odel for {A}udio
  {S}ynthesis.
\newblock \emph{arXiv preprint arXiv:2009.09761}, 2020.

\bibitem[Krizhevsky et~al.(2009)Krizhevsky, Hinton,
  et~al.]{krizhevsky2009learning}
Krizhevsky, A., Hinton, G., et~al.
\newblock Learning multiple layers of features from tiny images.
\newblock 2009.

\bibitem[Kynk{\"a}{\"a}nniemi et~al.(2019)Kynk{\"a}{\"a}nniemi, Karras, Laine,
  Lehtinen, and Aila]{kynkaanniemi2019improved}
Kynk{\"a}{\"a}nniemi, T., Karras, T., Laine, S., Lehtinen, J., and Aila, T.
\newblock Improved precision and recall metric for assessing generative models.
\newblock \emph{Advances in Neural Information Processing Systems}, 32, 2019.

\bibitem[Lee et~al.(2021)Lee, Chang, Jiang, Zhang, Tu, and Liu]{lee2021vitgan}
Lee, K., Chang, H., Jiang, L., Zhang, H., Tu, Z., and Liu, C.
\newblock Vitgan: Training gans with vision transformers.
\newblock \emph{arXiv preprint arXiv:2107.04589}, 2021.

\bibitem[Lillicrap et~al.(2015)Lillicrap, Hunt, Pritzel, Heess, Erez, Tassa,
  Silver, and Wierstra]{lillicrap2015continuous}
Lillicrap, T.~P., Hunt, J.~J., Pritzel, A., Heess, N., Erez, T., Tassa, Y.,
  Silver, D., and Wierstra, D.
\newblock Continuous control with deep reinforcement learning.
\newblock \emph{arXiv preprint arXiv:1509.02971}, 2015.

\bibitem[Liu et~al.(2019)Liu, Jiang, He, Chen, Liu, Gao, and
  Han]{liu2019variance}
Liu, L., Jiang, H., He, P., Chen, W., Liu, X., Gao, J., and Han, J.
\newblock On the variance of the adaptive learning rate and beyond.
\newblock \emph{arXiv preprint arXiv:1908.03265}, 2019.

\bibitem[Liu et~al.(2022)Liu, Gong, and Liu]{liu2022flow}
Liu, X., Gong, C., and Liu, Q.
\newblock Flow straight and fast: Learning to generate and transfer data with
  rectified flow.
\newblock \emph{arXiv preprint arXiv:2209.03003}, 2022.

\bibitem[Lu et~al.(2022)Lu, Zhou, Bao, Chen, Li, and Zhu]{lu2022dpm}
Lu, C., Zhou, Y., Bao, F., Chen, J., Li, C., and Zhu, J.
\newblock Dpm-solver: A fast ode solver for diffusion probabilistic model
  sampling in around 10 steps.
\newblock \emph{arXiv preprint arXiv:2206.00927}, 2022.

\bibitem[Luhman \& Luhman(2021)Luhman and Luhman]{luhman2021knowledge}
Luhman, E. and Luhman, T.
\newblock Knowledge distillation in iterative generative models for improved
  sampling speed.
\newblock \emph{arXiv preprint arXiv:2101.02388}, 2021.

\bibitem[Meng et~al.(2021)Meng, Song, Song, Wu, Zhu, and Ermon]{meng2021sdedit}
Meng, C., Song, Y., Song, J., Wu, J., Zhu, J.-Y., and Ermon, S.
\newblock Sdedit: Image synthesis and editing with stochastic differential
  equations.
\newblock \emph{arXiv preprint arXiv:2108.01073}, 2021.

\bibitem[Meng et~al.(2022)Meng, Gao, Kingma, Ermon, Ho, and
  Salimans]{meng2022distillation}
Meng, C., Gao, R., Kingma, D.~P., Ermon, S., Ho, J., and Salimans, T.
\newblock On distillation of guided diffusion models.
\newblock \emph{arXiv preprint arXiv:2210.03142}, 2022.

\bibitem[Mnih et~al.(2013)Mnih, Kavukcuoglu, Silver, Graves, Antonoglou,
  Wierstra, and Riedmiller]{mnih2013playing}
Mnih, V., Kavukcuoglu, K., Silver, D., Graves, A., Antonoglou, I., Wierstra,
  D., and Riedmiller, M.
\newblock Playing atari with deep reinforcement learning.
\newblock \emph{arXiv preprint arXiv:1312.5602}, 2013.

\bibitem[Mnih et~al.(2015)Mnih, Kavukcuoglu, Silver, Rusu, Veness, Bellemare,
  Graves, Riedmiller, Fidjeland, Ostrovski, et~al.]{mnih2015human}
Mnih, V., Kavukcuoglu, K., Silver, D., Rusu, A.~A., Veness, J., Bellemare,
  M.~G., Graves, A., Riedmiller, M., Fidjeland, A.~K., Ostrovski, G., et~al.
\newblock Human-level control through deep reinforcement learning.
\newblock \emph{nature}, 518\penalty0 (7540):\penalty0 529--533, 2015.

\bibitem[Nichol et~al.(2021)Nichol, Dhariwal, Ramesh, Shyam, Mishkin, McGrew,
  Sutskever, and Chen]{nichol2021glide}
Nichol, A., Dhariwal, P., Ramesh, A., Shyam, P., Mishkin, P., McGrew, B.,
  Sutskever, I., and Chen, M.
\newblock Glide: Towards photorealistic image generation and editing with
  text-guided diffusion models.
\newblock \emph{arXiv preprint arXiv:2112.10741}, 2021.

\bibitem[Parmar et~al.(2021)Parmar, Li, Lee, and Tu]{parmar2021dual}
Parmar, G., Li, D., Lee, K., and Tu, Z.
\newblock Dual contradistinctive generative autoencoder.
\newblock In \emph{Proceedings of the IEEE/CVF Conference on Computer Vision
  and Pattern Recognition}, pp.\  823--832, 2021.

\bibitem[Popov et~al.(2021)Popov, Vovk, Gogoryan, Sadekova, and
  Kudinov]{popov2021grad}
Popov, V., Vovk, I., Gogoryan, V., Sadekova, T., and Kudinov, M.
\newblock Grad-{TTS}: A diffusion probabilistic model for text-to-speech.
\newblock \emph{arXiv preprint arXiv:2105.06337}, 2021.

\bibitem[Ramesh et~al.(2022)Ramesh, Dhariwal, Nichol, Chu, and
  Chen]{ramesh2022hierarchical}
Ramesh, A., Dhariwal, P., Nichol, A., Chu, C., and Chen, M.
\newblock Hierarchical text-conditional image generation with clip latents.
\newblock \emph{arXiv preprint arXiv:2204.06125}, 2022.

\bibitem[Rezende et~al.(2014)Rezende, Mohamed, and
  Wierstra]{rezende2014stochastic}
Rezende, D.~J., Mohamed, S., and Wierstra, D.
\newblock Stochastic backpropagation and approximate inference in deep
  generative models.
\newblock In \emph{Proceedings of the 31st International Conference on Machine
  Learning}, pp.\  1278--1286, 2014.

\bibitem[Rombach et~al.(2022)Rombach, Blattmann, Lorenz, Esser, and
  Ommer]{rombach2022high}
Rombach, R., Blattmann, A., Lorenz, D., Esser, P., and Ommer, B.
\newblock High-resolution image synthesis with latent diffusion models.
\newblock In \emph{Proceedings of the IEEE/CVF Conference on Computer Vision
  and Pattern Recognition}, pp.\  10684--10695, 2022.

\bibitem[Saharia et~al.(2022)Saharia, Chan, Saxena, Li, Whang, Denton,
  Ghasemipour, Ayan, Mahdavi, Lopes, et~al.]{saharia2022photorealistic}
Saharia, C., Chan, W., Saxena, S., Li, L., Whang, J., Denton, E., Ghasemipour,
  S. K.~S., Ayan, B.~K., Mahdavi, S.~S., Lopes, R.~G., et~al.
\newblock Photorealistic text-to-image diffusion models with deep language
  understanding.
\newblock \emph{arXiv preprint arXiv:2205.11487}, 2022.

\bibitem[Salimans \& Ho(2022)Salimans and Ho]{salimans2022progressive}
Salimans, T. and Ho, J.
\newblock Progressive distillation for fast sampling of diffusion models.
\newblock In \emph{International Conference on Learning Representations}, 2022.
\newblock URL \url{https://openreview.net/forum?id=TIdIXIpzhoI}.

\bibitem[Salimans et~al.(2016)Salimans, Goodfellow, Zaremba, Cheung, Radford,
  and Chen]{salimans2016improved}
Salimans, T., Goodfellow, I., Zaremba, W., Cheung, V., Radford, A., and Chen,
  X.
\newblock Improved techniques for training gans.
\newblock In \emph{Advances in neural information processing systems}, pp.\
  2234--2242, 2016.

\bibitem[Sauer et~al.(2022)Sauer, Schwarz, and Geiger]{sauer2022stylegan}
Sauer, A., Schwarz, K., and Geiger, A.
\newblock Stylegan-xl: Scaling stylegan to large diverse datasets.
\newblock In \emph{ACM SIGGRAPH 2022 conference proceedings}, pp.\  1--10,
  2022.

\bibitem[Sohl-Dickstein et~al.(2015)Sohl-Dickstein, Weiss, Maheswaranathan, and
  Ganguli]{sohl2015deep}
Sohl-Dickstein, J., Weiss, E., Maheswaranathan, N., and Ganguli, S.
\newblock Deep {U}nsupervised {L}earning {U}sing {N}onequilibrium
  {T}hermodynamics.
\newblock In \emph{International Conference on Machine Learning}, pp.\
  2256--2265, 2015.

\bibitem[Song et~al.(2020)Song, Meng, and Ermon]{song2020denoising}
Song, J., Meng, C., and Ermon, S.
\newblock Denoising diffusion implicit models.
\newblock \emph{arXiv preprint arXiv:2010.02502}, 2020.

\bibitem[Song et~al.(2023)Song, Vahdat, Mardani, and
  Kautz]{song2023pseudoinverseguided}
Song, J., Vahdat, A., Mardani, M., and Kautz, J.
\newblock Pseudoinverse-guided diffusion models for inverse problems.
\newblock In \emph{International Conference on Learning Representations}, 2023.
\newblock URL \url{https://openreview.net/forum?id=9_gsMA8MRKQ}.

\bibitem[Song \& Ermon(2019)Song and Ermon]{song2019generative}
Song, Y. and Ermon, S.
\newblock Generative {M}odeling by {E}stimating {G}radients of the {D}ata
  {D}istribution.
\newblock In \emph{Advances in Neural Information Processing Systems}, pp.\
  11918--11930, 2019.

\bibitem[Song \& Ermon(2020)Song and Ermon]{song2020improved}
Song, Y. and Ermon, S.
\newblock Improved {T}echniques for {T}raining {S}core-{B}ased {G}enerative
  {M}odels.
\newblock \emph{Advances in Neural Information Processing Systems}, 33, 2020.

\bibitem[Song et~al.(2019)Song, Garg, Shi, and Ermon]{song2019sliced}
Song, Y., Garg, S., Shi, J., and Ermon, S.
\newblock Sliced score matching: {A} scalable approach to density and score
  estimation.
\newblock In \emph{Proceedings of the Thirty-Fifth Conference on Uncertainty in
  Artificial Intelligence, {UAI} 2019, Tel Aviv, Israel, July 22-25, 2019},
  pp.\  204, 2019.

\bibitem[Song et~al.(2021)Song, Sohl-Dickstein, Kingma, Kumar, Ermon, and
  Poole]{song2021scorebased}
Song, Y., Sohl-Dickstein, J., Kingma, D.~P., Kumar, A., Ermon, S., and Poole,
  B.
\newblock Score-based generative modeling through stochastic differential
  equations.
\newblock In \emph{International Conference on Learning Representations}, 2021.
\newblock URL \url{https://openreview.net/forum?id=PxTIG12RRHS}.

\bibitem[Song et~al.(2022)Song, Shen, Xing, and Ermon]{song2021medical}
Song, Y., Shen, L., Xing, L., and Ermon, S.
\newblock Solving inverse problems in medical imaging with score-based
  generative models.
\newblock In \emph{International Conference on Learning Representations}, 2022.
\newblock URL \url{https://openreview.net/forum?id=vaRCHVj0uGI}.

\bibitem[S{\"u}li \& Mayers(2003)S{\"u}li and Mayers]{suli2003introduction}
S{\"u}li, E. and Mayers, D.~F.
\newblock \emph{An introduction to numerical analysis}.
\newblock Cambridge university press, 2003.

\bibitem[Tian et~al.(2020)Tian, Wang, Huang, Li, Dai, Yang, Wang, and
  Fink]{tian2020off}
Tian, Y., Wang, Q., Huang, Z., Li, W., Dai, D., Yang, M., Wang, J., and Fink,
  O.
\newblock Off-policy reinforcement learning for efficient and effective gan
  architecture search.
\newblock In \emph{Computer Vision--ECCV 2020: 16th European Conference,
  Glasgow, UK, August 23--28, 2020, Proceedings, Part VII 16}, pp.\  175--192.
  Springer, 2020.

\bibitem[Vahdat et~al.(2021)Vahdat, Kreis, and Kautz]{vahdat2021score}
Vahdat, A., Kreis, K., and Kautz, J.
\newblock Score-based generative modeling in latent space.
\newblock \emph{Advances in Neural Information Processing Systems},
  34:\penalty0 11287--11302, 2021.

\bibitem[Vincent(2011)]{vincent2011connection}
Vincent, P.
\newblock A {C}onnection {B}etween {S}core {M}atching and {D}enoising
  {A}utoencoders.
\newblock \emph{Neural Computation}, 23\penalty0 (7):\penalty0 1661--1674,
  2011.

\bibitem[Wu et~al.(2019)Wu, Huang, Acharya, Li, Thoma, Paudel, and
  Gool]{wu2019sliced}
Wu, J., Huang, Z., Acharya, D., Li, W., Thoma, J., Paudel, D.~P., and Gool,
  L.~V.
\newblock Sliced wasserstein generative models.
\newblock In \emph{Proceedings of the IEEE/CVF Conference on Computer Vision
  and Pattern Recognition}, pp.\  3713--3722, 2019.

\bibitem[Xiao et~al.(2019)Xiao, Yan, and Amit]{xiao2019generative}
Xiao, Z., Yan, Q., and Amit, Y.
\newblock Generative latent flow.
\newblock \emph{arXiv preprint arXiv:1905.10485}, 2019.

\bibitem[Xiao et~al.(2022)Xiao, Kreis, and Vahdat]{xiao2022tackling}
Xiao, Z., Kreis, K., and Vahdat, A.
\newblock Tackling the generative learning trilemma with denoising diffusion
  {GAN}s.
\newblock In \emph{International Conference on Learning Representations}, 2022.
\newblock URL \url{https://openreview.net/forum?id=JprM0p-q0Co}.

\bibitem[Xu et~al.(2022)Xu, Liu, Tegmark, and Jaakkola]{xu2022poisson}
Xu, Y., Liu, Z., Tegmark, M., and Jaakkola, T.~S.
\newblock Poisson flow generative models.
\newblock In Oh, A.~H., Agarwal, A., Belgrave, D., and Cho, K. (eds.),
  \emph{Advances in Neural Information Processing Systems}, 2022.
\newblock URL \url{https://openreview.net/forum?id=voV_TRqcWh}.

\bibitem[Yu et~al.(2015)Yu, Seff, Zhang, Song, Funkhouser, and
  Xiao]{yu2015lsun}
Yu, F., Seff, A., Zhang, Y., Song, S., Funkhouser, T., and Xiao, J.
\newblock Lsun: Construction of a large-scale image dataset using deep learning
  with humans in the loop.
\newblock \emph{arXiv preprint arXiv:1506.03365}, 2015.

\bibitem[Zhang \& Chen(2022)Zhang and Chen]{zhang2022fast}
Zhang, Q. and Chen, Y.
\newblock Fast sampling of diffusion models with exponential integrator.
\newblock \emph{arXiv preprint arXiv:2204.13902}, 2022.

\bibitem[Zhang et~al.(2018)Zhang, Isola, Efros, Shechtman, and
  Wang]{zhang2018perceptual}
Zhang, R., Isola, P., Efros, A.~A., Shechtman, E., and Wang, O.
\newblock The unreasonable effectiveness of deep features as a perceptual
  metric.
\newblock In \emph{CVPR}, 2018.

\bibitem[Zheng et~al.(2022)Zheng, Nie, Vahdat, Azizzadenesheli, and
  Anandkumar]{zheng2022fast}
Zheng, H., Nie, W., Vahdat, A., Azizzadenesheli, K., and Anandkumar, A.
\newblock Fast sampling of diffusion models via operator learning.
\newblock \emph{arXiv preprint arXiv:2211.13449}, 2022.

\bibitem[Zheng et~al.(2023)Zheng, He, Chen, and Zhou]{zheng2023truncated}
Zheng, H., He, P., Chen, W., and Zhou, M.
\newblock Truncated diffusion probabilistic models and diffusion-based
  adversarial auto-encoders.
\newblock In \emph{The Eleventh International Conference on Learning
  Representations}, 2023.
\newblock URL \url{https://openreview.net/forum?id=HDxgaKk956l}.

\end{thebibliography}
\bibliographystyle{icml2023}

\newpage
\appendix
\onecolumn
\setcounter{tocdepth}{4}
\tableofcontents
\allowdisplaybreaks
\begin{appendices}
\section{Proofs}\label{app:proof}
\subsection{Notations}
We use $\vf_{\vtheta}(\rvx, t)$ to denote a consistency model parameterized by $\vtheta$, and $\vf(\rvx, t; \vphi)$ the consistency function of the empirical PF ODE in \cref{eq:e_pfode}. Here $\vphi$ symbolizes its dependency on the pre-trained score model $\vs_\vphi(\rvx, t)$. For the consistency function of the PF ODE in \cref{eq:pfode}, we denote it as $\vf(\rvx, t)$. Given a multi-variate function $\vh(\rvx, \rvy)$, we let $\partial_1 \vh(\rvx, \rvy)$ denote the Jacobian of $\vh$ over $\rvx$, and analogously $\partial_2 \vh(\rvx, \rvy)$ denote the Jacobian of $\vh$ over $\rvy$. Unless otherwise stated, $\rvx$ is supposed to be a random variable sampled from the data distribution $p_\text{data}(\rvx)$, $n$ is sampled uniformly at random from $\llbracket 1, N-1 \rrbracket$, and $\rvx_{t_{n}}$ is sampled from $\mcal{N}(\rvx; t_n^2 \mI)$. Here $\llbracket 1, N-1 \rrbracket$ represents the set of integers $\{1,2,\cdots, N-1\}$. Furthermore, recall that we define
\begin{align*}
    \hat{\rvx}_{t_n}^\vphi \coloneqq \rvx_{t_{n+1}} + (t_n - t_{n+1})\Phi(\rvx_{t_{n+1}}, t_{n+1}; \vphi),
\end{align*}
where $\Phi(\cdots; \vphi)$ denotes the update function of a one-step ODE solver for the empirical PF ODE defined by the score model $\vs_\vphi(\rvx, t)$. By default, $\mbb{E}[\cdot]$ denotes the expectation over all relevant random variables in the expression.

\subsection{Consistency Distillation}\label{app:proof_cd}
\begin{customthm}{\ref{thm:convergence}}
Let $\Delta t \coloneqq \max_{n \in \llbracket 1, N-1\rrbracket}\{|t_{n+1} - t_{n}|\}$, and $\vf(\cdot,\cdot;\vphi)$ be the consistency function of the empirical PF ODE in \cref{eq:e_pfode}. Assume $\vf_\vtheta$ satisfies the Lipschitz condition: there exists $L > 0$ such that for all $t \in [\epsilon, T]$, $\rvx$, and $\rvy$, we have $\norm{\vf_\vtheta(\rvx, t) - \vf_\vtheta(\rvy, t)}_2 \leq L \norm{\rvx - \rvy}_2$. Assume further that for all $n \in \llbracket 1, N-1 \rrbracket$, the ODE solver called at $t_{n+1}$ has local error uniformly bounded by $O((t_{n+1} - t_n)^{p+1})$ with $p\geq 1$. Then, if $\mcal{L}_\text{CD}^N(\vtheta, \vtheta; \vphi) = 0$, we have
\begin{align*}
    \sup_{n, \rvx}\|\vf_{\vtheta}(\rvx, t_n) - \vf(\rvx, t_n; \vphi)\|_2 = O((\Delta t)^p).
\end{align*}
\end{customthm}
\begin{proof}
    From $\mcal{L}_\text{CD}^N(\vtheta, \vtheta; \vphi) = 0$, we have
    \begin{align}
        \mcal{L}_\text{CD}^N(\vtheta, \vtheta; \vphi) = \mbb{E}[\lambda(t_n) d(\vf_\vtheta({\rvx}_{t_{n+1}}, t_{n+1}), \vf_{\vtheta}(\hat{\rvx}_{t_n}^\vphi, t_n))] = 0.\label{eq:zero_loss}
    \end{align}
    According to the definition, we have $p_{t_n}(\rvx_{t_n}) = p_\text{data}(\rvx) \otimes \mcal{N}(\bm{0}, t_n^2 \mI)$ where $t_n \geq \epsilon > 0$. It follows that $p_{t_n}(\rvx_{t_n}) > 0$ for every $\rvx_{t_n}$ and $1 \leq n \leq N$. Therefore, \cref{eq:zero_loss} entails
    \begin{align}
        \lambda(t_n) d(\vf_\vtheta({\rvx}_{t_{n+1}}, t_{n+1}), \vf_{\vtheta}(\hat{\rvx}_{t_n}^\vphi, t_n)) \equiv 0.
    \end{align}
    Because $\lambda(\cdot) > 0$ and $d(\rvx, \rvy) = 0 \Leftrightarrow \rvx = \rvy$, this further implies that
    \begin{align}
        \vf_\vtheta(\rvx_{t_{n+1}}, t_{n+1}) \equiv \vf_{\vtheta}(\hat{\rvx}_{t_n}^\vphi, t_n).\label{eq:zero_loss_identity}
    \end{align}
    Now let $\ve_{n}$ represent the error vector at $t_n$, which is defined as
    \begin{align*}
        \ve_{n} \coloneqq \vf_\vtheta(\rvx_{t_{n}}, t_{n}) - \vf(\rvx_{t_n}, t_n; \vphi).
    \end{align*}
    We can easily derive the following recursion relation
    \begin{align}
        \ve_{n+1} &= \vf_\vtheta(\rvx_{t_{n+1}}, t_{n+1}) - \vf(\rvx_{t_{n+1}}, t_{n+1}; \vphi)\notag\\
        &\stackrel{(i)}{=} \vf_\vtheta(\hat{\rvx}_{t_{n}}^\vphi, t_{n}) - \vf(\rvx_{t_{n}}, t_{n}; \vphi)\notag\\
        &= \vf_\vtheta(\hat{\rvx}_{t_{n}}^\vphi, t_{n}) - \vf_\vtheta(\rvx_{t_n}, t_n) + \vf_\vtheta(\rvx_{t_n}, t_n) - \vf(\rvx_{t_{n}}, t_{n}; \vphi)\notag\\
        &= \vf_\vtheta(\hat{\rvx}_{t_{n}}^\vphi, t_{n}) - \vf_\vtheta(\rvx_{t_n}, t_n) + \ve_{n},\label{eq:recursion}
    \end{align}
    where (i) is due to \cref{eq:zero_loss_identity} and $\vf(\rvx_{t_{n+1}}, t_{n+1}; \vphi) =  \vf(\rvx_{t_{n}}, t_{n}; \vphi)$. Because $\vf_\vtheta(\cdot, t_n)$ has Lipschitz constant $L$, we have
    \begin{align*}
        \norm{\ve_{n+1}}_2 &\leq \norm{\ve_{n}}_2 + L \norm{\hat{\rvx}_{t_n}^\vphi - \rvx_{t_n}}_2\\
        &\stackrel{(i)}{=}  \norm{\ve_{n}}_2 + L\cdot O((t_{n+1} - t_n)^{p+1})\\
        &=\norm{\ve_{n}}_2 + O((t_{n+1} - t_n)^{p+1}),
    \end{align*}
    where (i) holds because the ODE solver has local error bounded by $O((t_{n+1}-t_n)^{p+1})$. In addition, we observe that $\ve_1 = \bm{0}$, because
    \begin{align*}
        \ve_1 &= \vf_\vtheta(\rvx_{t_1}, t_1) - \vf(\rvx_{t_1}, t_1; \vphi) \\
        &\stackrel{(i)}{=} \rvx_{t_1} - \vf(\rvx_{t_1}, t_1; \vphi)\\
        &\stackrel{(ii)}{=} \rvx_{t_1} - \rvx_{t_1}\\
        &= \bm{0}.
    \end{align*}
    Here (i) is true because the consistency model is parameterized such that $\vf(\rvx_{t_1}, t_1; \vphi) = \rvx_{t_1}$ and (ii) is entailed by the definition of $\vf(\cdot, \cdot; \vphi)$. This allows us to perform induction on the recursion formula \cref{eq:recursion} to obtain
    \begin{align*}
        \norm{\ve_{n}}_2 &\leq \norm{\ve_{1}}_2 + \sum_{k=1}^{n-1} O((t_{k+1} - t_k)^{p+1}) \\
        &= \sum_{k=1}^{n-1} O((t_{k+1} - t_k)^{p+1})\\
        &= \sum_{k=1}^{n-1} (t_{k+1} - t_k) O((t_{k+1} - t_k)^{p})\\
        &\leq \sum_{k=1}^{n-1} (t_{k+1} - t_k) O((\Delta t)^{p})\\
        &= O((\Delta t)^p) \sum_{k=1}^{n-1} (t_{k+1} - t_k)\\
        &= O((\Delta t)^p) (t_{n} - t_1)\\
        &\leq O((\Delta t)^p) (T-\epsilon)\\
        &= O((\Delta t)^p),
    \end{align*}
    which completes the proof.
\end{proof}

\subsection{Consistency Training}\label{app:proof_ct}

The following lemma provides an unbiased estimator for the score function, which is crucial to our proof for \cref{thm:ct}.

\begin{lemma}\label{lem:grad_log_p_t}
    Let $\rvx \sim p_\text{data}(\rvx)$, $\rvx_t \sim \mcal{N}(\rvx; t^2 \mI)$, and $p_t(\rvx_t) = p_\text{data}(\rvx) \otimes \mcal{N}(\bm{0}, t^2\mI)$. We have $\nabla \log p_t(\rvx) = -\mbb{E}[\frac{\rvx_t - \rvx}{t^2} \mid \rvx_t ]$.
\end{lemma}
\begin{proof}
According to the definition of $p_t(\rvx_t)$, we have $\nabla \log p_t(\rvx_t) = \nabla_{\rvx_t} \log \int p_\text{data}(\rvx) p(\rvx_t \mid \rvx) \ud \rvx$, where $p(\rvx_t \mid \rvx) = \mcal{N}(\rvx_t; \rvx, t^2 \mI)$. This expression can be further simplified to yield
\begin{align*}
    \nabla \log p_t(\rvx_t) &= \frac{\int p_\text{data}(\rvx) \nabla_{\rvx_t} p(\rvx_t \mid \rvx) \ud \rvx}{\int p_\text{data}(\rvx) p(\rvx_t \mid \rvx) \ud \rvx}\\
    &=\frac{\int p_\text{data}(\rvx) p(\rvx_t \mid \rvx) \nabla_{\rvx_t} \log p(\rvx_t \mid \rvx) \ud \rvx}{\int p_\text{data}(\rvx) p(\rvx_t \mid \rvx) \ud \rvx}\\
    &=\frac{\int p_\text{data}(\rvx) p(\rvx_t \mid \rvx) \nabla_{\rvx_t} \log p(\rvx_t \mid \rvx) \ud \rvx}{p_t(\rvx_t)}\\
    &=\int \frac{p_\text{data}(\rvx) p(\rvx_t \mid \rvx)}{p_t(\rvx_t)} \nabla_{\rvx_t} \log p(\rvx_t \mid \rvx) \ud \rvx\\
    &\stackrel{(i)}{=}\int p(\rvx \mid \rvx_t) \nabla_{\rvx_t} \log p(\rvx_t \mid \rvx) \ud \rvx\\
    &= \mbb{E}[\nabla_{\rvx_t} \log p(\rvx_t \mid \rvx) \mid \rvx_t]\\
    &= -\mbb{E}\left[\frac{\rvx_t - \rvx}{t^2} \mid \rvx_t\right],
\end{align*}
where (i) is due to Bayes' rule.
\end{proof}

\begin{customthm}{\ref{thm:ct}}
    Let $\Delta t \coloneqq \max_{n \in \llbracket 1, N-1\rrbracket}\{|t_{n+1} - t_{n}|\}$. Assume $d$ and $\vf_{\vtheta^{-}}$ are both twice continuously differentiable with bounded second derivatives, the weighting function $\lambda(\cdot)$ is bounded, and $\mbb{E}[\norm{\nabla \log p_{t_n}(\rvx_{t_{n}})}_2^2] < \infty$. Assume further that we use the Euler ODE solver, and the pre-trained score model matches the ground truth, \ie, $\forall t\in[\epsilon, T]: \vs_{\vphi}(\rvx, t) \equiv \nabla \log p_t(\rvx)$. Then,
    \begin{align*}
       \mcal{L}_\text{CD}^N(\vtheta, \vtheta^{-}; \vphi) = \mcal{L}_\text{CT}^N(\vtheta, \vtheta^{-}) + o(\Delta t),
    \end{align*}
    where the expectation is taken with respect to $\rvx \sim p_\text{data}$, $n \sim \mcal{U}\llbracket 1,N-1 \rrbracket$, and $\rvx_{t_{n+1}} \sim \mcal{N}(\rvx; t_{n+1}^2 \mI)$. The consistency training objective, denoted by $\mcal{L}_\text{CT}^N(\vtheta, \vtheta^{-})$, is defined as
    \begin{align*}
        \mbb{E}[\lambda(t_n) d(\vf_\vtheta(\rvx + t_{n+1}\rvz, t_{n+1}), \vf_{\vtheta^{-}}(\rvx + t_n\rvz, t_n))],
    \end{align*}
    where $\rvz \sim \mcal{N}(\bf{0}, \mI)$. Moreover, $\mcal{L}_\text{CT}^N(\vtheta, \vtheta^{-}) \geq O(\Delta t)$ if $\inf_N \mcal{L}_\text{CD}^N(\vtheta, \vtheta^{-}; \vphi) > 0$.
\end{customthm}
\begin{proof}
With Taylor expansion, we have
\begin{align}
    &\mcal{L}_\text{CD}^N(\vtheta, \vtheta^{-}; \vphi) = \mbb{E}[\lambda(t_n)d(\vf_\vtheta(\rvx_{t_{n+1}}, t_{n+1}), \vf_{\vtheta^{-}}(\hat{\rvx}_{t_n}^\vphi, t_n)]\notag \\
    =& \mbb{E}[\lambda(t_n) d(\vf_\vtheta(\rvx_{t_{n+1}}, t_{n+1}), \vf_{\vtheta^{-}}(\rvx_{t_{n+1}} + (t_{n+1} - t_n)t_{n+1} \nabla\log p_{t_{n+1}}(\rvx_{t_{n+1}}), t_n))]\notag \\
    =& \mbb{E}[\lambda(t_n) d(\vf_\vtheta(\rvx_{t_{n+1}}, t_{n+1}), \vf_{\vtheta^{-}}(\rvx_{t_{n+1}}, t_{n+1}) + \partial_1\vf_{\vtheta^{-}}(\rvx_{t_{n+1}}, t_{n+1})(t_{n+1} - t_n)t_{n+1} \nabla \log p_{t_{n+1}}(\rvx_{t_{n+1}})\notag \\
    &\qquad + \partial_2 \vf_{\vtheta^{-}}(\rvx_{t_{n+1}}, t_{n+1}) (t_n-t_{n+1}) + o(|t_{n+1} - t_n|) )]\notag \\
    =& \mbb{E}\{\lambda(t_n) d(\vf_\vtheta(\rvx_{t_{n+1}}, t_{n+1}),\vf_{\vtheta^{-}}(\rvx_{t_{n+1}}, t_{n+1})) + \lambda(t_n)\partial_2 d(\vf_\vtheta(\rvx_{t_{n+1}}, t_{n+1}),\vf_{\vtheta^{-}}(\rvx_{t_{n+1}}, t_{n+1}))[\notag \\
    &\quad \partial_1\vf_{\vtheta^{-}}(\rvx_{t_{n+1}}, t_{n+1})(t_{n+1} - t_n)t_{n+1} \nabla \log p_{t_{n+1}}(\rvx_{t_{n+1}}) + \partial_2 \vf_{\vtheta^{-}}(\rvx_{t_{n+1}}, t_{n+1}) (t_n-t_{n+1}) + o(|t_{n+1} - t_n|)]\}\notag \\
    =& \mbb{E}[\lambda(t_n) d(\vf_\vtheta(\rvx_{t_{n+1}}, t_{n+1}),\vf_{\vtheta^{-}}(\rvx_{t_{n+1}}, t_{n+1}))]\notag \\
    &\quad + \mbb{E}\{\lambda(t_n) \partial_2 d(\vf_\vtheta(\rvx_{t_{n+1}}, t_{n+1}),\vf_{\vtheta^{-}}(\rvx_{t_{n+1}}, t_{n+1}))[\partial_1\vf_{\vtheta^{-}}(\rvx_{t_{n+1}}, t_{n+1})(t_{n+1} - t_n)t_{n+1} \nabla \log p_{t_{n+1}}(\rvx_{t_{n+1}})]\}\notag \\
    &\qquad + \mbb{E}\{\lambda(t_n) \partial_2 d(\vf_\vtheta(\rvx_{t_{n+1}}, t_{n+1}),\vf_{\vtheta^{-}}(\rvx_{t_{n+1}}, t_{n+1}))[\partial_2 \vf_{\vtheta^{-}}(\rvx_{t_{n+1}}, t_{n+1}) (t_n-t_{n+1})]\} + \mbb{E}[o(|t_{n+1} - t_n|)]
    .\label{eq:taylor1}
\end{align}
Then, we apply \cref{lem:grad_log_p_t} to \cref{eq:taylor1} and use Taylor expansion in the reverse direction to obtain
\begin{align}
    &\mcal{L}_\text{CD}^N(\vtheta, \vtheta^{-}; \vphi)\notag \\
    =& \mbb{E}[\lambda(t_n) d(\vf_\vtheta(\rvx_{t_{n+1}}, t_{n+1}),\vf_{\vtheta^{-}}(\rvx_{t_{n+1}}, t_{n+1}))]\notag \\
    &\quad + \mbb{E}\left\{\lambda(t_n)\partial_2 d(\vf_\vtheta(\rvx_{t_{n+1}}, t_{n+1}),\vf_{\vtheta^{-}}(\rvx_{t_{n+1}}, t_{n+1}))\left[\partial_1\vf_{\vtheta^{-}}(\rvx_{t_{n+1}}, t_{n+1})(t_{n} - t_{n+1})t_{n+1} \mbb{E}\left[\frac{\rvx_{t_{n+1}} - \rvx}{t_{n+1}^2}\Big| \rvx_{t_{n+1}} \right]\right]\right \}\notag \\
    &\qquad + \mbb{E}\{\lambda(t_n)\partial_2 d(\vf_\vtheta(\rvx_{t_{n+1}}, t_{n+1}),\vf_{\vtheta^{-}}(\rvx_{t_{n+1}}, t_{n+1}))[\partial_2 \vf_{\vtheta^{-}}(\rvx_{t_{n+1}}, t_{n+1}) (t_n-t_{n+1})]\} + \mbb{E}[o(|t_{n+1} - t_n|)]\notag\\
    \stackrel{(i)}{=}& \mbb{E}[\lambda(t_n)d(\vf_\vtheta(\rvx_{t_{n+1}}, t_{n+1}),\vf_{\vtheta^{-}}(\rvx_{t_{n+1}}, t_{n+1}))]\notag \\
    &\quad + \mbb{E}\left\{\lambda(t_n)\partial_2 d(\vf_\vtheta(\rvx_{t_{n+1}}, t_{n+1}),\vf_{\vtheta^{-}}(\rvx_{t_{n+1}}, t_{n+1}))\left[\partial_1\vf_{\vtheta^{-}}(\rvx_{t_{n+1}}, t_{n+1})(t_{n} - t_{n+1})t_{n+1} \left(\frac{\rvx_{t_{n+1}} - \rvx}{t_{n+1}^2} \right)\right]\right \}\notag \\
    &\qquad + \mbb{E}\{\lambda(t_n)\partial_2 d(\vf_\vtheta(\rvx_{t_{n+1}}, t_{n+1}),\vf_{\vtheta^{-}}(\rvx_{t_{n+1}}, t_{n+1}))[\partial_2 \vf_{\vtheta^{-}}(\rvx_{t_{n+1}}, t_{n+1}) (t_n-t_{n+1})]\} + \mbb{E}[o(|t_{n+1} - t_n|)]\notag\\
    =& \mbb{E}\bigg[\lambda(t_n)d(\vf_\vtheta(\rvx_{t_{n+1}}, t_{n+1}),\vf_{\vtheta^{-}}(\rvx_{t_{n+1}}, t_{n+1}))\notag \\
    &\quad + \lambda(t_n)\partial_2 d(\vf_\vtheta(\rvx_{t_{n+1}}, t_{n+1}),\vf_{\vtheta^{-}}(\rvx_{t_{n+1}}, t_{n+1}))\left[\partial_1\vf_{\vtheta^{-}}(\rvx_{t_{n+1}}, t_{n+1})(t_{n} - t_{n+1})t_{n+1} \left(\frac{\rvx_{t_{n+1}} - \rvx}{t_{n+1}^2} \right)\right] \notag \\
    &\quad +\lambda(t_n)\partial_2 d(\vf_\vtheta(\rvx_{t_{n+1}}, t_{n+1}),\vf_{\vtheta^{-}}(\rvx_{t_{n+1}}, t_{n+1}))[\partial_2 \vf_{\vtheta^{-}}(\rvx_{t_{n+1}}, t_{n+1}) (t_n-t_{n+1})] + o(|t_{n+1} - t_n|)\bigg] \notag\\
    &\qquad + \mbb{E}[o(|t_{n+1} - t_n|)]\notag \\
    =& \mbb{E}\left[\lambda(t_n) d\left(\vf_\vtheta(\rvx_{t_{n+1}}, t_{n+1}), \vf_{\vtheta^{-}}\left(\rvx_{t_{n+1}} + (t_{n} - t_{n+1})t_{n+1}\frac{\rvx_{t_{n+1}} - \rvx}{t_{n+1}^2} , t_n\right)\right)\right] + \mbb{E}[o(|t_{n+1} - t_n|)]\notag\\
    =& \mbb{E}\left[\lambda(t_n) d\left(\vf_\vtheta(\rvx_{t_{n+1}}, t_{n+1}), \vf_{\vtheta^{-}}\left(\rvx_{t_{n+1}} + (t_{n} - t_{n+1})\frac{\rvx_{t_{n+1}} - \rvx}{t_{n+1}} , t_n\right)\right)\right] + \mbb{E}[o(|t_{n+1} - t_n|)]\notag\\
    =& \mbb{E}\left[\lambda(t_n)d\left(\vf_\vtheta(\rvx + t_{n+1} \rvz, t_{n+1}), \vf_{\vtheta^{-}}\left(\rvx + t_{n+1}\rvz + (t_{n} - t_{n+1})\rvz , t_n\right)\right)\right] + \mbb{E}[o(|t_{n+1} - t_n|)]\notag\\
    =& \mbb{E}\left[\lambda(t_n)d\left(\vf_\vtheta(\rvx + t_{n+1} \rvz, t_{n+1}), \vf_{\vtheta^{-}}\left(\rvx + t_{n}\rvz , t_n\right)\right)\right] + \mbb{E}[o(|t_{n+1} - t_n|)]\notag\\
    =& \mbb{E}\left[\lambda(t_n)d\left(\vf_\vtheta(\rvx + t_{n+1} \rvz, t_{n+1}), \vf_{\vtheta^{-}}\left(\rvx + t_{n}\rvz , t_n\right)\right)\right] + \mbb{E}[o(\Delta t)]\notag\\
    =& \mbb{E}\left[\lambda(t_n)d\left(\vf_\vtheta(\rvx + t_{n+1} \rvz, t_{n+1}), \vf_{\vtheta^{-}}\left(\rvx + t_{n}\rvz , t_n\right)\right)\right] + o(\Delta t)\notag \\
    =& \mcal{L}_\text{CT}^N(\vtheta, \vtheta^{-}) + o(\Delta t),
\end{align}
where (i) is due to the law of total expectation, and $\rvz \coloneqq \frac{\rvx_{t_{n+1}} - \rvx}{t_{n+1}} \sim \mcal{N}(\bm{0}, \mI)$. This implies $\mcal{L}_\text{CD}^N(\vtheta, \vtheta^{-}; \vphi) = \mcal{L}_\text{CT}^N(\vtheta, \vtheta^{-}) + o(\Delta t)$ and thus completes the proof for \cref{eq:cd2}. Moreover, we have $\mcal{L}_\text{CT}^N(\vtheta, \vtheta^{-}) \geq O(\Delta t)$ whenever $\inf_N \mcal{L}_\text{CD}^N(\vtheta, \vtheta^{-}; \vphi) > 0$. Otherwise, $\mcal{L}_\text{CT}^N(\vtheta, \vtheta^{-}) < O(\Delta t)$ and thus $\lim_{\Delta t \to 0} \mcal{L}_\text{CD}^N(\vtheta, \vtheta^{-}; \vphi) = 0$, which is a clear contradiction to $\inf_N \mcal{L}_\text{CD}^N(\vtheta, \vtheta^{-}; \vphi) > 0$.
\end{proof}
\begin{remark}
    When the condition $\mathcal{L}_\text{CT}^N(\vtheta, \vtheta^{-}) \geq O(\Delta t)$ is not satisfied, such as in the case where $\vtheta^{-} = \operatorname{stopgrad}(\vtheta)$, the validity of $\mathcal{L}_\text{CT}^N(\vtheta, \vtheta^{-})$ as a training objective for consistency models can still be justified by referencing the result provided in \cref{thm:ctct}.
\end{remark}

\section{Continuous-Time Extensions}\label{app:continuous}
The consistency distillation and consistency training objectives can be generalized to hold for infinite time steps ($N\to\infty$) under suitable conditions.

\subsection{Consistency Distillation in Continuous Time} \label{sec:ctcd}
Depending on whether $\vtheta^- = \vtheta$ or $\vtheta^- = \operatorname{stopgrad}(\vtheta)$ (same as setting $\mu=0$), there are two possible continuous-time extensions for the consistency distillation objective $\mcal{L}_\text{CD}^N(\vtheta, \vtheta^{-}; \vphi)$. Given a twice continuously differentiable metric function $d(\rvx, \rvy)$, we define $\mG(\rvx)$ as a matrix, whose $(i,j)$-th entry is given by
\begin{align*}
    [\mG(\rvx)]_{ij} \coloneqq \frac{\partial^2 d(\rvx, \rvy)}{\partial y_i \partial y_j} \bigg|_{\rvy=\rvx}.
\end{align*}
Similarly, we define $\mH(\rvx)$ as
\begin{align*}
    [\mH(\rvx)]_{ij} \coloneqq \frac{\partial^2 d(\rvy, \rvx)}{\partial y_i \partial y_j} \bigg|_{\rvy=\rvx}.
\end{align*}
The matrices $\mG$ and $\mH$ play a crucial role in forming continuous-time objectives for consistency distillation. Additionally, we denote the Jacobian of $\vf_\vtheta(\rvx, t)$ with respect to $\rvx$ as $\frac{\partial \vf_\vtheta(\rvx, t)}{\partial \rvx}$.

When $\vtheta^{-}=\vtheta$ (with no stopgrad operator), we have the following theoretical result.
\begin{theorem}\label{thm:ctcd1}
    Let $t_n = \tau(\frac{n-1}{N-1})$, where $n \in \llbracket 1, N \rrbracket$, and $\tau(\cdot)$ is a strictly monotonic function with $\tau(0) = \epsilon$ and $\tau(1) = T$. Assume $\tau$ is continuously differentiable in $[0,1]$, $d$ is three times continuously differentiable with bounded third derivatives, and $\vf_{\vtheta}$ is twice continuously differentiable with bounded first and second derivatives. Assume further that the weighting function $\lambda(\cdot)$ is bounded, and $\sup_{\rvx,t\in[\epsilon, T]}\norm{\vs_\vphi(\rvx, t)}_2 < \infty$. Then with the Euler solver in consistency distillation, we have
    \begin{align}
        \lim_{N \to \infty} (N-1)^2 \mcal{L}_\text{CD}^N(\vtheta, \vtheta; \vphi) = \mcal{L}_\text{CD}^\infty(\vtheta, \vtheta; \vphi) \label{eq:ctcd_obj},
    \end{align}
    where $\mcal{L}_\text{CD}^{\infty} (\vtheta, \vtheta; \vphi)$ is defined as
    \begin{align}
        \frac{1}{2} \mbb{E}\left[\frac{\lambda(t)}{[(\tau^{-1})'(t)]^2} \left(\frac{\partial \vf_\vtheta(\rvx_t, t)}{\partial t} - t \frac{\partial \vf_\vtheta(\rvx_t, t)}{\partial \rvx_t} \vs_\vphi(\rvx_{t}, t)\right)\tran \mG(\vf_\vtheta(\rvx_t, t)) \left(\frac{\partial \vf_\vtheta(\rvx_t, t)}{\partial t} - t \frac{\partial \vf_\vtheta(\rvx_t, t)}{\partial \rvx_t} \vs_\vphi(\rvx_{t}, t)\right)\right].
    \end{align}
    Here the expectation above is taken over $\rvx \sim p_\text{data}$, $u \sim \mcal{U}[0, 1]$, $t = \tau(u)$, and $\rvx_t \sim \mcal{N}(\rvx, t^2\mI)$.
\end{theorem}
\begin{proof}
    Let $\Delta u = \frac{1}{N-1}$ and $u_n = \frac{n-1}{N-1}$. First, we can derive the following equation with Taylor expansion:
    \begin{align}
        &\vf_\vtheta(\hat{\rvx}_{t_n}^\vphi, t_n) - \vf_\vtheta(\rvx_{t_{n+1}}, t_{n+1}) = \vf_{\vtheta}(\rvx_{t_{n+1}} +  t_{n+1} \vs_\vphi(\rvx_{t_{n+1}}, t_{n+1})\tau'(u_{n})\Delta u, t_n) - \vf_\vtheta(\rvx_{t_{n+1}}, t_{n+1})\notag \\
        =& t_{n+1} \frac{\partial \vf_\vtheta(\rvx_{t_{n+1}}, t_{n+1})}{\partial \rvx_{t_{n+1}}}\vs_\vphi(\rvx_{t_{n+1}}, t_{n+1})\tau'(u_{n})\Delta u  - \frac{\partial \vf_\vtheta(\rvx_{t_{n+1}}, t_{n+1})}{\partial t_{n+1}}\tau'(u_{n})\Delta u  + O((\Delta u)^2),\label{eq:ctcd1}
    \end{align}
    Note that $\tau'(u_{n}) = \frac{1}{\tau^{-1}(t_{n+1})}$. Then, we apply Taylor expansion to the consistency distillation loss, which gives
    \begin{align}
        &(N-1)^2 \mcal{L}_\text{CD}^N(\vtheta, \vtheta; \vphi) = \frac{1}{(\Delta u)^2}\mcal{L}_\text{CD}^N(\vtheta, \vtheta; \vphi) = \frac{1}{(\Delta u)^2} \mbb{E}[\lambda(t_n)d(\vf_\vtheta(\rvx_{t_{n+1}}, t_{n+1}), \vf_{\vtheta}(\hat{\rvx}_{t_n}^\vphi, t_n)]\notag \\
        \stackrel{(i)}{=}&\begin{multlined}[t][0.9\displaywidth]
            \frac{1}{2 (\Delta u)^2}\bigg(\mbb{E}\{\lambda(t_n)\tau'(u_{n})^2 [\vf_\vtheta(\hat{\rvx}_{t_n}^\vphi, t_n) - \vf_\vtheta(\rvx_{t_{n+1}}, t_{n+1})]\tran \mG(\vf_\vtheta(\rvx_{t_{n+1}}, t_{n+1}))\\ \cdot [\vf_\vtheta(\hat{\rvx}_{t_n}^\vphi, t_n) - \vf_\vtheta(\rvx_{t_{n+1}}, t_{n+1})]\} + \mbb{E}[O(|\Delta u|^3)]\bigg)
        \end{multlined}\notag \\
        \stackrel{(ii)}{=}&\!\begin{multlined}[t][0.9\displaywidth]
            \frac{1}{2}\mbb{E}\bigg[\lambda(t_n) \tau'(u_{n})^2 \left(\frac{\partial \vf_\vtheta(\rvx_{t_{n+1}}, t_{n+1})}{\partial t_{n+1}} - t_{n+1} \frac{\partial \vf_\vtheta(\rvx_{t_{n+1}}, t_{n+1})}{\partial \rvx_{t_{n+1}}} \vs_\vphi(\rvx_{t_{n+1}}, t_{n+1})\right)\tran \mG(\vf_\vtheta(\rvx_{t_{n+1}}, t_{n+1})) \\
            \cdot \bigg(\frac{\partial \vf_\vtheta(\rvx_{t_{n+1}}, t_{n+1})}{\partial t_{n+1}} - t_{n+1} \frac{\partial \vf_\vtheta(\rvx_{t_{n+1}}, t_{n+1})}{\partial \rvx_{t_{n+1}}} \vs_\vphi(\rvx_{t_{n+1}}, t_{n+1})\bigg)\bigg] + \mbb{E}[O(|\Delta u|)]
        \end{multlined}\notag \\
        =&\!\begin{multlined}[t][0.9\displaywidth]
                \frac{1}{2}\mbb{E}\bigg[\frac{\lambda(t_n)}{[(\tau^{-1})'(t_{n})]^2} \left(\frac{\partial \vf_\vtheta(\rvx_{t_{n+1}}, t_{n+1})}{\partial t_{n+1}} - t_{n+1} \frac{\partial \vf_\vtheta(\rvx_{t_{n+1}}, t_{n+1})}{\partial \rvx_{t_{n+1}}} \vs_\vphi(\rvx_{t_{n+1}}, t_{n+1})\right)\tran \mG(\vf_\vtheta(\rvx_{t_{n+1}}, t_{n+1})) \\
                \cdot \bigg(\frac{\partial \vf_\vtheta(\rvx_{t_{n+1}}, t_{n+1})}{\partial t_{n+1}} - t_{n+1} \frac{\partial \vf_\vtheta(\rvx_{t_{n+1}}, t_{n+1})}{\partial \rvx_{t_{n+1}}} \vs_\vphi(\rvx_{t_{n+1}}, t_{n+1})\bigg)\bigg] + \mbb{E}[O(|\Delta u|)] \label{eq:ctcd2}
        \end{multlined}
    \end{align}
    where we obtain (i) by expanding $d(\vf_\vtheta(\rvx_{t_{n+1}}, t_{n+1}), \cdot)$ to second order and observing $d(\rvx, \rvx) \equiv 0$ and $\nabla_\rvy d(\rvx, \rvy)|_{\rvy=\rvx} \equiv \bm{0}$. We obtain (ii) using \cref{eq:ctcd1}. By taking the limit for both sides of \cref{eq:ctcd2} as $\Delta u \to 0$ or equivalently $N \to \infty$, we arrive at \cref{eq:ctcd_obj}, which completes the proof.
\end{proof}
\begin{remark}
    Although \cref{thm:ctcd1} assumes the Euler ODE solver for technical simplicity, we believe an analogous result can be derived for more general solvers, since all ODE solvers should perform similarly as $N \to \infty$. We leave a more general version of \cref{thm:ctcd1} as future work.
\end{remark}
\begin{remark}
    \cref{thm:ctcd1} implies that consistency models can be trained by minimizing $\mcal{L}_\text{CD}^\infty (\vtheta, \vtheta;\vphi)$. In particular, when $d(\rvx, \rvy) = \norm{\rvx - \rvy}_2^2$, we have
    \begin{align}
        \mcal{L}_\text{CD}^{\infty} (\vtheta, \vtheta; \vphi) = \mbb{E}\left[\frac{\lambda(t)}{[(\tau^{-1})'(t)]^2}\norm{\frac{\partial \vf_\vtheta(\rvx_t, t)}{\partial t} - t \frac{\partial \vf_\vtheta(\rvx_t, t)}{\partial \rvx_t} \vs_\vphi(\rvx_{t}, t)}^2_2 \right].
    \end{align}
    However, this continuous-time objective requires computing Jacobian-vector products as a subroutine to evaluate the loss function, which can be slow and laborious to implement in deep learning frameworks that do not support forward-mode automatic differentiation.
\end{remark}
\begin{remark}\label{remark}
    If $\vf_\vtheta(\rvx, t)$ matches the ground truth consistency function for the empirical PF ODE of $\vs_\vphi(\rvx, t)$, then
    \begin{align*}
        \frac{\partial \vf_\vtheta(\rvx, t)}{\partial t} - t \frac{\partial \vf_\vtheta(\rvx, t)}{\partial \rvx} \vs_\vphi(\rvx, t) \equiv 0
    \end{align*}
    and therefore $\mcal{L}_\text{CD}^\infty(\vtheta, \vtheta; \vphi) = 0$. This can be proved by noting that $\vf_\vtheta(\rvx_t, t) \equiv \rvx_\epsilon$ for all $t \in [\epsilon, T]$, and then taking the time-derivative of this identity:
    \begin{align*}
        &\vf_\vtheta(\rvx_t, t) \equiv \rvx_\epsilon\\
        \Longleftrightarrow&\frac{\partial \vf_\vtheta(\rvx_t, t)}{\partial \rvx_t} \frac{\ud \rvx_t}{\ud t} + \frac{\partial \vf_\vtheta(\rvx_t, t)}{\partial t} \equiv 0\\
        \Longleftrightarrow&\frac{\partial \vf_\vtheta(\rvx_t, t)}{\partial \rvx_t} [-t \vs_\vphi(\rvx_t, t)] + \frac{\partial \vf_\vtheta(\rvx_t, t)}{\partial t} \equiv 0\\
        \Longleftrightarrow&\frac{\partial \vf_\vtheta(\rvx_t, t)}{\partial t} - t \frac{\partial \vf_\vtheta(\rvx_t, t)}{\partial \rvx_t} \vs_\vphi(\rvx_t, t) \equiv 0.
    \end{align*}
    The above observation provides another motivation for $\mcal{L}_\text{CD}^\infty(\vtheta, \vtheta; \vphi)$, as it is minimized if and only if the consistency model matches the ground truth consistency function.
\end{remark}

For some metric functions, such as the $\ell_1$ norm, the Hessian $\mG(\rvx)$ is zero so \cref{thm:ctcd1} is vacuous. Below we show that a non-vacuous statement holds for the $\ell_1$ norm with just a small modification of the proof for \cref{thm:ctcd1}.
\begin{theorem}\label{thm:ctcd_l1}
    Let $t_n = \tau(\frac{n-1}{N-1})$, where $n \in \llbracket 1, N \rrbracket$, and $\tau(\cdot)$ is a strictly monotonic function with $\tau(0) = \epsilon$ and $\tau(1) = T$. Assume $\tau$ is continuously differentiable in $[0,1]$, and $\vf_{\vtheta}$ is twice continuously differentiable with bounded first and second derivatives. Assume further that the weighting function $\lambda(\cdot)$ is bounded, and $\sup_{\rvx, t\in[\epsilon,T]}\norm{\vs_\vphi(\rvx, t)}_2 < \infty$. Suppose we use the Euler ODE solver, and set $d(\rvx, \rvy) = \norm{\rvx - \rvy}_1$ in consistency distillation. Then we have
    \begin{align}
        \lim_{N \to \infty} (N-1) \mcal{L}_\text{CD}^N(\vtheta, \vtheta; \vphi) = \mcal{L}_\text{CD, $\ell_1$}^\infty(\vtheta, \vtheta; \vphi),\label{eq:ctcd_l1_obj}
    \end{align}
    where
    \begin{align*}
        \mcal{L}_\text{CD, $\ell_1$}^\infty(\vtheta, \vtheta; \vphi) \coloneqq \mbb{E}\left[\frac{\lambda(t)}{(\tau^{-1})'(t)}\norm{t \frac{\partial \vf_\vtheta(\rvx_{t}, t)}{\partial \rvx_{t}}\vs_\vphi(\rvx_{t}, t)  - \frac{\partial \vf_\vtheta(\rvx_{t}, t)}{\partial t}}_1\right]
    \end{align*}
    where the expectation above is taken over $\rvx \sim p_\text{data}$, $u \sim \mcal{U}[0, 1]$, $t = \tau(u)$, and $\rvx_t \sim \mcal{N}(\rvx, t^2\mI)$.
\end{theorem}
\begin{proof}
    Let $\Delta u = \frac{1}{N-1}$ and $u_n = \frac{n-1}{N-1}$. We have
    \begin{align}
        &(N-1) \mcal{L}_\text{CD}^N(\vtheta, \vtheta; \vphi) = \frac{1}{\Delta u}\mcal{L}_\text{CD}^N(\vtheta, \vtheta; \vphi) = \frac{1}{\Delta u} \mbb{E}[\lambda(t_n)\| \vf_\vtheta(\rvx_{t_{n+1}}, t_{n+1}) - \vf_{\vtheta}(\hat{\rvx}_{t_n}^\vphi, t_n)\|_1]\notag \\
        \stackrel{(i)}{=}& \frac{1}{\Delta u} \mbb{E}\left[\lambda(t_n)\norm{t_{n+1} \frac{\partial \vf_\vtheta(\rvx_{t_{n+1}}, t_{n+1})}{\partial \rvx_{t_{n+1}}}\vs_\vphi(\rvx_{t_{n+1}}, t_{n+1})\tau'(u_{n}) - \frac{\partial \vf_\vtheta(\rvx_{t_{n+1}}, t_{n+1})}{\partial t_{n+1}}\tau'(u_{n}) + O((\Delta u)^2)}_1\right]\notag\\
        =& \mbb{E}\left[\lambda(t_n)\tau'(u_{n})\norm{t_{n+1} \frac{\partial \vf_\vtheta(\rvx_{t_{n+1}}, t_{n+1})}{\partial \rvx_{t_{n+1}}}\vs_\vphi(\rvx_{t_{n+1}}, t_{n+1})  - \frac{\partial \vf_\vtheta(\rvx_{t_{n+1}}, t_{n+1})}{\partial t_{n+1}}  + O(\Delta u)}_1\right]\notag\\
        =& \mbb{E}\left[\frac{\lambda(t_n)}{(\tau^{-1})'(t_{n})}\norm{t_{n+1} \frac{\partial \vf_\vtheta(\rvx_{t_{n+1}}, t_{n+1})}{\partial \rvx_{t_{n+1}}}\vs_\vphi(\rvx_{t_{n+1}}, t_{n+1})  - \frac{\partial \vf_\vtheta(\rvx_{t_{n+1}}, t_{n+1})}{\partial t_{n+1}}  + O(\Delta u)}_1\right]\label{eq:ctcd_l1}
    \end{align}
    where (i) is obtained by plugging \cref{eq:ctcd1} into the previous equation. Taking the limit for both sides of \cref{eq:ctcd_l1} as $\Delta u \to 0$ or equivalently $N\to \infty$ leads to \cref{eq:ctcd_l1_obj}, which completes the proof.
\end{proof}
\begin{remark}
    According to \cref{thm:ctcd_l1}, consistency models can be trained by minimizing $\mcal{L}_\text{CD, $\ell_1$}^\infty(\vtheta, \vtheta; \vphi)$. Moreover, the same reasoning in \cref{remark} can be applied to show that $\mcal{L}_\text{CD, $\ell_1$}^\infty(\vtheta, \vtheta; \vphi) = 0$ if and only if $\vf_\vtheta(\rvx_t, t) = \rvx_\epsilon$ for all $\rvx_t \in \mathbb{R}^d$ and $t \in [\epsilon, T]$.
\end{remark}

In the second case where $\vtheta^- = \operatorname{stopgrad}(\vtheta)$, we can derive a so-called ``pseudo-objective'' whose gradient matches the gradient of $\mcal{L}_\text{CD}^N(\vtheta, \vtheta^{-}; \vphi)$ in the limit of $N\to\infty$. Minimizing this pseudo-objective with gradient descent gives another way to train consistency models via distillation. This pseudo-objective is provided by the theorem below.

\begin{theorem}\label{thm:ctcd2}
    Let $t_n = \tau(\frac{n-1}{N-1})$, where $n \in \llbracket 1, N \rrbracket$, and $\tau(\cdot)$ is a strictly monotonic function with $\tau(0) = \epsilon$ and $\tau(1) = T$. Assume $\tau$ is continuously differentiable in $[0,1]$, $d$ is three times continuously differentiable with bounded third derivatives, and $\vf_{\vtheta}$ is twice continuously differentiable with bounded first and second derivatives. Assume further that the weighting function $\lambda(\cdot)$ is bounded, $\sup_{\rvx, t\in[\epsilon,T]}\norm{\vs_\vphi(\rvx, t)}_2 < \infty$, and $\sup_{\rvx, t\in[\epsilon, T]}\norm{\nabla_\vtheta \vf_\vtheta(\rvx, t)}_2 < \infty$. Suppose we use the Euler ODE solver, and $\vtheta^{-} = \operatorname{stopgrad}(\vtheta)$ in consistency distillation. Then,
    \begin{align}
        \lim_{N \to \infty} (N-1) \nabla_\vtheta \mcal{L}_\text{CD}^N(\vtheta, \vtheta^{-}; \vphi) = \nabla_\vtheta \mcal{L}_\text{CD}^\infty(\vtheta, \vtheta^{-}; \vphi),\label{eq:ctcd_obj2}
    \end{align}
    where
    \begin{align}
        \mcal{L}_\text{CD}^{\infty} (\vtheta, \vtheta^{-}; \vphi) \coloneqq \mbb{E}\left[\frac{\lambda(t)}{(\tau^{-1})'(t)} \vf_\vtheta(\rvx_t, t) \tran \mH(\vf_{\vtheta^-}(\rvx_t, t)) \left(\frac{\partial \vf_{\vtheta^-}(\rvx_t, t)}{\partial t} - t \frac{\partial \vf_{\vtheta^-}(\rvx_t, t)}{\partial \rvx_t} \vs_\vphi(\rvx_{t}, t)\right)\right].
    \end{align}
    Here the expectation above is taken over $\rvx \sim p_\text{data}$, $u \sim \mcal{U}[0, 1]$, $t = \tau(u)$, and $\rvx_t \sim \mcal{N}(\rvx, t^2\mI)$.
\end{theorem}
\begin{proof}
    We denote $\Delta u = \frac{1}{N-1}$ and $u_n = \frac{n-1}{N-1}$. First, we leverage Taylor series expansion to obtain
    \begin{align}
        &(N-1)\mcal{L}_\text{CD}^N(\vtheta, \vtheta^{-}; \vphi) = \frac{1}{\Delta u}\mcal{L}_\text{CD}^N(\vtheta, \vtheta^{-}; \vphi) = \frac{1}{\Delta u} \mbb{E}[\lambda(t_n)d(\vf_\vtheta(\rvx_{t_{n+1}}, t_{n+1}), \vf_{\vtheta^{-}}(\hat{\rvx}_{t_n}^\vphi, t_n)]\notag \\
        \stackrel{(i)}{=}&\begin{multlined}[t][0.9\displaywidth]
            \frac{1}{2 \Delta u}\bigg(\mbb{E}\{\lambda(t_n)[\vf_\vtheta(\rvx_{t_{n+1}}, t_{n+1}) - \vf_{\vtheta^-}(\hat{\rvx}_{t_n}^\vphi, t_n)]\tran \mH(\vf_{\vtheta^{-}}(\hat{\rvx}_{t_n}^\vphi, t_n))\\\cdot [\vf_\vtheta(\rvx_{t_{n+1}}, t_{n+1}) - \vf_{\vtheta^{-}}(\hat{\rvx}_{t_n}^\vphi, t_n)]\} + \mbb{E}[O(|\Delta u|^3)]\bigg)
        \end{multlined}\notag \\
        =&\frac{1}{2\Delta u}\mbb{E}\{\lambda(t_n)[\vf_\vtheta(\rvx_{t_{n+1}}, t_{n+1}) - \vf_{\vtheta^-}(\hat{\rvx}_{t_n}^\vphi, t_n)]\tran \mH(\vf_{\vtheta^{-}}(\hat{\rvx}_{t_n}^\vphi, t_n))[\vf_\vtheta(\rvx_{t_{n+1}}, t_{n+1}) - \vf_{\vtheta^-}(\hat{\rvx}_{t_n}^\vphi, t_n)]\} + \mbb{E}[O(|\Delta u|^2)]\label{eq:ctcd2_1}
    \end{align}
    where (i) is derived by expanding $d(\cdot, \vf_{\vtheta^{-}}(\hat{\rvx}_{t_n}^\vphi, t_n))$ to second order and leveraging $d(\rvx, \rvx) \equiv 0$ and $\nabla_\rvy d(\rvy, \rvx)|_{\rvy=\rvx} \equiv \bm{0}$. Next, we compute the gradient of \cref{eq:ctcd2_1} with respect to $\vtheta$ and simplify the result to obtain
    \begin{align}
        &(N-1) \nabla_\vtheta \mcal{L}_\text{CD}^N(\vtheta, \vtheta^{-}; \vphi) = \frac{1}{\Delta u} \nabla_\vtheta \mcal{L}_\text{CD}^N(\vtheta, \vtheta^{-}; \vphi)\notag \\
        =&\frac{1}{2\Delta u} \nabla_\vtheta \mbb{E}\{\lambda(t_n)[\vf_\vtheta(\rvx_{t_{n+1}}, t_{n+1}) - \vf_{\vtheta^-}(\hat{\rvx}_{t_n}^\vphi, t_n)]\tran \mH(\vf_{\vtheta^{-}}(\hat{\rvx}_{t_n}^\vphi, t_n))[\vf_\vtheta(\rvx_{t_{n+1}}, t_{n+1}) - \vf_{\vtheta^-}(\hat{\rvx}_{t_n}^\vphi, t_n)]\} + \mbb{E}[O(|\Delta u|^2)]\notag \\
        \stackrel{(i)}{=}&\frac{1}{\Delta u} \mbb{E}\{\lambda(t_n)[\nabla_\vtheta \vf_\vtheta(\rvx_{t_{n+1}}, t_{n+1})]\tran \mH(\vf_{\vtheta^{-}}(\hat{\rvx}_{t_n}^\vphi, t_n))[\vf_\vtheta(\rvx_{t_{n+1}}, t_{n+1}) - \vf_{\vtheta^-}(\hat{\rvx}_{t_n}^\vphi, t_n)]\} + \mbb{E}[O(|\Delta u|^2)]\notag \\
        \stackrel{(ii)}{=}&\!\begin{multlined}[t][0.9\displaywidth]
                \frac{1}{\Delta u}\mbb{E}\bigg\{\lambda(t_n)[\nabla_\vtheta \vf_\vtheta(\rvx_{t_{n+1}}, t_{n+1})]\tran \mH(\vf_{\vtheta^{-}}(\hat{\rvx}_{t_n}^\vphi, t_n))\bigg[t_{n+1} \frac{\partial \vf_{\vtheta^-}(\rvx_{t_{n+1}}, t_{n+1})}{\partial \rvx_{t_{n+1}}}\vs_\vphi(\rvx_{t_{n+1}}, t_{n+1})\tau'(u_{n})\Delta u \\ - \frac{\partial \vf_{\vtheta^-}(\rvx_{t_{n+1}}, t_{n+1})}{\partial t_{n+1}}\tau'(u_{n})\Delta u\bigg]\bigg\} + \mbb{E}[O(|\Delta u|)]
            \end{multlined}\notag \\
        =&\!\begin{multlined}[t][0.9\displaywidth]
            \mbb{E}\bigg\{\lambda(t_n)[\nabla_\vtheta \vf_\vtheta(\rvx_{t_{n+1}}, t_{n+1})]\tran \mH(\vf_{\vtheta^{-}}(\hat{\rvx}_{t_n}^\vphi, t_n))\bigg[t_{n+1} \frac{\partial \vf_{\vtheta^-}(\rvx_{t_{n+1}}, t_{n+1})}{\partial \rvx_{t_{n+1}}}\vs_\vphi(\rvx_{t_{n+1}}, t_{n+1})\tau'(u_{n}) \\ - \frac{\partial \vf_{\vtheta^-}(\rvx_{t_{n+1}}, t_{n+1})}{\partial t_{n+1}}\tau'(u_{n})\bigg]\bigg\} + \mbb{E}[O(|\Delta u|)]
        \end{multlined}\notag \\
        =&\!\begin{multlined}[t][0.9\displaywidth]
            \nabla_\vtheta \mbb{E}\bigg\{\lambda(t_n)[\vf_\vtheta(\rvx_{t_{n+1}}, t_{n+1})]\tran \mH(\vf_{\vtheta^{-}}(\hat{\rvx}_{t_n}^\vphi, t_n))\bigg[t_{n+1} \frac{\partial \vf_{\vtheta^-}(\rvx_{t_{n+1}}, t_{n+1})}{\partial \rvx_{t_{n+1}}}\vs_\vphi(\rvx_{t_{n+1}}, t_{n+1})\tau'(u_{n}) \\ - \frac{\partial \vf_{\vtheta^-}(\rvx_{t_{n+1}}, t_{n+1})}{\partial t_{n+1}}\tau'(u_{n})\bigg]\bigg\} + \mbb{E}[O(|\Delta u|)]
        \end{multlined}\\\notag
        =&\!\begin{multlined}[t][0.9\displaywidth]
            \nabla_\vtheta \mbb{E}\bigg\{\frac{\lambda(t_n)}{(\tau^{-1})'(t_{n})}[\vf_\vtheta(\rvx_{t_{n+1}}, t_{n+1})]\tran \mH(\vf_{\vtheta^{-}}(\hat{\rvx}_{t_n}^\vphi, t_n))\bigg[t_{n+1} \frac{\partial \vf_{\vtheta^-}(\rvx_{t_{n+1}}, t_{n+1})}{\partial \rvx_{t_{n+1}}}\vs_\vphi(\rvx_{t_{n+1}}, t_{n+1}) \\ - \frac{\partial \vf_{\vtheta^-}(\rvx_{t_{n+1}}, t_{n+1})}{\partial t_{n+1}}\bigg]\bigg\} + \mbb{E}[O(|\Delta u|)]
        \end{multlined}\label{eq:ctcd2_grad}
    \end{align}
    Here (i) results from the chain rule, and (ii) follows from \cref{eq:ctcd1} and $\vf_\vtheta(\rvx, t) \equiv \vf_{\vtheta^{-}}(\rvx, t)$, since $\vtheta^{-} = \operatorname{stopgrad}(\vtheta)$. Taking the limit for both sides of \cref{eq:ctcd2_grad} as $\Delta u \to 0$ (or $N\to\infty$) yields \cref{eq:ctcd_obj2}, which completes the proof.
\end{proof}
\begin{remark}
    When $d(\rvx, \rvy) = \norm{\rvx - \rvy}_2^2$, the pseudo-objective $\mcal{L}_\text{CD}^\infty (\vtheta, \vtheta^{-}; \vphi)$ can be simplified to
    \begin{align}
        \mcal{L}_\text{CD}^{\infty} (\vtheta, \vtheta^{-}; \vphi) = 2 \mbb{E}\left[\frac{\lambda(t)}{(\tau^{-1})'(t)}\vf_\vtheta(\rvx_t, t) \tran \left(\frac{\partial \vf_{\vtheta^-}(\rvx_t, t)}{\partial t} - t \frac{\partial \vf_{\vtheta^-}(\rvx_t, t)}{\partial \rvx_t} \vs_\vphi(\rvx_{t}, t)\right)\right].
    \end{align}
\end{remark}
\begin{remark}
    The objective $\mcal{L}_\text{CD}^\infty(\vtheta, \vtheta^{-}; \vphi)$ defined in \cref{thm:ctcd2} is only meaningful in terms of its gradient---one cannot measure the progress of training by tracking the value of $\mcal{L}_\text{CD}^\infty(\vtheta, \vtheta^{-}; \vphi)$, but can still apply gradient descent to this objective to distill consistency models from pre-trained diffusion models. Because this objective is not a typical loss function, we refer to it as the ``pseudo-objective'' for consistency distillation.
\end{remark}
\begin{remark}\label{remark2}
    Following the same reasoning in \cref{remark}, we can easily derive that $\mcal{L}_\text{CD}^\infty(\vtheta, \vtheta^{-}; \vphi) = 0$ and $\nabla_\vtheta \mcal{L}_\text{CD}^\infty(\vtheta, \vtheta^{-}; \vphi) = \bm{0}$ if $\vf_\vtheta(\rvx, t)$ matches the ground truth consistency function for the empirical PF ODE that involves $\vs_\vphi(\rvx, t)$. However, the converse does not hold true in general. This distinguishes $\mcal{L}_\text{CD}^\infty(\vtheta, \vtheta^{-}; \vphi)$ from $\mcal{L}_\text{CD}^\infty(\vtheta, \vtheta; \vphi)$, the latter of which is a true loss function.
\end{remark}

\subsection{Consistency Training in Continuous Time}\label{sec:ctct}
A remarkable observation is that the pseudo-objective in \cref{thm:ctcd2} can be estimated without any pre-trained diffusion models, which enables direct consistency training of consistency models. More precisely, we have the following result.

\begin{theorem}\label{thm:ctct}
    Let $t_n = \tau(\frac{n-1}{N-1})$, where $n \in \llbracket 1, N \rrbracket$, and $\tau(\cdot)$ is a strictly monotonic function with $\tau(0) = \epsilon$ and $\tau(1) = T$. Assume $\tau$ is continuously differentiable in $[0,1]$, $d$ is three times continuously differentiable with bounded third derivatives, and $\vf_{\vtheta}$ is twice continuously differentiable with bounded first and second derivatives. Assume further that the weighting function $\lambda(\cdot)$ is bounded, $\mbb{E}[\norm{\nabla \log p_{t_n}(\rvx_{t_{n}})}_2^2] < \infty$, $\sup_{\rvx, t\in[\epsilon, T]}\norm{\nabla_\vtheta \vf_\vtheta(\rvx, t)}_2 < \infty$, and $\vphi$ represents diffusion model parameters that satisfy $\vs_\vphi(\rvx, t) \equiv \nabla \log p_t(\rvx)$. Then if $\vtheta^{-} = \operatorname{stopgrad}(\vtheta)$, we have
    \begin{align}
        \lim_{N \to \infty} (N-1)\nabla_\vtheta \mcal{L}_\text{CD}^N(\vtheta, \vtheta^{-}; \vphi) = \lim_{N \to \infty} (N-1)\nabla_\vtheta \mcal{L}_\text{CT}^N(\vtheta, \vtheta^{-}) = \nabla_\vtheta \mcal{L}_\text{CT}^\infty(\vtheta, \vtheta^{-}),\label{eq:ctct_obj}
    \end{align}
    where $\mcal{L}^N_\text{CD}$ uses the Euler ODE solver, and
    \begin{align}
        \mcal{L}_\text{CT}^{\infty} (\vtheta, \vtheta^{-}) \coloneqq \mbb{E}\left[\frac{\lambda(t)}{(\tau^{-1})'(t)} \vf_\vtheta(\rvx_t, t) \tran \mH(\vf_{\vtheta^-}(\rvx_t, t)) \left(\frac{\partial \vf_{\vtheta^-}(\rvx_t, t)}{\partial t} + \frac{\partial \vf_{\vtheta^-}(\rvx_t, t)}{\partial \rvx_t} \cdot \frac{\rvx_t - \rvx}{t}\right)\right].
    \end{align}
    Here the expectation above is taken over $\rvx \sim p_\text{data}$, $u\sim\mcal{U}[0,1]$, $t=\tau(u)$, and $\rvx_t \sim \mcal{N}(\rvx, t^2\mI)$.
\end{theorem}

\begin{proof}
    The proof mostly follows that of \cref{thm:ctcd2}. First, we leverage Taylor series expansion to obtain
    \begin{align}
        &(N-1)\mcal{L}_\text{CT}^N(\vtheta, \vtheta^{-}) = \frac{1}{\Delta u}\mcal{L}_\text{CT}^N(\vtheta, \vtheta^{-}) = \frac{1}{\Delta u} \mbb{E}[\lambda(t_n) d(\vf_\vtheta(\rvx + t_{n+1}\rvz, t_{n+1}), \vf_{\vtheta^{-}}(\rvx + t_n\rvz, t_n))]\notag \\
        \stackrel{(i)}{=}&\begin{multlined}[t][0.9\displaywidth]
            \frac{1}{2 \Delta u}\bigg(\mbb{E}\{\lambda(t_n)[\vf_\vtheta(\rvx + t_{n+1}\rvz, t_{n+1}) - \vf_{\vtheta^-}(\rvx + t_n \rvz, t_n)]\tran \mH(\vf_{\vtheta^{-}}(\rvx + t_n \rvz, t_n))\\ \cdot [\vf_\vtheta(\rvx + t_{n+1}\rvz, t_{n+1}) - \vf_{\vtheta^{-}}(\rvx + t_n \rvz, t_n)]\} + \mbb{E}[O(|\Delta u|^3)]\bigg)
        \end{multlined}\notag \\
        =&\begin{multlined}[t][0.9\displaywidth]
            \frac{1}{2\Delta u}\mbb{E}\{\lambda(t_n)[\vf_\vtheta(\rvx + t_{n+1}\rvz, t_{n+1}) - \vf_{\vtheta^-}(\rvx + t_n \rvz, t_n)]\tran \mH(\vf_{\vtheta^{-}}(\rvx + t_n \rvz, t_n))\\ \cdot [\vf_\vtheta(\rvx + t_{n+1}\rvz, t_{n+1}) - \vf_{\vtheta^-}(\rvx + t_n \rvz, t_n)]\} + \mbb{E}[O(|\Delta u|^2)]
        \end{multlined}\label{eq:ctct_1}
    \end{align}
    where $\rvz \sim \mcal{N}(\bm{0}, \mI)$, (i) is derived by first expanding $d(\cdot, \vf_{\vtheta^{-}}(\rvx + t_n\rvz, t_n))$ to second order, and then noting that $d(\rvx, \rvx) \equiv 0$ and $\nabla_\rvy d(\rvy, \rvx)|_{\rvy=\rvx} \equiv \bm{0}$. Next, we compute the gradient of \cref{eq:ctct_1} with respect to $\vtheta$ and simplify the result to obtain
    \begin{align}
        &(N-1) \nabla_\vtheta \mcal{L}_\text{CT}^N(\vtheta, \vtheta^{-}) = \frac{1}{\Delta u} \nabla_\vtheta \mcal{L}_\text{CT}^N(\vtheta, \vtheta^{-})\notag \\
        =&\begin{multlined}[t][0.9\displaywidth]
            \frac{1}{2\Delta u} \nabla_\vtheta \mbb{E}\{\lambda(t_n)[\vf_\vtheta(\rvx + t_{n+1}\rvz, t_{n+1}) - \vf_{\vtheta^-}(\rvx + t_n \rvz, t_n)]\tran \mH(\vf_{\vtheta^{-}}(\rvx + t_n \rvz, t_n))\\ \cdot[\vf_\vtheta(\rvx + t_{n+1}\rvz, t_{n+1}) - \vf_{\vtheta^-}(\rvx + t_n \rvz, t_n)]\} + \mbb{E}[O(|\Delta u|^2)]
        \end{multlined}\notag \\
        \stackrel{(i)}{=}&\begin{multlined}[t][0.9\displaywidth]
            \frac{1}{\Delta u} \mbb{E}\{\lambda(t_n)[\nabla_\vtheta \vf_\vtheta(\rvx + t_{n+1}\rvz, t_{n+1})]\tran \mH(\vf_{\vtheta^{-}}(\rvx + t_n \rvz, t_n))\\ \cdot [\vf_\vtheta(\rvx + t_{n+1}\rvz, t_{n+1}) - \vf_{\vtheta^-}(\rvx + t_n \rvz, t_n)]\} + \mbb{E}[O(|\Delta u|^2)]
        \end{multlined}\label{eq:continue} \\
        \stackrel{(ii)}{=}&\begin{multlined}[t][0.9\displaywidth]
                \frac{1}{\Delta u}\mbb{E}\bigg\{\lambda(t_n)[\nabla_\vtheta \vf_\vtheta(\rvx + t_{n+1}\rvz, t_{n+1})]\tran \mH(\vf_{\vtheta^{-}}(\rvx + t_n \rvz, t_n))\bigg[\tau'(u_n)\Delta u \partial_1 \vf_{\vtheta^{-}}(\rvx + t_n \rvz, t_n)\rvz \\ + \partial_2 \vf_{\vtheta^{-}}(\rvx + t_n \rvz, t_n)\tau'(u_n)\Delta u \bigg]\bigg\} + \mbb{E}[O(|\Delta u|)]
            \end{multlined}\notag \\
        =&\begin{multlined}[t][0.9\displaywidth]
            \mbb{E}\bigg\{\lambda(t_n)\tau'(u_n)[\nabla_\vtheta \vf_\vtheta(\rvx + t_{n+1}\rvz, t_{n+1})]\tran \mH(\vf_{\vtheta^{-}}(\rvx + t_n \rvz, t_n))\bigg[\partial_1 \vf_{\vtheta^{-}}(\rvx + t_n \rvz, t_n)\rvz \\ + \partial_2 \vf_{\vtheta^{-}}(\rvx + t_n \rvz, t_n)\bigg]\bigg\} + \mbb{E}[O(|\Delta u|)]
        \end{multlined}\notag \\
        =&\begin{multlined}[t][0.9\displaywidth]
            \nabla_\vtheta \mbb{E}\bigg\{\lambda(t_n)\tau'(u_n)[\vf_\vtheta(\rvx + t_{n+1}\rvz, t_{n+1})]\tran \mH(\vf_{\vtheta^{-}}(\rvx + t_n \rvz, t_n))\bigg[\partial_1 \vf_{\vtheta^{-}}(\rvx + t_n \rvz, t_n)\rvz \\ + \partial_2 \vf_{\vtheta^{-}}(\rvx + t_n \rvz, t_n)\bigg]\bigg\} + \mbb{E}[O(|\Delta u|)]
        \end{multlined}\notag \\
        =&\nabla_\vtheta \mbb{E}\bigg\{\lambda(t_n)\tau'(u_n)[\vf_\vtheta(\rvx_{t_{n+1}}, t_{n+1})]\tran \mH(\vf_{\vtheta^{-}}(\rvx_{t_n}, t_n))\bigg[\partial_1 \vf_{\vtheta^{-}}(\rvx_{t_n}, t_n)\frac{\rvx_{t_n} - \rvx}{t_n} + \partial_2 \vf_{\vtheta^{-}}(\rvx_{t_n}, t_n)\bigg]\bigg\} + \mbb{E}[O(|\Delta u|)]\notag \\
        =&\nabla_\vtheta \mbb{E}\bigg\{\frac{\lambda(t_n)}{(\tau^{-1})'(t_n)}[\vf_\vtheta(\rvx_{t_{n+1}}, t_{n+1})]\tran \mH(\vf_{\vtheta^{-}}(\rvx_{t_n}, t_n))\bigg[\partial_1 \vf_{\vtheta^{-}}(\rvx_{t_n}, t_n)\frac{\rvx_{t_n} - \rvx}{t_n} + \partial_2 \vf_{\vtheta^{-}}(\rvx_{t_n}, t_n)\bigg]\bigg\} + \mbb{E}[O(|\Delta u|)]\label{eq:ctct_grad}
    \end{align}
    Here (i) results from the chain rule, and (ii) follows from Taylor expansion. Taking the limit for both sides of \cref{eq:ctct_grad} as $\Delta u \to 0$ or $N\to\infty$ yields the second equality in \cref{eq:ctct_obj}.

    Now we prove the first equality. Applying Taylor expansion again, we obtain
    \begin{align*}
        &(N-1)\nabla_\vtheta \mcal{L}_\text{CD}^N(\vtheta, \vtheta^{-}; \vphi) = \frac{1}{\Delta u}\nabla_\vtheta \mcal{L}_\text{CD}^N(\vtheta, \vtheta^{-}; \vphi) = \frac{1}{\Delta u}\nabla_\vtheta \mbb{E}[\lambda(t_n)d(\vf_\vtheta(\rvx_{t_{n+1}}, t_{n+1}), \vf_{\vtheta^{-}}(\hat{\rvx}_{t_n}^\vphi, t_n))]\\
        =& \frac{1}{\Delta u}\mbb{E}[\lambda(t_n)\nabla_\vtheta d(\vf_\vtheta(\rvx_{t_{n+1}}, t_{n+1}), \vf_{\vtheta^{-}}(\hat{\rvx}_{t_n}^\vphi, t_n))]\\
        =& \frac{1}{\Delta u}\mbb{E}[\lambda(t_n) \nabla_\vtheta \vf_\vtheta(\rvx_{t_{n+1}}, t_{n+1})\tran \partial_1 d(\vf_\vtheta(\rvx_{t_{n+1}}, t_{n+1}), \vf_{\vtheta^{-}}(\hat{\rvx}_{t_n}^\vphi, t_n))]\\
        =& \frac{1}{\Delta u}\begin{multlined}[t][0.9\displaywidth]
            \mbb{E}\bigg\{\lambda(t_n) \nabla_\vtheta \vf_\vtheta(\rvx_{t_{n+1}}, t_{n+1})\tran \bigg[\partial_1 d(\vf_{\vtheta^{-}}(\hat{\rvx}_{t_n}^\vphi, t_n), \vf_{\vtheta^{-}}(\hat{\rvx}_{t_n}^\vphi, t_n)) \\+ \mH(\vf_{\vtheta^{-}}(\hat{\rvx}_{t_n}^\vphi, t_n)) (\vf_\vtheta(\rvx_{t_{n+1}}, t_{n+1}) - \vf_{\vtheta^{-}}(\hat{\rvx}_{t_n}^\vphi, t_n)) + O(|\Delta u|^2)\bigg]\bigg\}
        \end{multlined}\\
        =& \frac{1}{\Delta u}\mbb{E}\{\lambda(t_n) \nabla_\vtheta \vf_\vtheta(\rvx_{t_{n+1}}, t_{n+1})\tran [\mH(\vf_{\vtheta^{-}}(\hat{\rvx}_{t_n}^\vphi, t_n)) (\vf_\vtheta(\rvx_{t_{n+1}}, t_{n+1}) - \vf_{\vtheta^{-}}(\hat{\rvx}_{t_n}^\vphi, t_n))] + O(|\Delta u|^2)\}
        \\
        =& \frac{1}{\Delta u}\mbb{E}\{\lambda(t_n) \nabla_\vtheta \vf_\vtheta(\rvx_{t_{n+1}}, t_{n+1})\tran [\mH(\vf_{\vtheta^{-}}(\hat{\rvx}_{t_n}^\vphi, t_n)) (\vf_{\vtheta^{-}}(\rvx_{t_{n+1}}, t_{n+1}) - \vf_{\vtheta^{-}}(\hat{\rvx}_{t_n}^\vphi, t_n))] + O(|\Delta u|^2)\}\\
        \stackrel{(i)}{=}&\begin{multlined}[t][0.9\displaywidth]
            \frac{1}{\Delta u} \mbb{E}\{\lambda(t_n)[\nabla_\vtheta \vf_\vtheta(\rvx + t_{n+1}\rvz, t_{n+1})]\tran \mH(\vf_{\vtheta^{-}}(\rvx + t_n \rvz, t_n))\\ \cdot [\vf_\vtheta(\rvx + t_{n+1}\rvz, t_{n+1}) - \vf_{\vtheta^-}(\rvx + t_n \rvz, t_n)]\} + \mbb{E}[O(|\Delta u|^2)]
        \end{multlined}
    \end{align*}
    where (i) holds because $\rvx_{t_{n+1}} = \rvx + t_{n+1} \rvz$ and $\hat{\rvx}_{t_n}^\vphi = \rvx_{t_{n+1}} -(t_n - t_{n+1}) t_{n+1} \frac{-(\rvx_{t_{n+1}} - \rvx)}{t_{n+1}^2} = \rvx_{t_{n+1}} + (t_n - t_{n+1}) \rvz = \rvx + t_n \rvz$. Because (i) matches \cref{eq:continue}, we can use the same reasoning procedure from \cref{eq:continue} to \cref{eq:ctct_grad} to conclude $\lim_{N \to \infty} (N-1)\nabla_\vtheta \mcal{L}_\text{CD}^N(\vtheta, \vtheta^{-}; \vphi) = \lim_{N \to \infty} (N-1)\nabla_\vtheta \mcal{L}_\text{CT}^N(\vtheta, \vtheta^{-})$, completing the proof.
\end{proof}
\begin{remark}
    Note that $\mcal{L}_\text{CT}^\infty(\vtheta, \vtheta^{-})$ does not depend on the diffusion model parameter $\vphi$ and hence can be optimized without any pre-trained diffusion models.
\end{remark}
\begin{remark}
    When $d(\rvx, \rvy) = \norm{\rvx - \rvy}_2^2$, the continuous-time consistency training objective becomes
    \begin{align}
        \mcal{L}_\text{CT}^{\infty} (\vtheta, \vtheta^{-}) = 2\mbb{E}\left[\frac{\lambda(t)}{(\tau^{-1})'(t)} \vf_\vtheta(\rvx_t, t) \tran \left(\frac{\partial \vf_{\vtheta^-}(\rvx_t, t)}{\partial t} + \frac{\partial \vf_{\vtheta^-}(\rvx_t, t)}{\partial \rvx_t} \cdot \frac{\rvx_t - \rvx}{t}\right)\right].
    \end{align}
\end{remark}
\begin{remark}
    Similar to $\mcal{L}_\text{CD}^\infty (\vtheta, \vtheta^{-}; \vphi)$ in \cref{thm:ctcd2}, $\mcal{L}_\text{CT}^\infty(\vtheta, \vtheta^{-})$ is a pseudo-objective; one cannot track training by monitoring the value of $\mcal{L}_\text{CT}^\infty(\vtheta, \vtheta^{-})$, but can still apply gradient descent on this loss function to train a consistency model $\vf_\vtheta(\rvx, t)$ directly from data. Moreover, the same observation in \cref{remark2} holds true: $\mcal{L}_\text{CT}^\infty(\vtheta, \vtheta^{-}) = 0$ and $\nabla_\vtheta \mcal{L}_\text{CT}^\infty(\vtheta, \vtheta^{-}) = \bm{0}$ if $\vf_\vtheta(\rvx, t)$ matches the ground truth consistency function for the PF ODE.
\end{remark}
\subsection{Experimental Verifications}
\begin{figure}
    \centering
    \begin{subfigure}[b]{0.4\linewidth}
        \includegraphics[width=\linewidth]{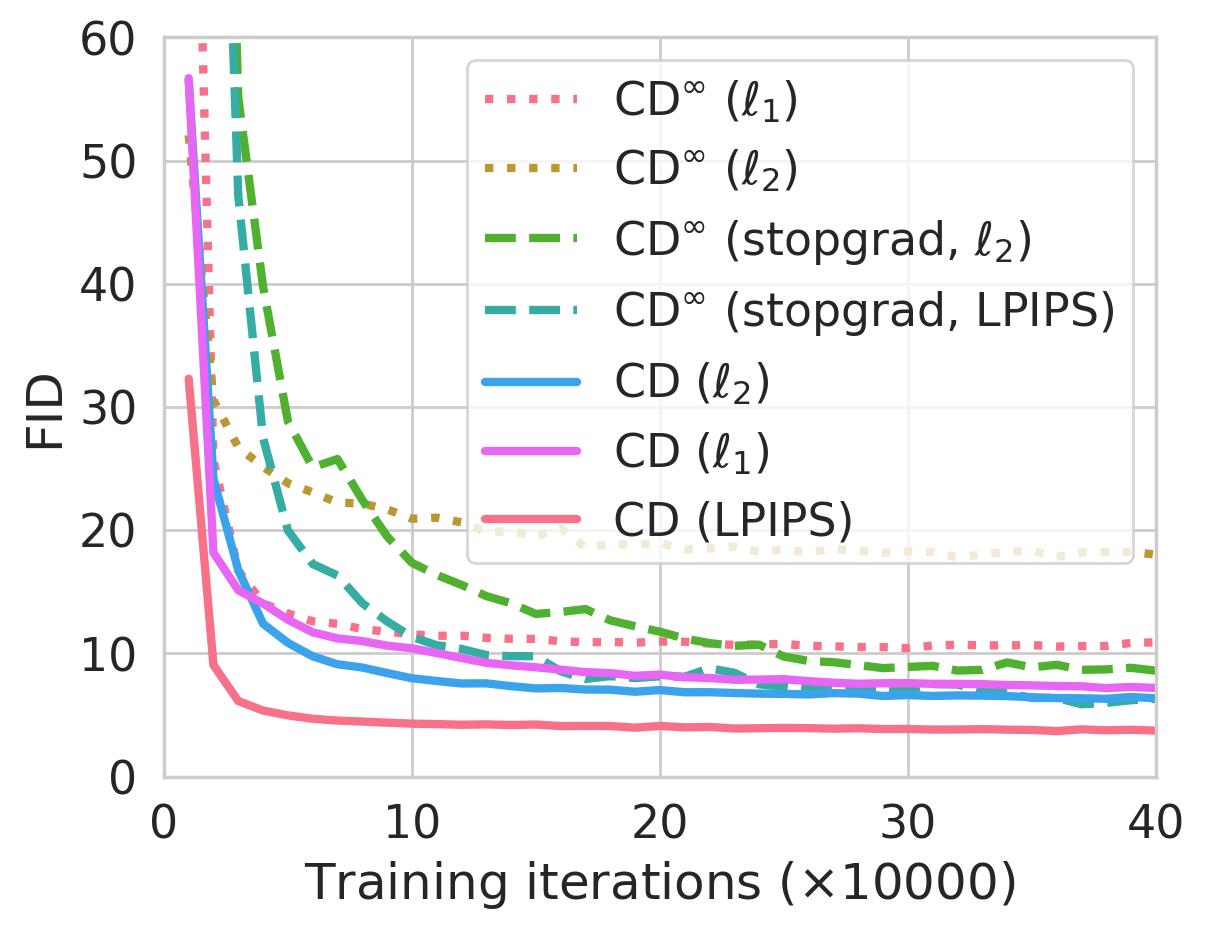}
        \caption{Consistency Distillation}\label{fig:distillation_compare}
    \end{subfigure}
    \begin{subfigure}[b]{0.4\linewidth}
        \includegraphics[width=\linewidth]{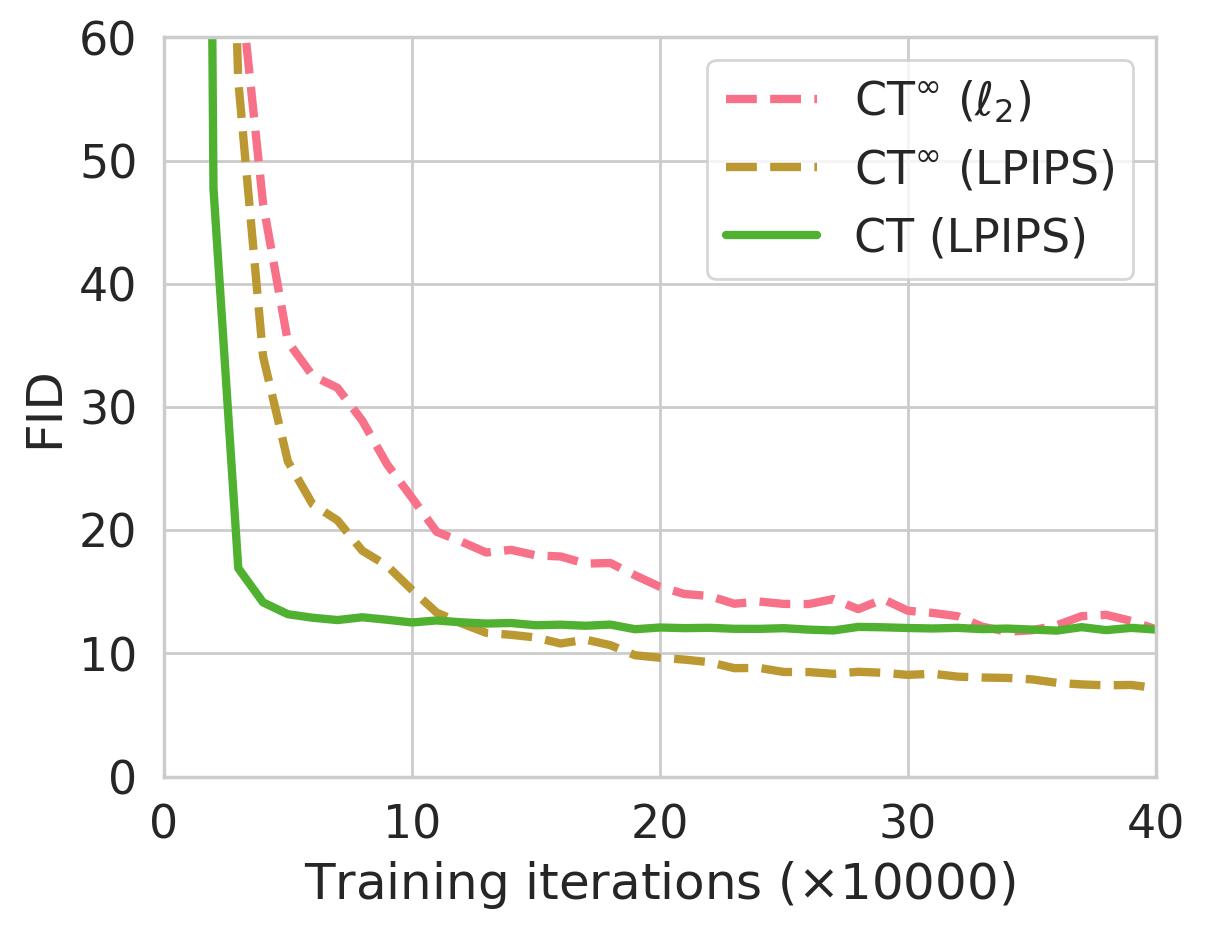}
        \caption{Consistency Training}\label{fig:gen_compare}
    \end{subfigure}
    \caption{Comparing discrete consistency distillation/training algorithms with continuous counterparts.}\label{fig:compare}
\end{figure}
To experimentally verify the efficacy of our continuous-time CD and CT objectives, we train consistency models with a variety of loss functions on CIFAR-10. All results are provided in \cref{fig:compare}. We set $\lambda(t) = (\tau^{-1})'(t)$ for all continuous-time experiments. Other hyperparameters are the same as in \cref{tab:hyperparameters}. We occasionally modify some hyperparameters for improved performance. For distillation, we compare the following objectives:
\begin{itemize}
    \item CD $(\ell_2)$: Consistency distillation $\mcal{L}^{N}_\text{CD}$ with $N=18$ and the $\ell_2$ metric.
    \item CD $(\ell_1)$: Consistency distillation $\mcal{L}^{N}_\text{CD}$ with $N=18$ and the $\ell_1$ metric. We set the learning rate to 2e-4.
    \item CD (LPIPS): Consistency distillation $\mcal{L}^{N}_\text{CD}$ with $N=18$ and the LPIPS metric.
    \item CD$^\infty$ $(\ell_2)$: Consistency distillation $\mcal{L}^\infty_\text{CD}$ in \cref{thm:ctcd1} with the $\ell_2$ metric. We set the learning rate to 1e-3 and dropout to 0.13.
    \item CD$^\infty$ $(\ell_1)$: Consistency distillation $\mcal{L}^\infty_\text{CD}$ in \cref{thm:ctcd_l1} with the $\ell_1$ metric. We set the learning rate to 1e-3 and dropout to 0.3.
    \item CD$^\infty$ (stopgrad, $\ell_2$): Consistency distillation $\mcal{L}^\infty_\text{CD}$ in \cref{thm:ctcd2} with the $\ell_2$ metric. We set the learning rate to 5e-6.
    \item CD$^\infty$ (stopgrad, LPIPS): Consistency distillation $\mcal{L}^\infty_\text{CD}$ in \cref{thm:ctcd2} with the LPIPS metric. We set the learning rate to 5e-6.
\end{itemize}
We did not investigate using the LPIPS metric in \cref{thm:ctcd1} because minimizing the resulting objective would require back-propagating through second order derivatives of the VGG network used in LPIPS, which is computationally expensive and prone to numerical instability. As revealed by \cref{fig:distillation_compare}, the stopgrad version of continuous-time distillation (\cref{thm:ctcd2}) works better than the non-stopgrad version (\cref{thm:ctcd1}) for both the LPIPS and $\ell_2$ metrics, and the LPIPS metric works the best for all distillation approaches. Additionally, discrete-time consistency distillation outperforms continuous-time consistency distillation, possibly due to the larger variance in continuous-time objectives, and the fact that one can use effective higher-order ODE solvers in discrete-time objectives.

For consistency training (CT), we find it important to initialize consistency models from a pre-trained EDM model in order to stabilize training when using continuous-time objectives. We hypothesize that this is caused by the large variance in our continuous-time loss functions. For fair comparison, we thus initialize all consistency models from the same pre-trained EDM model on CIFAR-10 for both discrete-time and continuous-time CT, even though the former works well with random initialization. We leave variance reduction techniques for continuous-time CT to future research.

We empirically compare the following objectives:
\begin{itemize}
    \item CT (LPIPS): Consistency training $\mcal{L}_\text{CT}^N$ with $N=120$ and the LPIPS metric. We set the learning rate to 4e-4, and the EMA decay rate for the target network to 0.99. We do not use the schedule functions for $N$ and $\mu$ here because they cause slower learning when the consistency model is initialized from a pre-trained EDM model.
    \item CT$^\infty$ $(\ell_2)$: Consistency training $\mcal{L}^{\infty}_\text{CT}$ with the $\ell_2$ metric. We set the learning rate to 5e-6.
    \item CT$^\infty$ (LPIPS): Consistency training $\mcal{L}^\infty_\text{CT}$ with the LPIPS metric. We set the learning rate to 5e-6.
\end{itemize}
As shown in \cref{fig:gen_compare}, the LPIPS metric leads to improved performance for continuous-time CT. We also find that continuous-time CT outperforms discrete-time CT with the same LPIPS metric. This is likely due to the bias in discrete-time CT, as $\Delta t > 0$ in \cref{thm:ct} for discrete-time objectives, whereas continuous-time CT has no bias since it implicitly drives $\Delta t$ to $0$.

\section{Additional Experimental Details}\label{app:exp}
\begin{table}
    \setlength{\tabcolsep}{15pt}
    \caption{Hyperparameters used for training CD and CT models}\label{tab:hyperparameters}
    \centering
    \begin{adjustbox}{max width=\linewidth}
        \begin{tabular}{l|cc|cc|cc}
            \Xhline{3\arrayrulewidth}
            Hyperparameter & \multicolumn{2}{c|}{CIFAR-10} & \multicolumn{2}{c|}{ImageNet $64\times 64$} & \multicolumn{2}{c}{LSUN $256\times 256$} \\
            & CD & CT & CD & CT & CD & CT\\
            \hline
            Learning rate & 4e-4 & 4e-4 & 8e-6 & 8e-6 & 1e-5 & 1e-5\\
            Batch size & 512 & 512 & 2048 & 2048 & 2048 & 2048\\
            $\mu$ & 0 &  & 0.95 & & 0.95 & \\
            $\mu_0$ & & 0.9 & & 0.95 & & 0.95\\
            $s_0$ & & 2 & & 2 & & 2 \\
            $s_1$ & & 150 & & 200 & & 150 \\
            $N$ & 18 & & 40 & & 40 & \\
            ODE solver & Heun & & Heun & & Heun & \\
            EMA decay rate & 0.9999 & 0.9999 & 0.999943 & 0.999943 & 0.999943 & 0.999943\\
            Training iterations & 800k & 800k & 600k & 800k & 600k & 1000k\\
            Mixed-Precision (FP16) & No & No & Yes & Yes & Yes & Yes\\
            Dropout probability & 0.0 & 0.0 & 0.0 & 0.0 & 0.0 & 0.0\\
            Number of GPUs & 8 & 8 & 64 & 64 & 64 & 64\\
            \Xhline{3\arrayrulewidth}
        \end{tabular}
    \end{adjustbox}
\end{table}

\paragraph{Model Architectures} We follow \citet{song2021scorebased,dhariwal2021diffusion} for model architectures. Specifically, we use the NCSN++ architecture in \citet{song2021scorebased} for all CIFAR-10 experiments, and take the corresponding network architectures from \citet{dhariwal2021diffusion} when performing experiments on ImageNet $64\times 64$, LSUN Bedroom $256\times 256$ and LSUN Cat $256\times 256$.

\paragraph{Parameterization for Consistency Models} We use the same architectures for consistency models as those used for EDMs. The only difference is we slightly modify the skip connections in EDM to ensure the boundary condition holds for consistency models. Recall that in \cref{sec:consistency} we propose to parameterize a consistency model in the following form:
\begin{align*}
    \vf_\vtheta(\rvx, t) = c_\text{skip}(t)\rvx + c_\text{out}(t) F_\vtheta(\rvx, t).
\end{align*}
In EDM \cite{karras2022edm}, authors choose
\begin{align*}
    c_\text{skip}(t) = \frac{\sigma_\text{data}^2}{t^2 + \sigma_\text{data}^2},\quad c_\text{out}(t) = \frac{\sigma_\text{data} t }{\sqrt{\sigma_\text{data}^2 + t^2}},
\end{align*}
where $\sigma_\text{data} = 0.5$. However, this choice of $c_\text{skip}$ and $c_\text{out}$ does not satisfy the boundary condition when the smallest time instant $\epsilon \neq 0$. To remedy this issue, we modify them to
\begin{align*}
    c_\text{skip}(t) = \frac{\sigma_\text{data}^2}{(t-\epsilon)^2 + \sigma_\text{data}^2},\quad c_\text{out}(t) = \frac{\sigma_\text{data} (t-\epsilon) }{\sqrt{\sigma_\text{data}^2 + t^2}},
\end{align*}
which clearly satisfies $c_\text{skip}(\epsilon) = 1$ and $c_\text{out}(\epsilon) = 0$.

\paragraph{Schedule Functions for Consistency Training} As discussed in \cref{sec:generation}, consistency generation requires specifying schedule functions $N(\cdot)$ and $\mu(\cdot)$ for best performance. Throughout our experiments, we use schedule functions that take the form below:
\begin{align*}
    N(k) &= \left\lceil \sqrt{\frac{k}{K} ((s_1 + 1)^2 - s_0^2) + s_0^2} - 1 \right\rceil + 1\\
    \mu(k) &= \exp\left(\frac{s_0 \log \mu_0}{N(k)}\right),
\end{align*}
where $K$ denotes the total number of training iterations, $s_0$ denotes the initial discretization steps, $s_1 > s_0$ denotes the target discretization steps at the end of training, and $\mu_0 > 0$ denotes the EMA decay rate at the beginning of model training.

\paragraph{Training Details}
In both consistency distillation and progressive distillation, we distill EDMs \cite{karras2022edm}. We trained these EDMs ourselves according to the specifications given in \citet{karras2022edm}. The original EDM paper did not provide hyperparameters for the LSUN Bedroom $256\times 256$ and Cat $256\times 256$ datasets, so we mostly used the same hyperparameters as those for the ImageNet $64\times 64$ dataset. The difference is that we trained for 600k and 300k iterations for the LSUN Bedroom and Cat datasets respectively, and reduced the batch size from 4096 to 2048.

We used the same EMA decay rate for LSUN $256\times 256$ datasets as for the ImageNet $64\times 64$ dataset. For progressive distillation, we used the same training settings as those described in \citet{salimans2022progressive} for CIFAR-10 and ImageNet $64\times 64$. Although the original paper did not test on LSUN $256\times 256$ datasets, we used the same settings for ImageNet $64\times 64$ and found them to work well.

In all distillation experiments, we initialized the consistency model with pre-trained EDM weights. For consistency training, we initialized the model randomly, just as we did for training the EDMs. We trained all consistency models with the Rectified Adam optimizer \cite{liu2019variance}, with no learning rate decay or warm-up, and no weight decay. We also applied EMA to the weights of the online consistency models in both consistency distillation and consistency training, as well as to the weights of the training online consistency models according to \citet{karras2022edm}. For LSUN $256\times 256$ datasets, we chose the EMA decay rate to be the same as that for ImageNet $64\times 64$, except for consistency distillation on LSUN Bedroom $256\times 256$, where we found that using zero EMA worked better.

When using the LPIPS metric on CIFAR-10 and ImageNet $64\times 64$, we rescale images to resolution $224\times 224$ with bilinear upsampling before feeding them to the LPIPS network. For LSUN $256\times 256$, we evaluated LPIPS without rescaling inputs. In addition, we performed horizontal flips for data augmentation for all models and on all datasets. We trained all models on a cluster of Nvidia A100 GPUs. Additional hyperparameters for consistency training and distillation are listed in \cref{tab:hyperparameters}.

\section{Additional Results on Zero-Shot Image Editing}\label{app:editing}
\begin{algorithm}[tb]
    \caption{Zero-Shot Image Editing}
    \label{alg:editing}
 \begin{algorithmic}[1]
    \STATE {\bfseries Input:} Consistency model $\vf_\vtheta(\cdot, \cdot)$, sequence of time points $t_1 > t_2 > \cdots > t_{N}$, reference image $\rvy$, invertible linear transformation $\mA$, and binary image mask $\bm{\Omega}$
    \STATE $\rvy \gets \mA^{-1}[(\mA \rvy) \odot (1 - \bm{\Omega}) + \bm{0} \odot \bm{\Omega}]$
    \STATE Sample $\rvx \sim \mcal{N}(\rvy, t_1^2 \mI)$
    \STATE $\rvx \gets \vf_\vtheta(\rvx, t_1)$
    \STATE $\rvx \gets \mA^{-1}[(\mA \rvy) \odot (1 - \bm{\Omega}) + (\mA\rvx) \odot \bm{\Omega}]$
    \FOR{$n=2$ {\bfseries to} $N$}
        \STATE Sample $\rvx \sim \mcal{N}(\rvx, (t_n^2 - \epsilon^2) \mI)$
        \STATE $\rvx \gets \vf_\vtheta(\rvx, t_n)$
        \STATE $\rvx \gets \mA^{-1}[(\mA\rvy) \odot (1-\bm{\Omega}) + (\mA\rvx) \odot \bm{\Omega}]$
    \ENDFOR
    \STATE {\bfseries Output:} $\rvx$
 \end{algorithmic}
\end{algorithm}

With consistency models, we can perform a variety of zero-shot image editing tasks. As an example, we present additional results on colorization (\cref{fig:bedroom_colorization}), super-resolution (\cref{fig:bedroom_superres}), inpainting (\cref{fig:bedroom_inpainting}), interpolation (\cref{fig:bedroom_interp}), denoising (\cref{fig:bedroom_denoising}), and stroke-guided image generation (SDEdit, \citet{meng2021sdedit}, \cref{fig:bedroom_sdedit}). The consistency model used here is trained via consistency distillation on the LSUN Bedroom $256\times 256$.

All these image editing tasks, except for image interpolation and denoising, can be performed via a small modification to the multistep sampling algorithm in \cref{alg:sampling}. The resulting pseudocode is provided in \cref{alg:editing}. Here $\rvy$ is a reference image that guides sample generation, $\bm{\Omega}$ is a binary mask, $\odot$ computes element-wise products, and $\mA$ is an invertible linear transformation that maps images into a latent space where the conditional information in $\rvy$ is infused into the iterative generation procedure by masking with $\bm{\Omega}$. Unless otherwise stated, we choose
\begin{align*}
    t_i = \left(T^{1/\rho} + \frac{i-1}{N-1}(\epsilon^{1/\rho} - T^{1/\rho})\right)^\rho
\end{align*}
in our experiments, where $N=40$ for LSUN Bedroom $256\times 256$.

Below we describe how to perform each task using \cref{alg:editing}.
\paragraph{Inpainting}
When using \cref{alg:editing} for inpainting, we let $\rvy$ be an image where missing pixels are masked out, $\bm{\Omega}$ be a binary mask where 1 indicates the missing pixels, and $\mA$ be the identity transformation.

\paragraph{Colorization}
The algorithm for image colorization is similar, as colorization becomes a special case of inpainting once we transform data into a decoupled space. Specifically, let $\rvy \in \mbb{R}^{h\times w\times 3}$ be a gray-scale image that we aim to colorize, where all channels of $\rvy$ are assumed to be the same, \ie, $\rvy[:, :, 0] = \rvy[:, :, 1] = \rvy[:, :, 2]$ in NumPy notation. In our experiments, each channel of this gray scale image is obtained from a colorful image by averaging the RGB channels with
\begin{align*}
    0.2989 R + 0.5870 G + 0.1140 B.
\end{align*}
We define $\bm{\Omega} \in \{0, 1\}^{h\times w\times 3}$ to be a binary mask such that
\begin{align*}
    \bm{\Omega}[i, j, k] = \begin{cases}
        1, &\quad \text{$k=1$ or $2$}\\
        0, &\quad \text{$k=0$}
    \end{cases}.
\end{align*}
Let $\mQ \in \mbb{R}^{3\times 3}$ be an orthogonal matrix whose first column is proportional to the vector $(0.2989, 0.5870, 0.1140)$. This orthogonal matrix can be obtained easily via QR decomposition, and we use the following in our experiments
\begin{align*}
    \mQ = \begin{pmatrix}
        0.4471 & -0.8204 & 0.3563\\
        0.8780 & 0.4785 &  0 \\
        0.1705 & -0.3129 & -0.9343
    \end{pmatrix}.
\end{align*}
We then define the linear transformation $\mA: \rvx \in \mbb{R}^{h\times w \times 3} \mapsto \rvy \in \mbb{R}^{h\times w \times 3}$, where
\begin{align*}
    \rvy[i, j, k] = \sum_{l=0}^2 \rvx[i, j, l] \mQ[l, k].
\end{align*}
Because $\mQ$ is orthogonal, the inversion $\mA^{-1} : \rvy \in \mbb{R}^{h \times w \time 3} \mapsto \rvx \in \mbb{R}^{h \times w \times 3}$ is easy to compute, where
\begin{align*}
    \rvx[i, j, k] = \sum_{l=0}^2 \rvy[i, j, l] \mQ[k, l].
\end{align*}
With $\mA$ and $\bm{\Omega}$ defined as above, we can now use \cref{alg:editing} for image colorization.

\paragraph{Super-resolution} With a similar strategy, we employ \cref{alg:editing} for image super-resolution. For simplicity, we assume that the down-sampled image is obtained by averaging non-overlapping patches of size $p\times p$. Suppose the shape of full resolution images is $h \times w \times 3$. Let $\rvy \in \mbb{R}^{h\times w\times 3}$ denote a low-resolution image naively up-sampled to full resolution, where pixels in each non-overlapping patch share the same value. Additionally, let $\bm{\Omega} \in \{0, 1\}^{h/p\times w/p \times p^2 \times 3}$ be a binary mask such that
\begin{align*}
    \bm{\Omega}[i, j, k, l] = \begin{cases}
        1, &\quad k \geq 1\\
        0, &\quad k=0
    \end{cases}.
\end{align*}
Similar to image colorization, super-resolution requires an orthogonal matrix $\mQ \in \mbb{R}^{p^2 \times p^2}$ whose first column is $(\nicefrac{1}{p}, \nicefrac{1}{p}, \cdots, \nicefrac{1}{p})$. This orthogonal matrix can be obtained with QR decomposition. To perform super-resolution, we define the linear transformation $\mA: \rvx \in \mbb{R}^{h \times w\times 3} \mapsto \rvy \in \mbb{R}^{h/p\times w/p \times p^2 \times 3}$, where
\begin{align*}
    \rvy[i, j, k, l] = \sum_{m=0}^{p^2-1} \rvx[i\times p+ (m - m\bmod p)/p, j \times p + m \bmod p, l ] \mQ[m, k].
\end{align*}
The inverse transformation $\mA^{-1}: \rvy \in \mbb{R}^{h/p\times w/p \times p^2 \times 3} \mapsto \rvx \in \mbb{R}^{h \times w\times 3}$ is easy to derive, with
\begin{align*}
    \rvx[i, j, k, l] = \sum_{m=0}^{p^2-1} \rvy[i\times p+ (m - m\bmod p)/p, j \times p + m \bmod p, l ] \mQ[k, m].
\end{align*}
Above definitions of $\mA$ and $\bm{\Omega}$ allow us to use \cref{alg:editing} for image super-resolution.

\paragraph{Stroke-guided image generation} We can also use \cref{alg:editing} for stroke-guided image generation as introduced in SDEdit \cite{meng2021sdedit}. Specifically, we let $\rvy \in \mbb{R}^{h\times w \times 3}$ be a stroke painting. We set $\mA = \mI$, and define $\bm{\Omega}\in \mbb{R}^{h\times w \times 3}$ as a matrix of ones. In our experiments, we set $t_1 = 5.38$ and $t_2 = 2.24$, with $N=2$.

\paragraph{Denoising} It is possible to denoise images perturbed with various scales of Gaussian noise using a single consistency model. Suppose the input image $\rvx$ is perturbed with $\mcal{N}(\bm{0}; \sigma^2 \mI)$. As long as $\sigma \in [\epsilon, T]$, we can evaluate $\vf_\vtheta(\rvx, \sigma)$ to produce the denoised image.

\paragraph{Interpolation} We can interpolate between two images generated by consistency models. Suppose the first sample $\rvx_1$ is produced by noise vector $\rvz_1$, and the second sample $\rvx_2$ is produced by noise vector $\rvz_2$. In other words, $\rvx_1 = \vf_\vtheta(\rvz_1, T)$ and $\rvx_2 = \vf_\vtheta(\rvz_2, T)$. To interpolate between $\rvx_1$ and $\rvx_2$, we first use spherical linear interpolation to get
\begin{align*}
    \rvz = \frac{\sin[(1-\alpha) \psi]}{\sin(\psi)} \rvz_1 + \frac{\sin(\alpha \psi)}{\sin(\psi)}\rvz_2,
\end{align*}
where $\alpha \in [0, 1]$ and $\psi = \arccos(\frac{\rvz_1\tran \rvz_2}{\norm{\rvz_1}_2 \norm{\rvz_2}_2})$, then evaluate $\vf_\vtheta(\rvz, T)$ to produce the interpolated image.

\begin{figure}
    \centering
    \begin{subfigure}[b]{0.11\textwidth}
        \includegraphics[width=\textwidth]{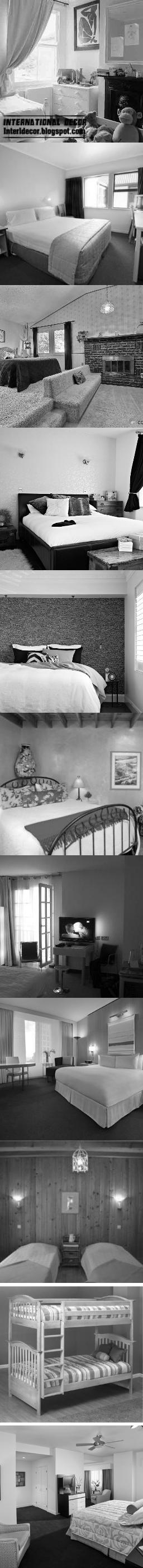}
    \end{subfigure}\hfill
    \begin{subfigure}[b]{0.77\textwidth}
        \includegraphics[width=\textwidth]{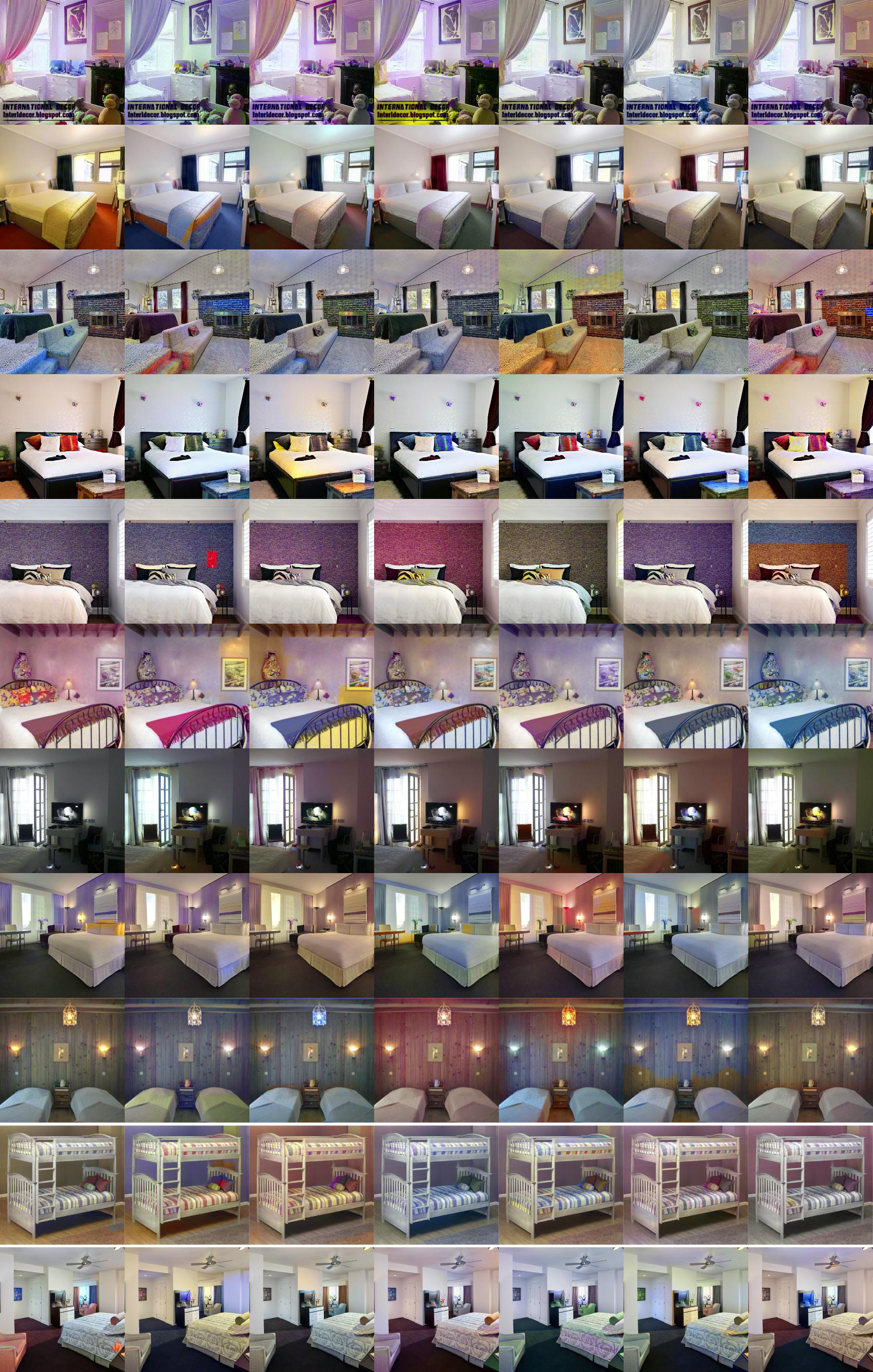}
    \end{subfigure}\hfill
    \begin{subfigure}[b]{0.11\textwidth}
        \includegraphics[width=\textwidth]{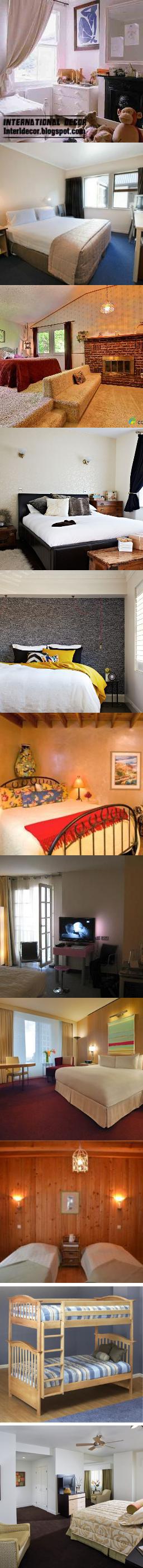}
    \end{subfigure}\hfill
    \caption{Gray-scale images (left), colorized images by a consistency model (middle), and ground truth (right).}
    \label{fig:bedroom_colorization}
\end{figure}

\begin{figure}
    \centering
    \begin{subfigure}[b]{0.11\textwidth}
        \includegraphics[width=\textwidth]{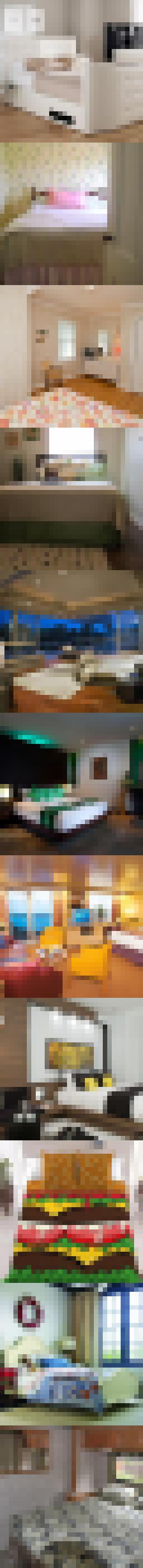}
    \end{subfigure}\hfill
    \begin{subfigure}[b]{0.77\textwidth}
        \includegraphics[width=\textwidth]{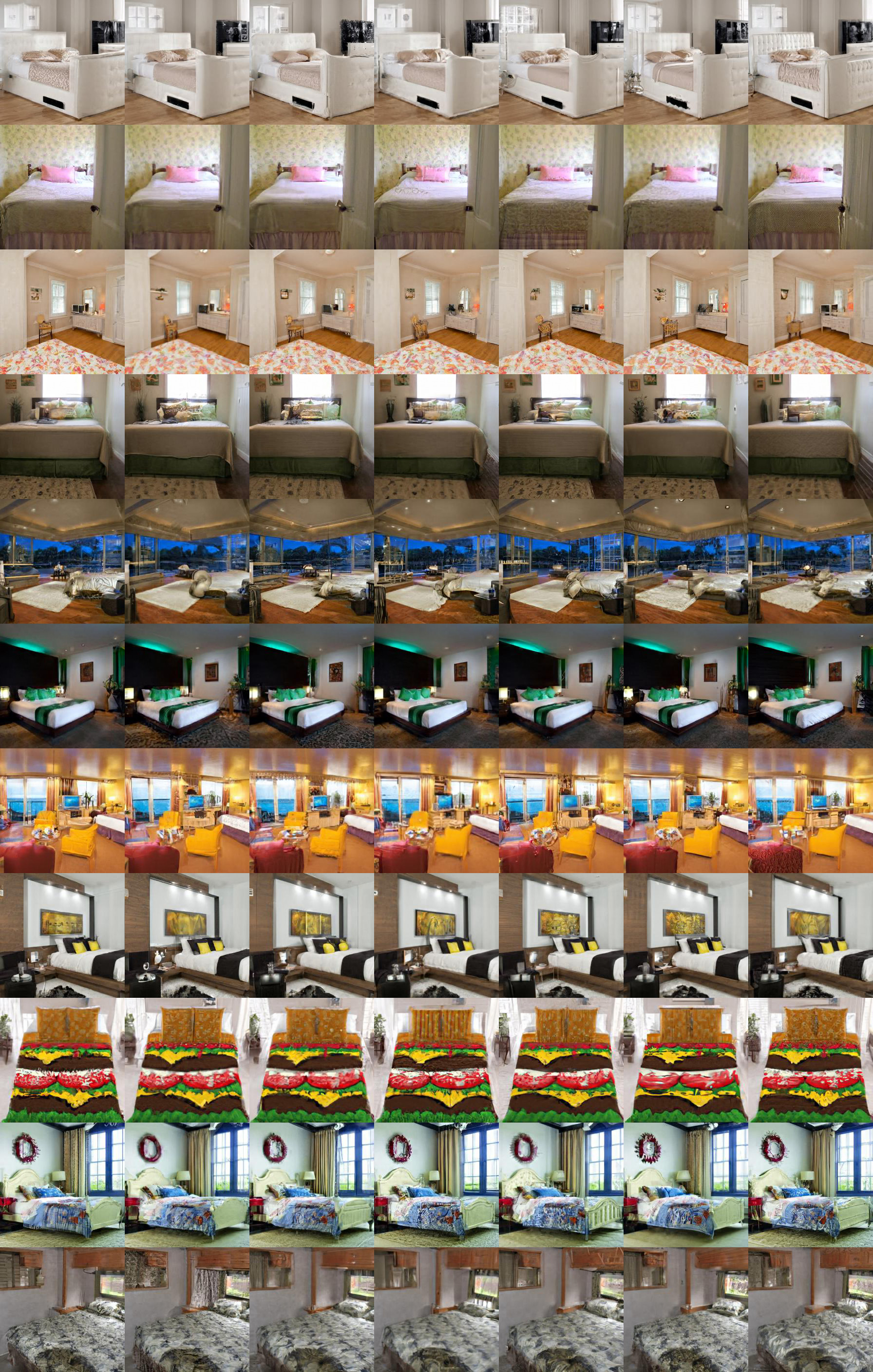}
    \end{subfigure}\hfill
    \begin{subfigure}[b]{0.11\textwidth}
        \includegraphics[width=\textwidth]{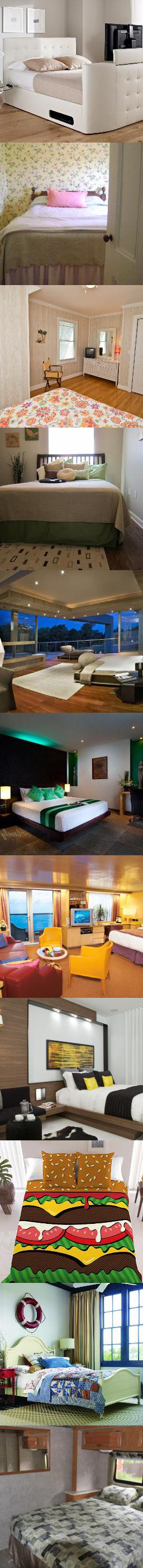}
    \end{subfigure}\hfill
    \caption{Downsampled images of resolution $32\times 32$ (left), full resolution ($256\times 256$) images generated by a consistency model (middle), and ground truth images of resolution $256\times 256$ (right).}
    \label{fig:bedroom_superres}
\end{figure}

\begin{figure}
    \centering
    \begin{subfigure}[b]{0.11\textwidth}
        \includegraphics[width=\textwidth]{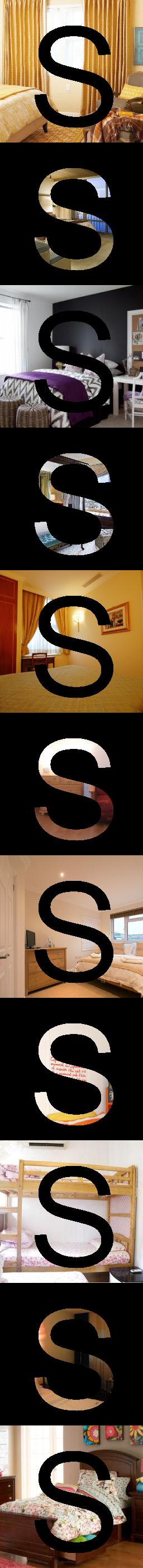}
    \end{subfigure}\hfill
    \begin{subfigure}[b]{0.77\textwidth}
        \includegraphics[width=\textwidth]{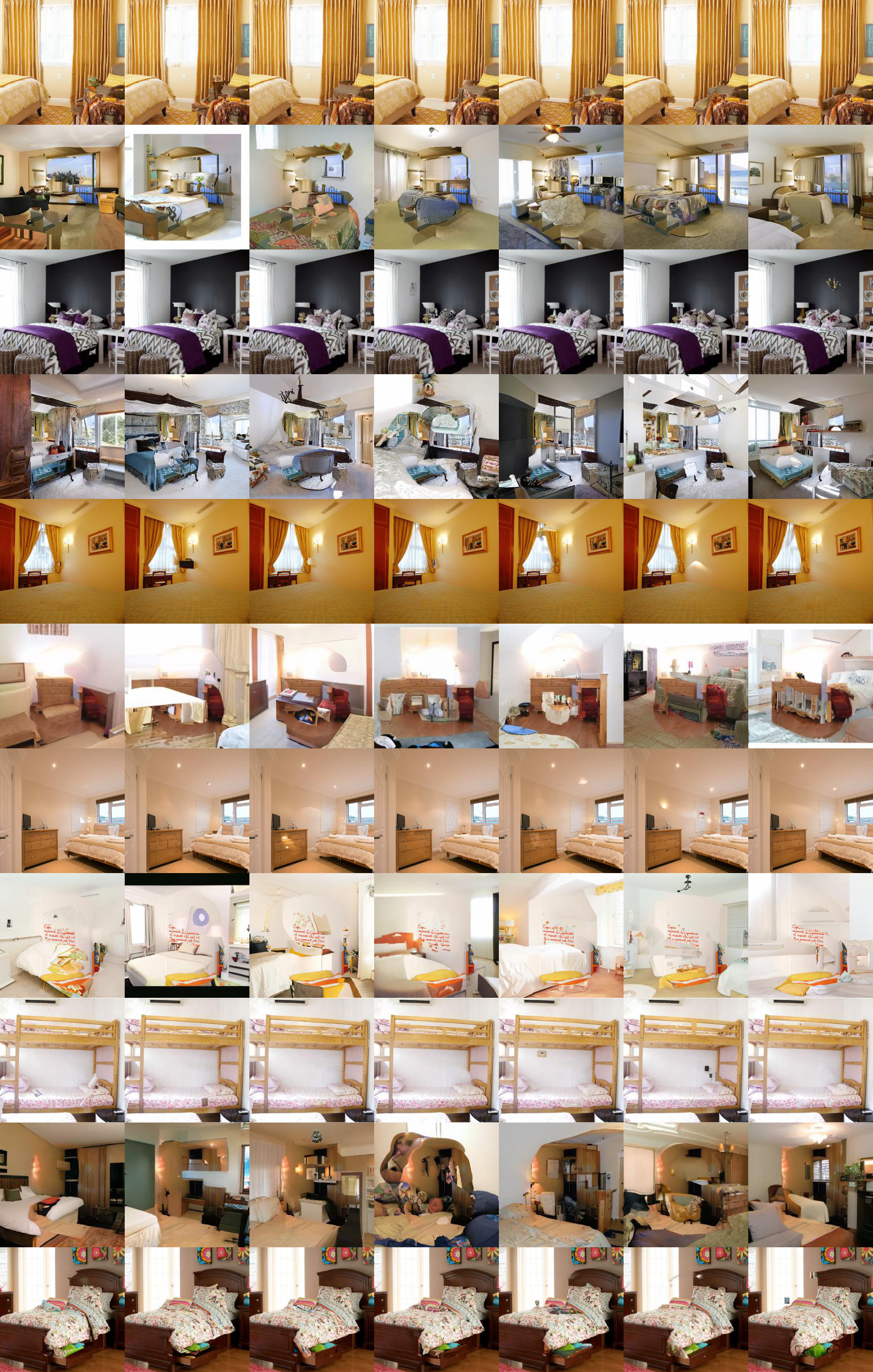}
    \end{subfigure}\hfill
    \begin{subfigure}[b]{0.11\textwidth}
        \includegraphics[width=\textwidth]{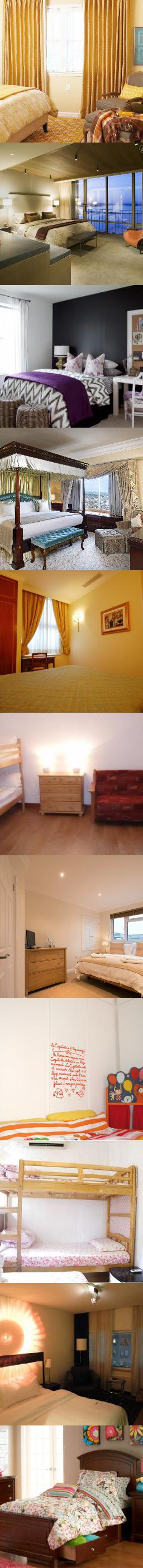}
    \end{subfigure}\hfill
    \caption{Masked images (left), imputed images by a consistency model (middle), and ground truth (right).}
    \label{fig:bedroom_inpainting}
\end{figure}

\begin{figure}
    \centering
    \begin{subfigure}[b]{0.11\textwidth}
        \includegraphics[width=\textwidth]{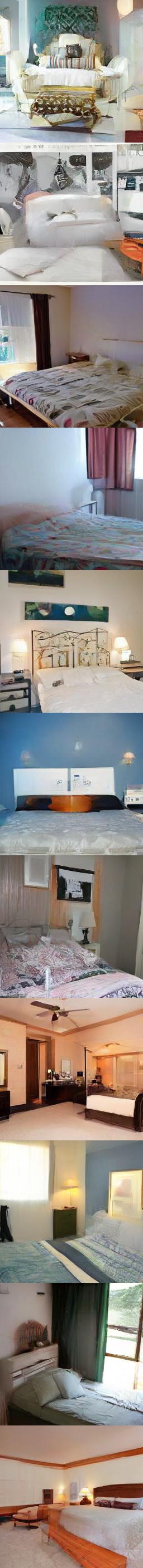}
    \end{subfigure}\hfill
    \begin{subfigure}[b]{0.77\textwidth}
        \includegraphics[width=\textwidth]{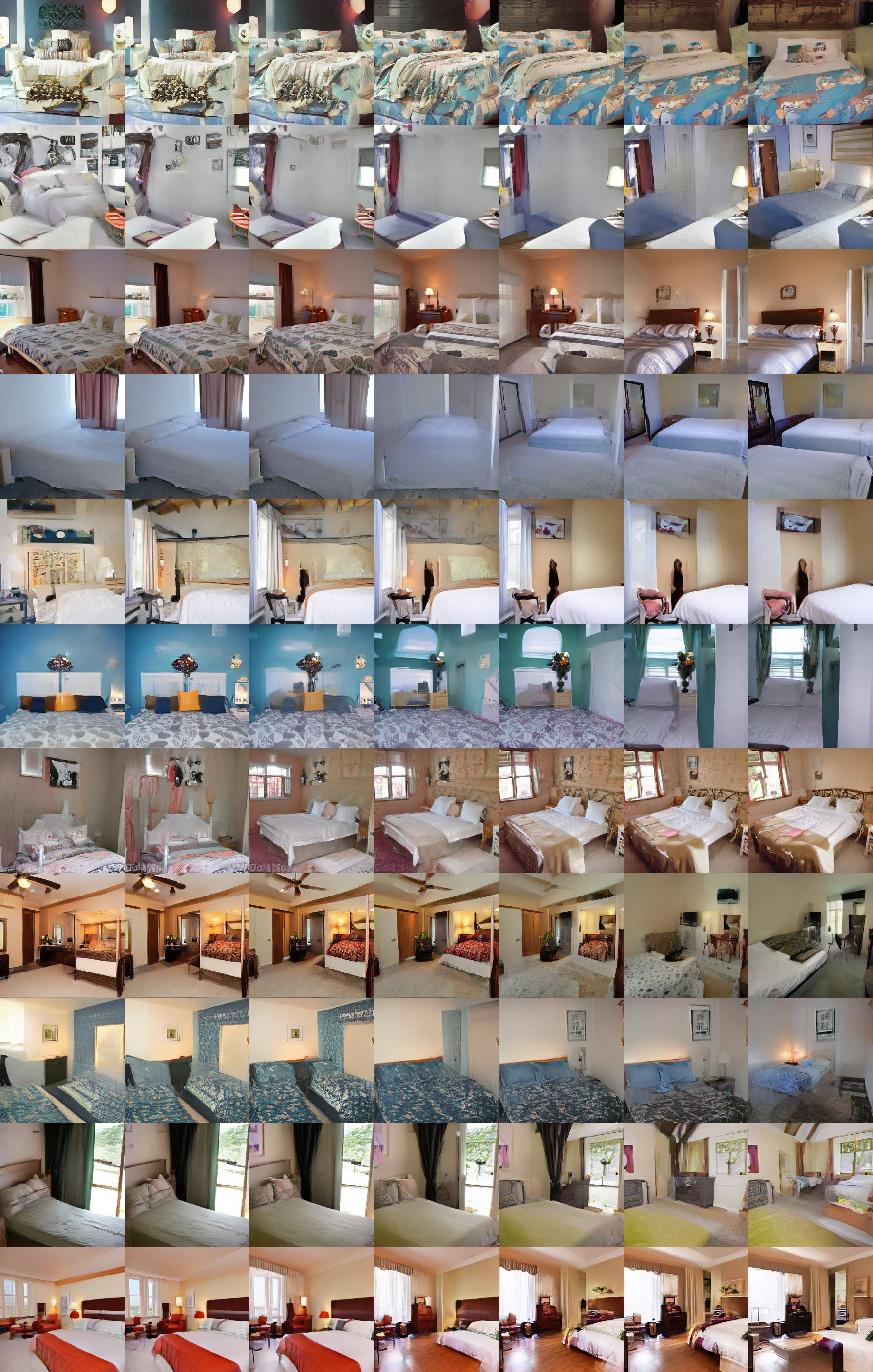}
    \end{subfigure}\hfill
    \begin{subfigure}[b]{0.11\textwidth}
        \includegraphics[width=\textwidth]{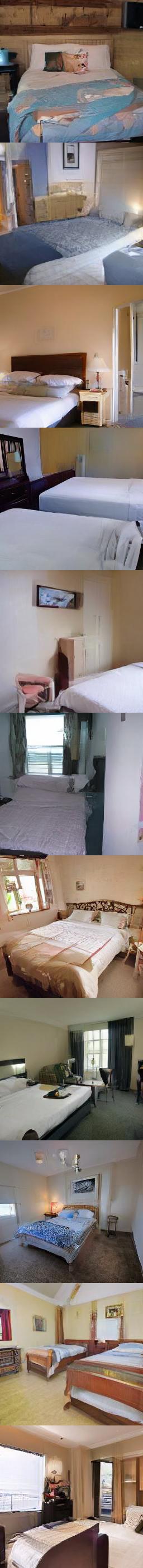}
    \end{subfigure}\hfill
    \caption{Interpolating between leftmost and rightmost images with spherical linear interpolation. All samples are generated by a consistency model trained on LSUN Bedroom $256\times 256$.}
    \label{fig:bedroom_interp}
\end{figure}

\begin{figure}
    \centering
    \begin{subfigure}[b]{0.11\textwidth}
        \includegraphics[width=\textwidth]{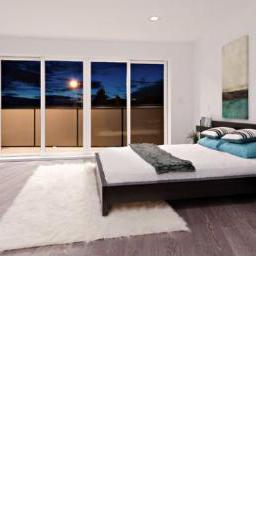}
    \end{subfigure}\hfill
    \begin{subfigure}[b]{0.88\textwidth}
        \includegraphics[width=\textwidth]{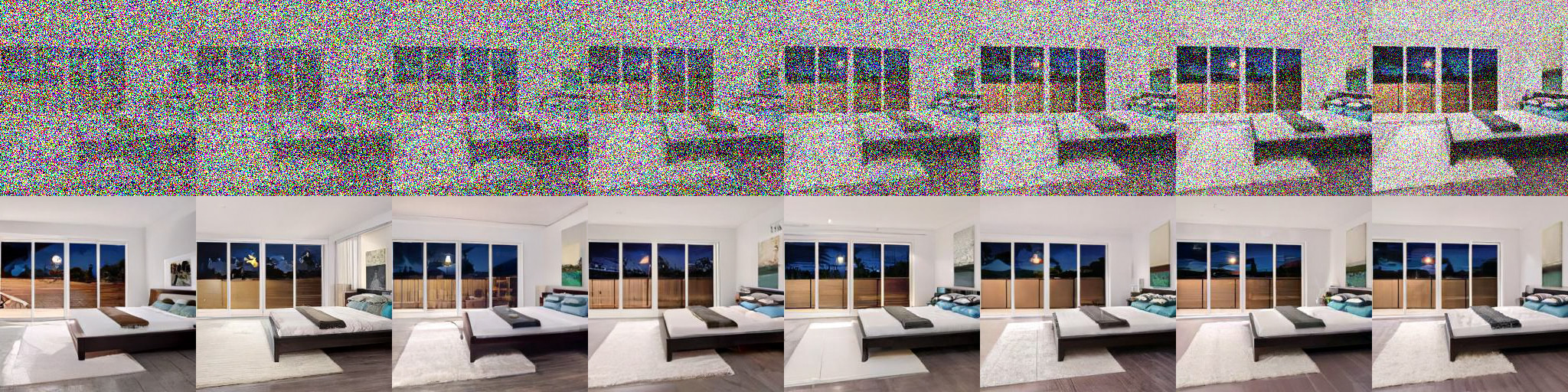}
    \end{subfigure}
    \begin{subfigure}[b]{0.11\textwidth}
        \includegraphics[width=\textwidth]{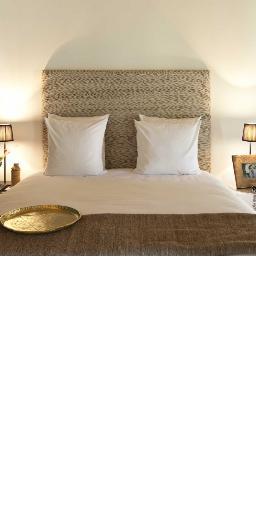}
    \end{subfigure}\hfill
    \begin{subfigure}[b]{0.88\textwidth}
        \includegraphics[width=\textwidth]{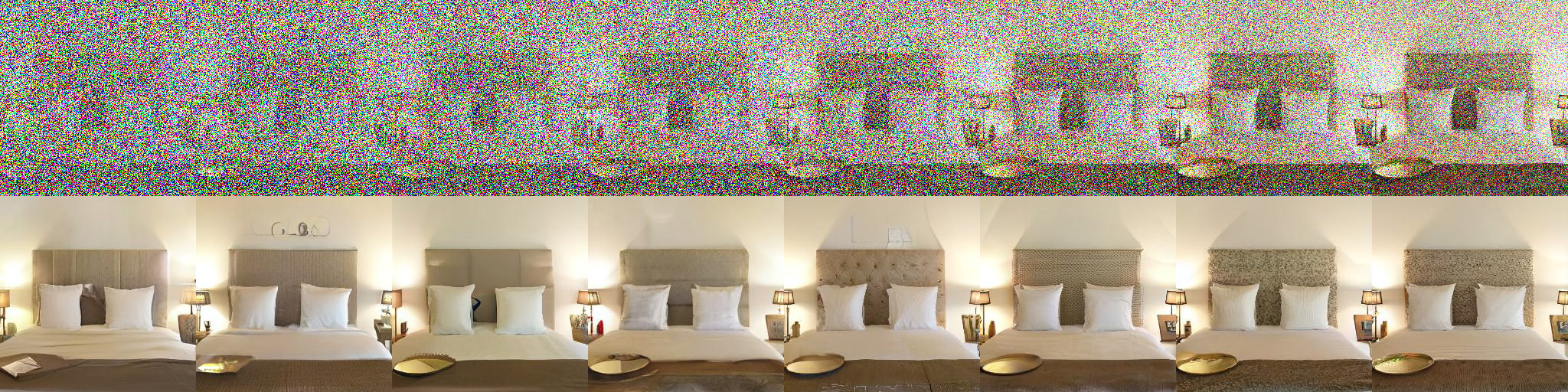}
    \end{subfigure}
    \begin{subfigure}[b]{0.11\textwidth}
        \includegraphics[width=\textwidth]{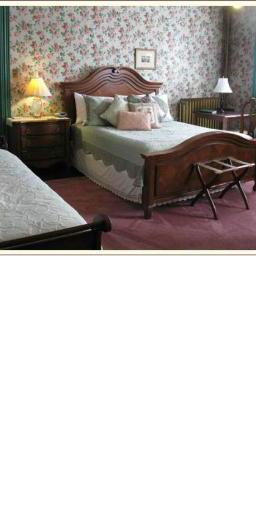}
    \end{subfigure}\hfill
    \begin{subfigure}[b]{0.88\textwidth}
        \includegraphics[width=\textwidth]{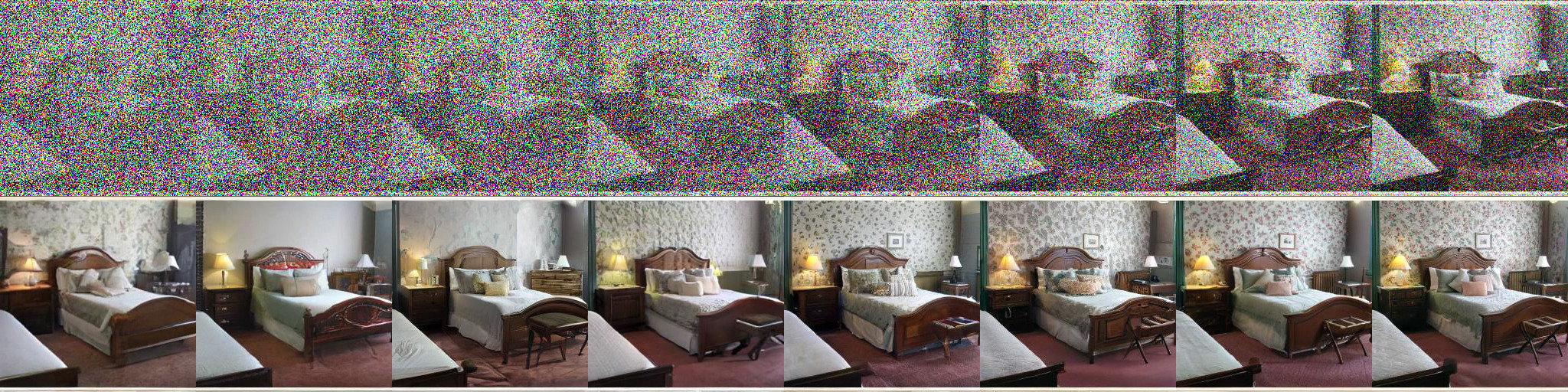}
    \end{subfigure}
    \begin{subfigure}[b]{0.11\textwidth}
        \includegraphics[width=\textwidth]{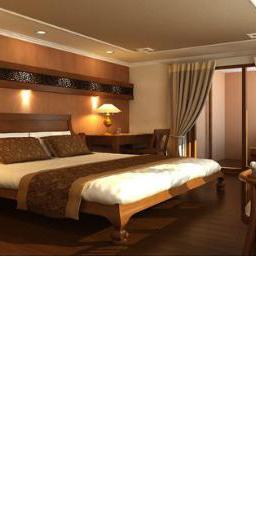}
    \end{subfigure}\hfill
    \begin{subfigure}[b]{0.88\textwidth}
        \includegraphics[width=\textwidth]{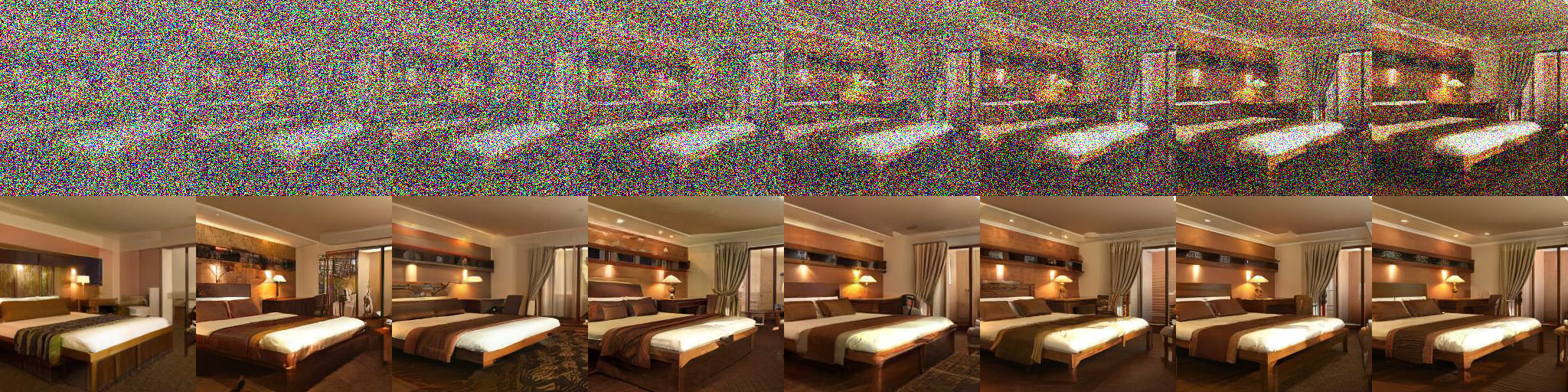}
    \end{subfigure}
    \begin{subfigure}[b]{0.11\textwidth}
        \includegraphics[width=\textwidth]{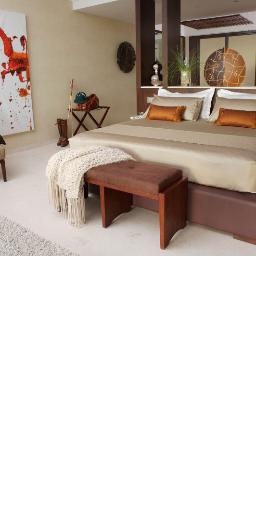}
    \end{subfigure}\hfill
    \begin{subfigure}[b]{0.88\textwidth}
        \includegraphics[width=\textwidth]{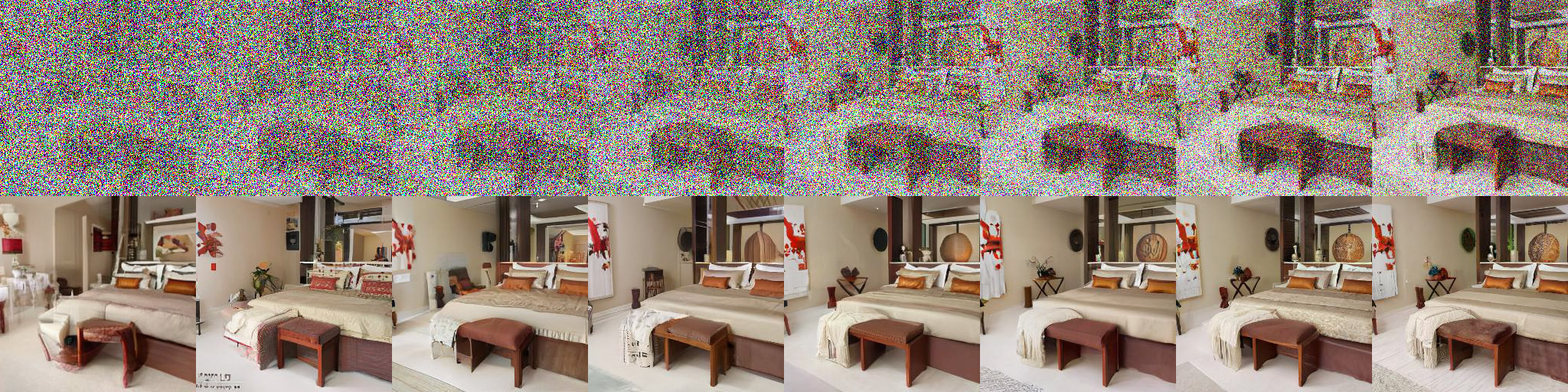}
    \end{subfigure}
    \caption{Single-step denoising with a consistency model. The leftmost images are ground truth. For every two rows, the top row shows noisy images with different noise levels, while the bottom row gives denoised images.}
    \label{fig:bedroom_denoising}
\end{figure}

\begin{figure}
    \centering
    \begin{subfigure}[b]{0.11\textwidth}
        \includegraphics[width=\textwidth]{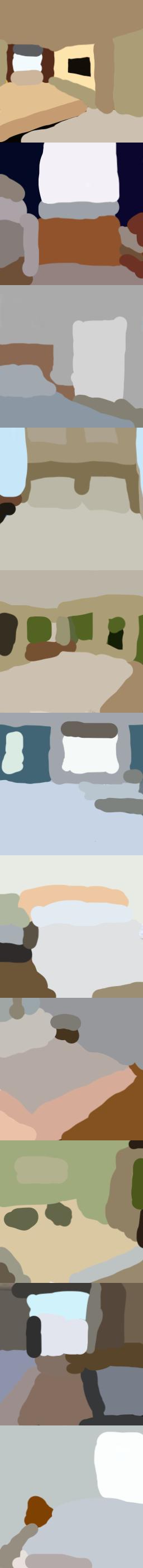}
    \end{subfigure}\hfill
    \begin{subfigure}[b]{0.88\textwidth}
        \includegraphics[width=\textwidth]{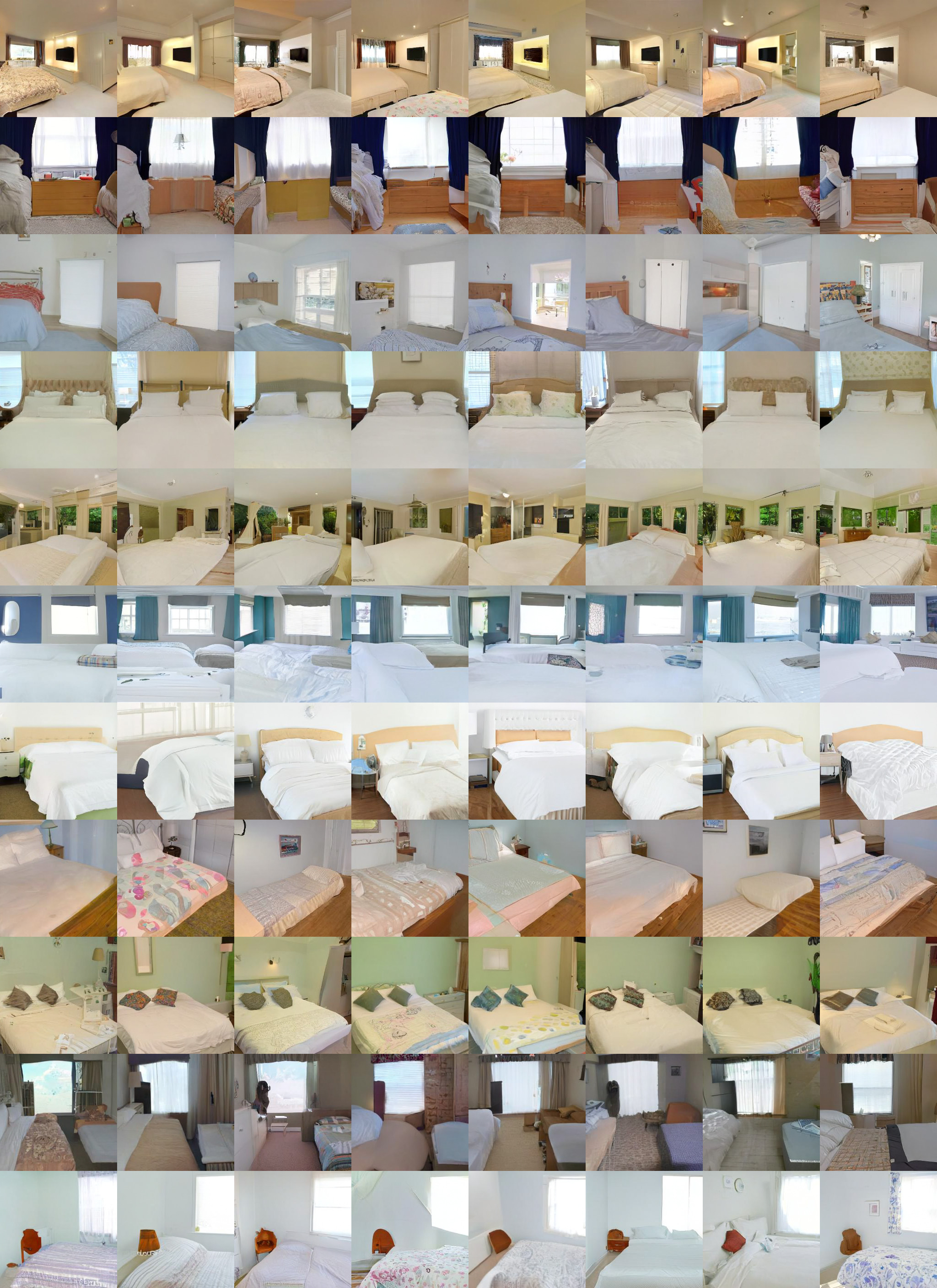}
    \end{subfigure}
    \caption{SDEdit with a consistency model. The leftmost images are stroke painting inputs. Images on the right side are the results of stroke-guided image generation (SDEdit).}
    \label{fig:bedroom_sdedit}
\end{figure}

\section{Additional Samples from Consistency Models}\label{app:samples}

We provide additional samples from consistency distillation (CD) and consistency training (CT) on CIFAR-10 (\cref{fig:cifar10_full_cd,fig:cifar10_full}), ImageNet $64\times 64$ (\cref{fig:imagenet64_full_cd,fig:imagenet64_full}), LSUN Bedroom $256\times 256$ (\cref{fig:bedroom_full_cd,fig:bedroom_full}) and LSUN Cat $256\times 256$ (\cref{fig:cat_full_cd,fig:cat_full}).

\begin{figure}
    \centering
    \begin{subfigure}[b]{\textwidth}
        \includegraphics[width=\textwidth]{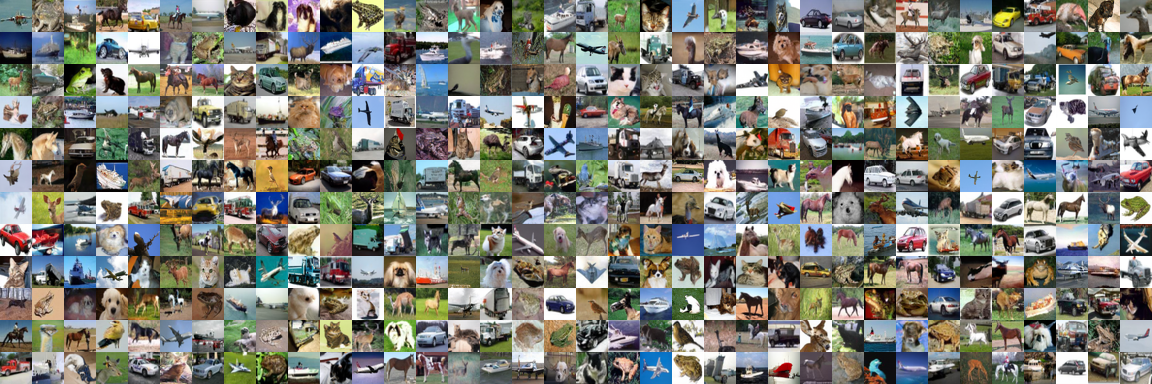}
        \caption{EDM (FID=2.04)}
    \end{subfigure}
    \begin{subfigure}[b]{\textwidth}
        \includegraphics[width=\textwidth]{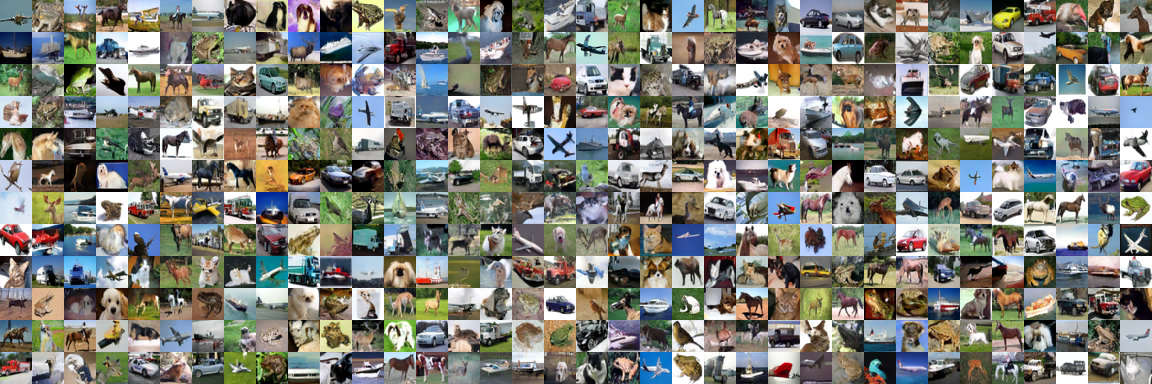}
        \caption{CD with single-step generation (FID=3.55)}
    \end{subfigure}
    \begin{subfigure}[b]{\textwidth}
        \includegraphics[width=\textwidth]{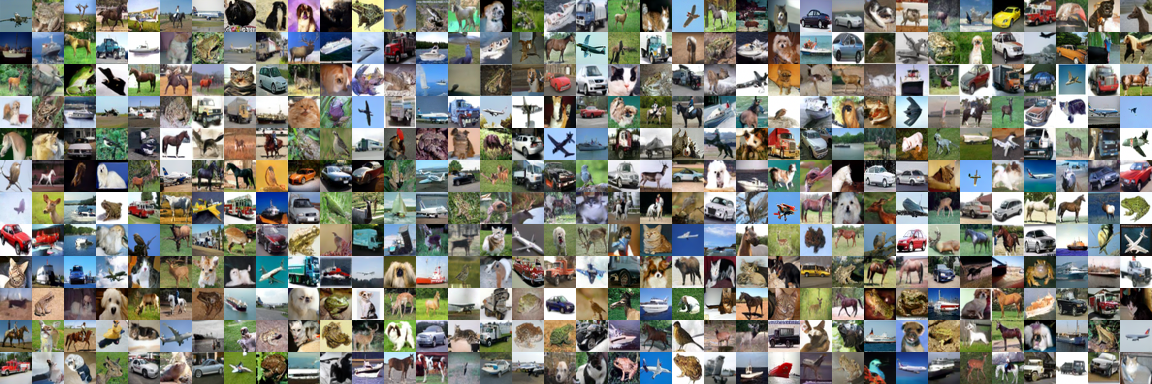}
        \caption{CD with two-step generation (FID=2.93)}
    \end{subfigure}
    \caption{Uncurated samples from CIFAR-10 $32\times 32$. All corresponding samples use the same initial noise.}
    \label{fig:cifar10_full_cd}
\end{figure}

\begin{figure}
    \centering
    \begin{subfigure}[b]{\textwidth}
        \includegraphics[width=\textwidth]{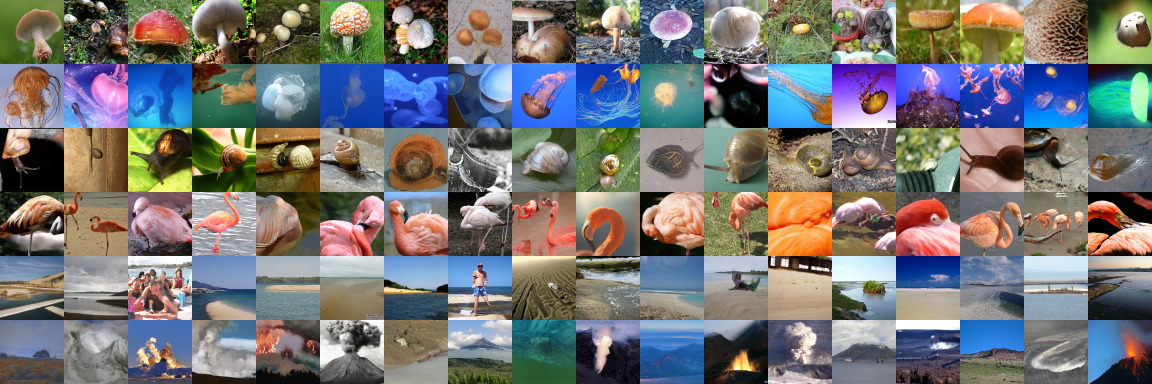}
        \caption{EDM (FID=2.44)}
    \end{subfigure}
    \begin{subfigure}[b]{\textwidth}
        \includegraphics[width=\textwidth]{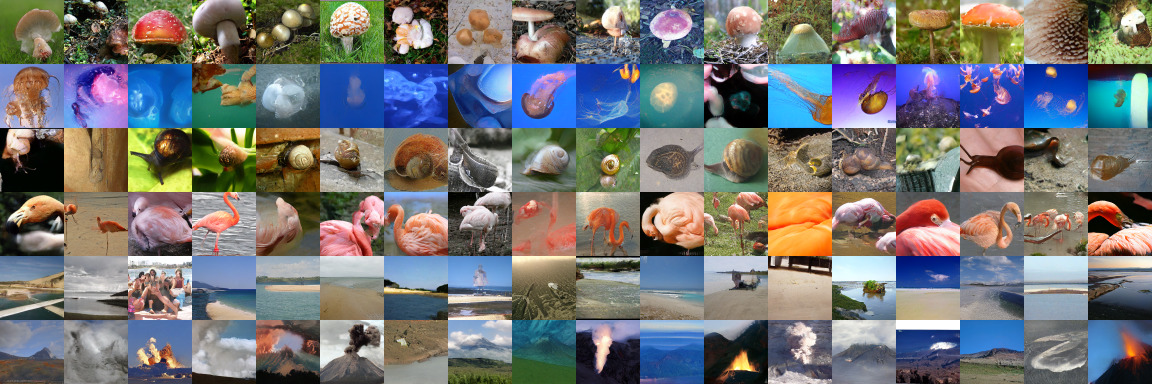}
        \caption{CD with single-step generation (FID=6.20)}
    \end{subfigure}
    \begin{subfigure}[b]{\textwidth}
        \includegraphics[width=\textwidth]{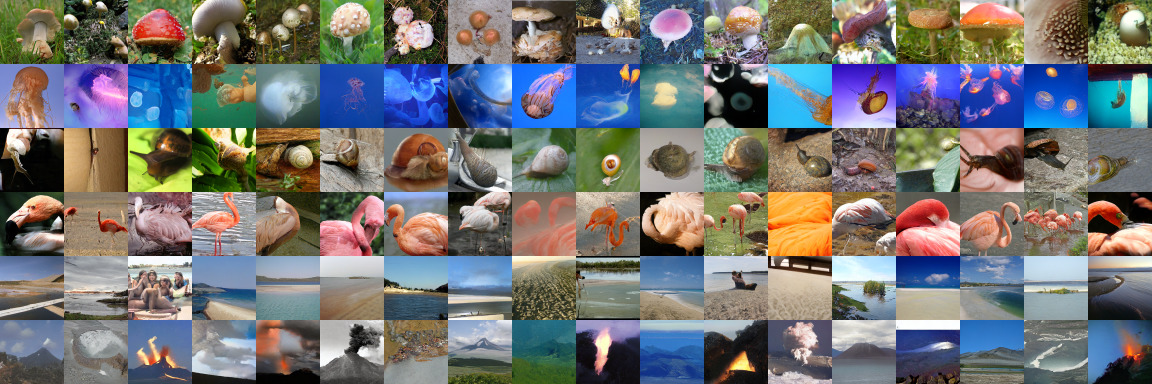}
        \caption{CD with two-step generation (FID=4.70)}
    \end{subfigure}
    \caption{Uncurated samples from ImageNet $64\times 64$. All corresponding samples use the same initial noise.}
    \label{fig:imagenet64_full_cd}
\end{figure}

\begin{figure}
    \centering
    \begin{subfigure}[b]{\textwidth}
        \includegraphics[width=\textwidth]{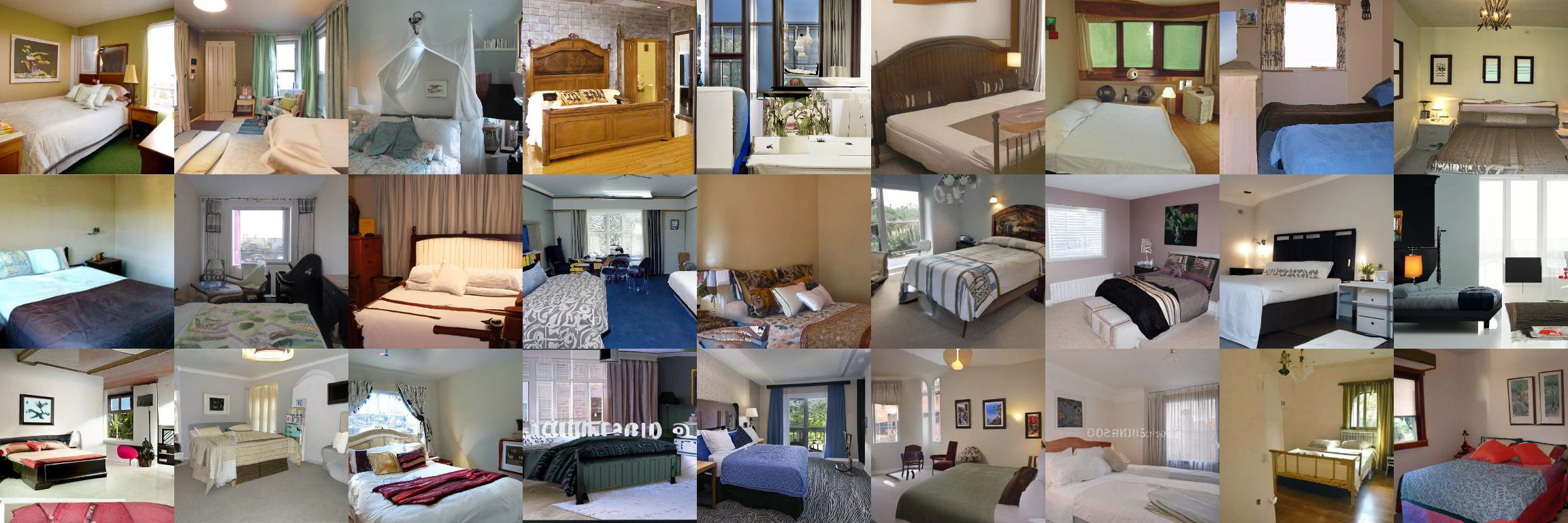}
        \caption{EDM (FID=3.57)}
    \end{subfigure}
    \begin{subfigure}[b]{\textwidth}
        \includegraphics[width=\textwidth]{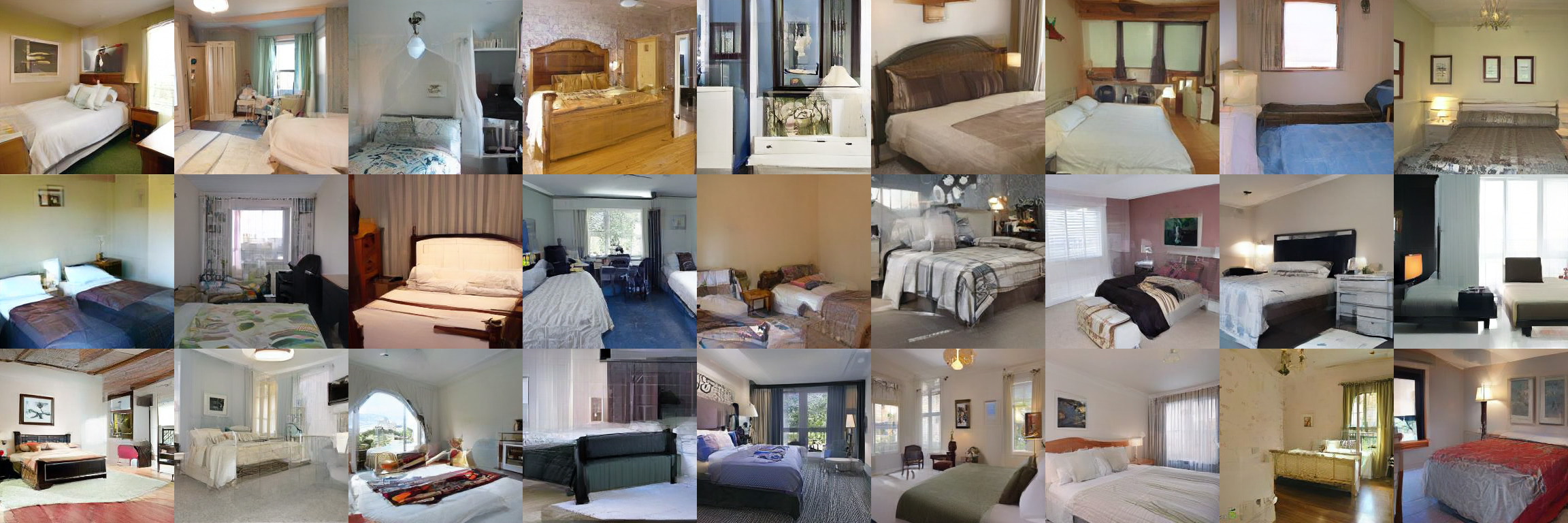}
        \caption{CD with single-step generation (FID=7.80)}
    \end{subfigure}
    \begin{subfigure}[b]{\textwidth}
        \includegraphics[width=\textwidth]{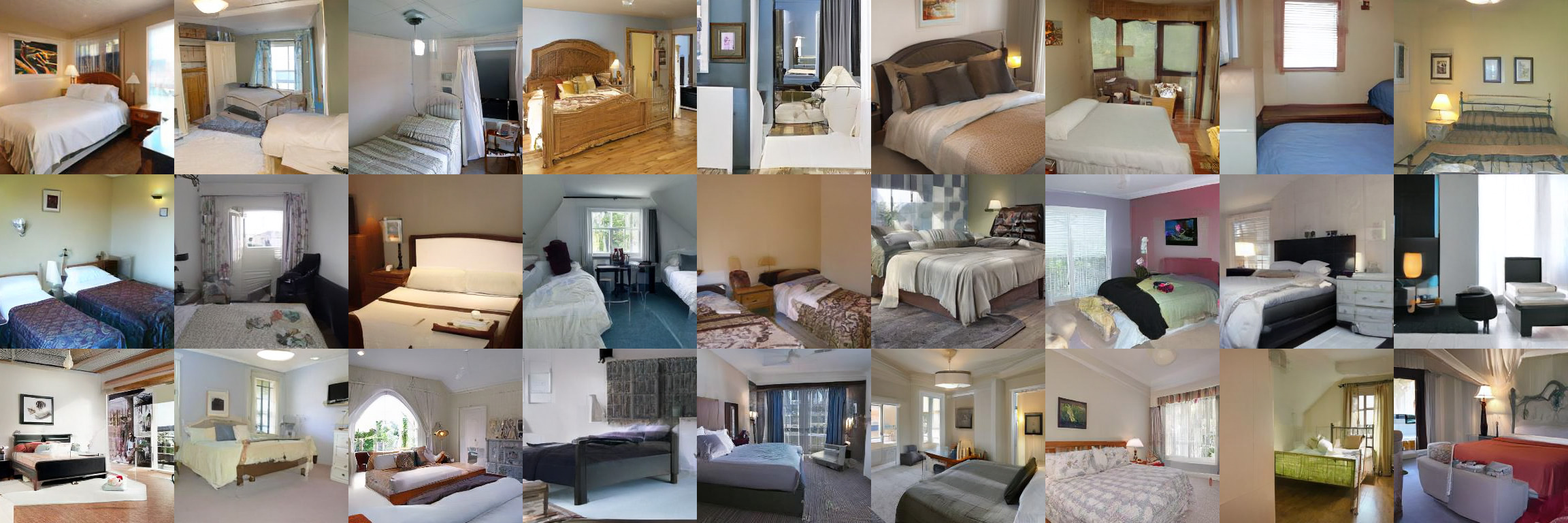}
        \caption{CD with two-step generation (FID=5.22)}
    \end{subfigure}
    \caption{Uncurated samples from LSUN Bedroom $256\times 256$. All corresponding samples use the same initial noise.}
    \label{fig:bedroom_full_cd}
\end{figure}

\begin{figure}
    \centering
    \begin{subfigure}[b]{\textwidth}
        \includegraphics[width=\textwidth]{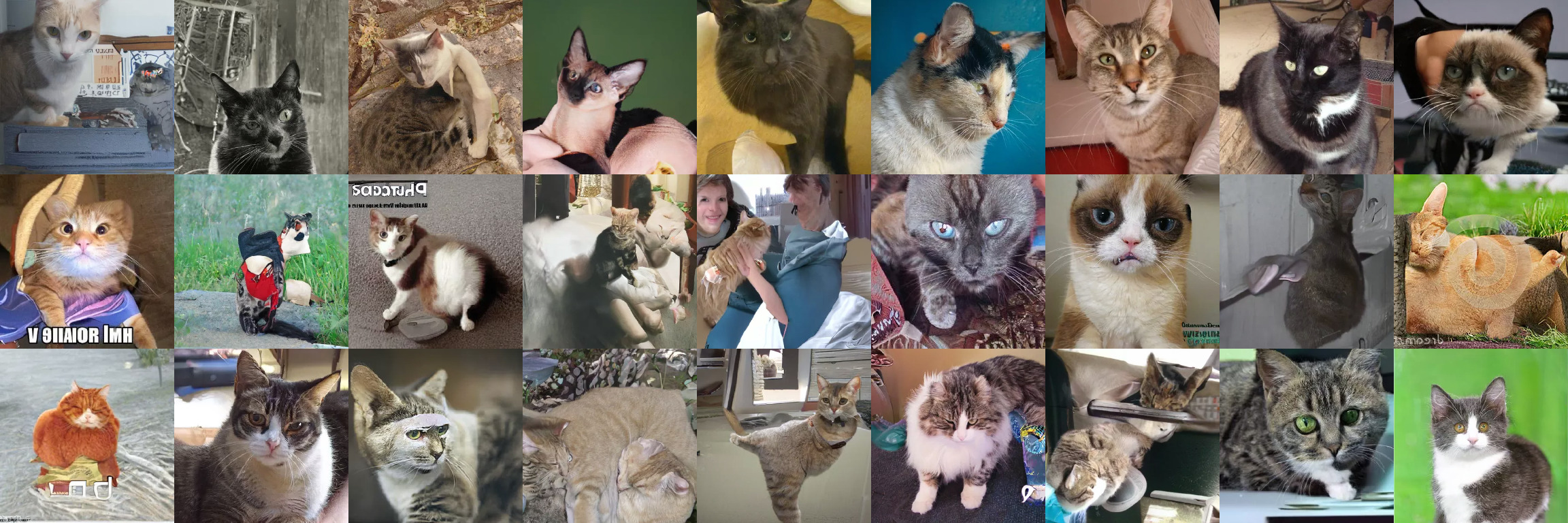}
        \caption{EDM (FID=6.69)}
    \end{subfigure}
    \begin{subfigure}[b]{\textwidth}
        \includegraphics[width=\textwidth]{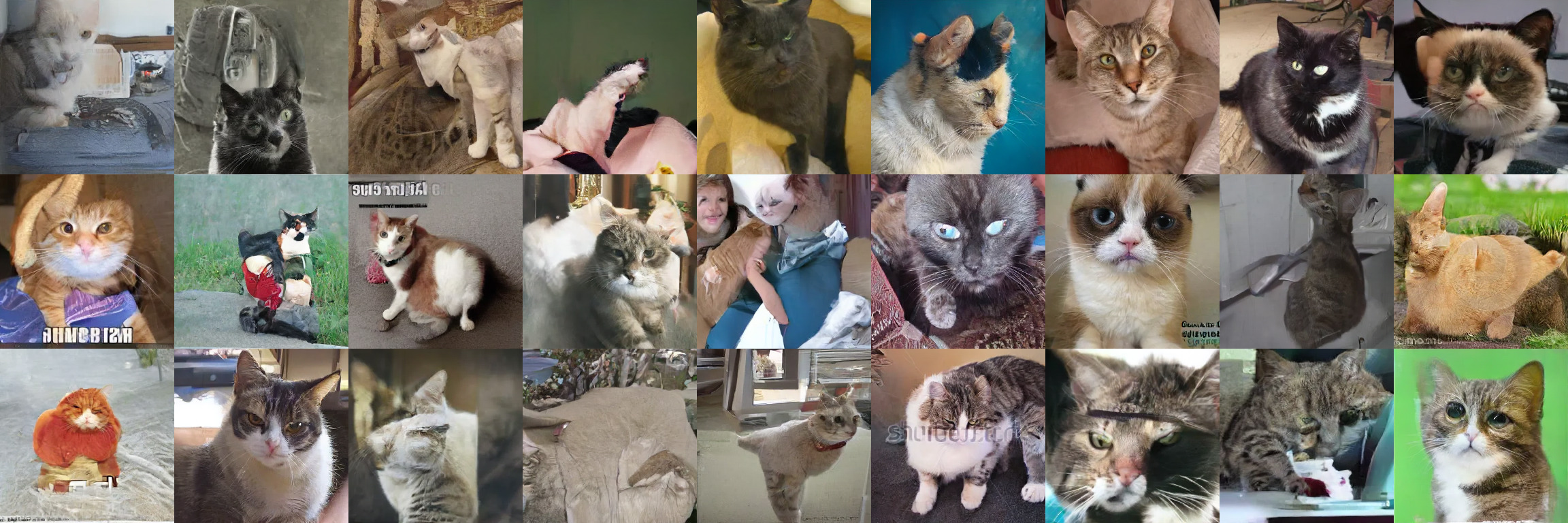}
        \caption{CD with single-step generation (FID=10.99)}
    \end{subfigure}
    \begin{subfigure}[b]{\textwidth}
        \includegraphics[width=\textwidth]{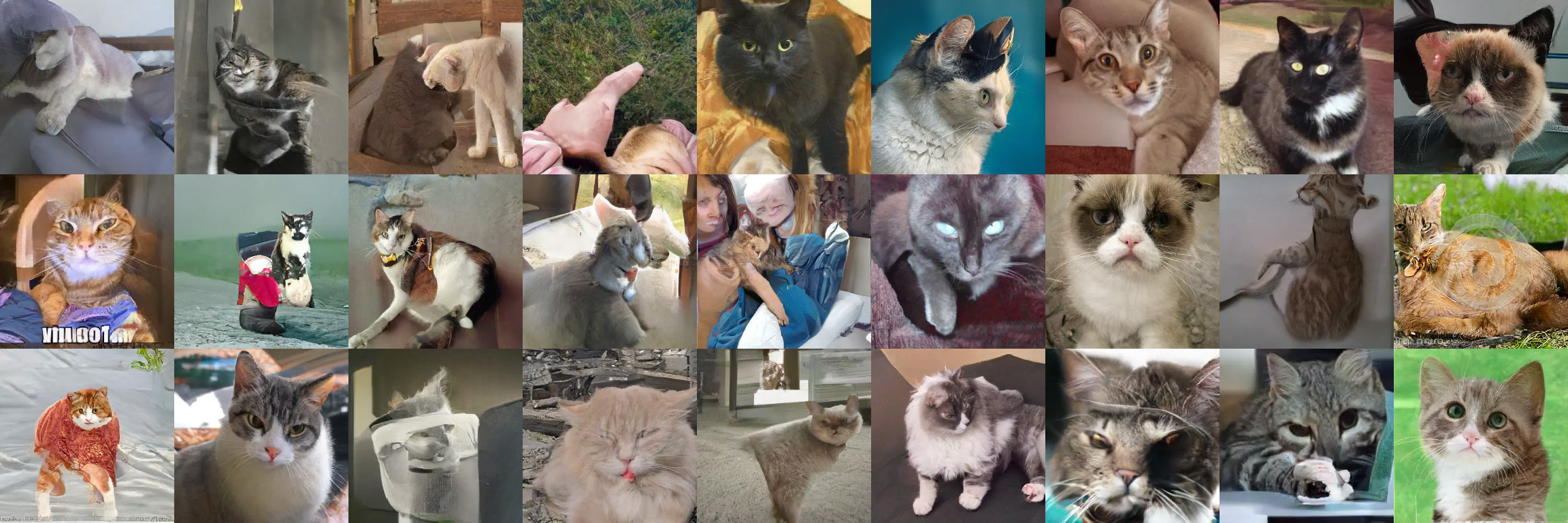}
        \caption{CD with two-step generation (FID=8.84)}
    \end{subfigure}
    \caption{Uncurated samples from LSUN Cat $256\times 256$. All corresponding samples use the same initial noise.}
    \label{fig:cat_full_cd}
\end{figure}

\begin{figure}
    \centering
    \begin{subfigure}[b]{\textwidth}
        \includegraphics[width=\textwidth]{figures/cifar10_edm_full.png}
        \caption{EDM (FID=2.04)}
    \end{subfigure}
    \begin{subfigure}[b]{\textwidth}
        \includegraphics[width=\textwidth]{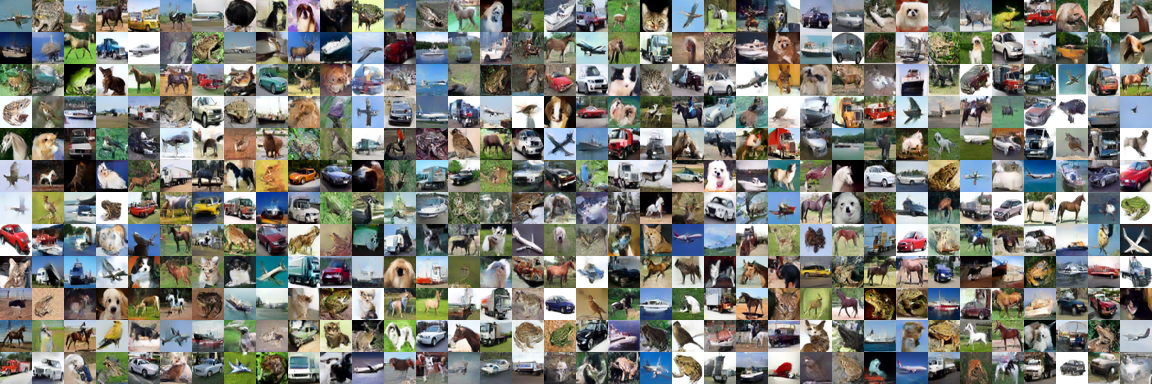}
        \caption{CT with single-step generation (FID=8.73)}
    \end{subfigure}
    \begin{subfigure}[b]{\textwidth}
        \includegraphics[width=\textwidth]{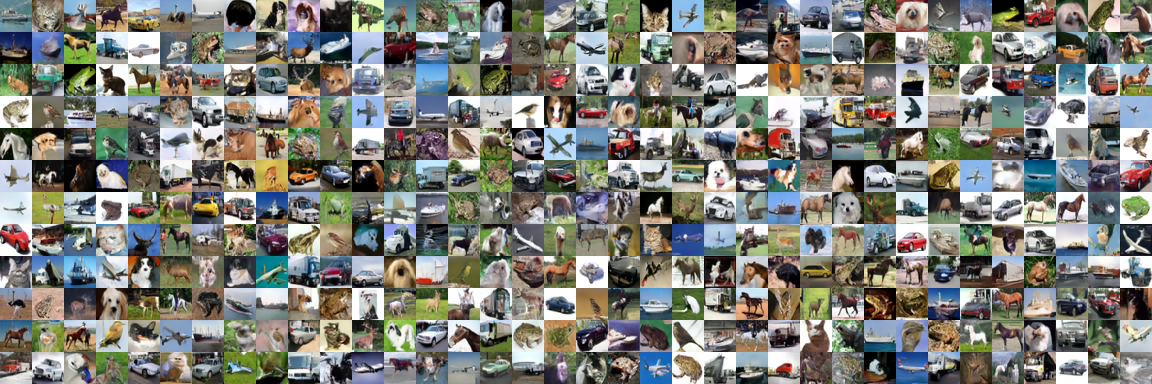}
        \caption{CT with two-step generation (FID=5.83)}
    \end{subfigure}
    \caption{Uncurated samples from CIFAR-10 $32\times 32$. All corresponding samples use the same initial noise.}
    \label{fig:cifar10_full}
\end{figure}

\begin{figure}
    \centering
    \begin{subfigure}[b]{\textwidth}
        \includegraphics[width=\textwidth]{figures/imagenet64_edm_full.jpg}
        \caption{EDM (FID=2.44)}
    \end{subfigure}
    \begin{subfigure}[b]{\textwidth}
        \includegraphics[width=\textwidth]{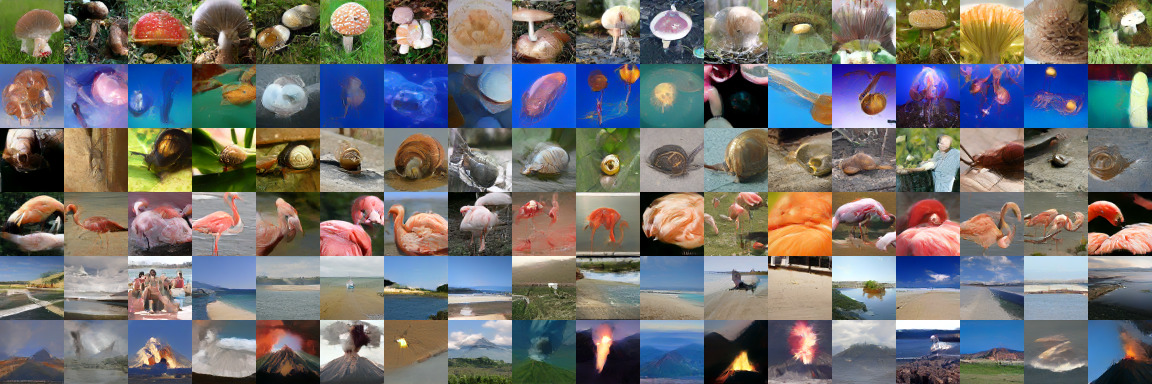}
        \caption{CT with single-step generation (FID=12.96)}
    \end{subfigure}
    \begin{subfigure}[b]{\textwidth}
        \includegraphics[width=\textwidth]{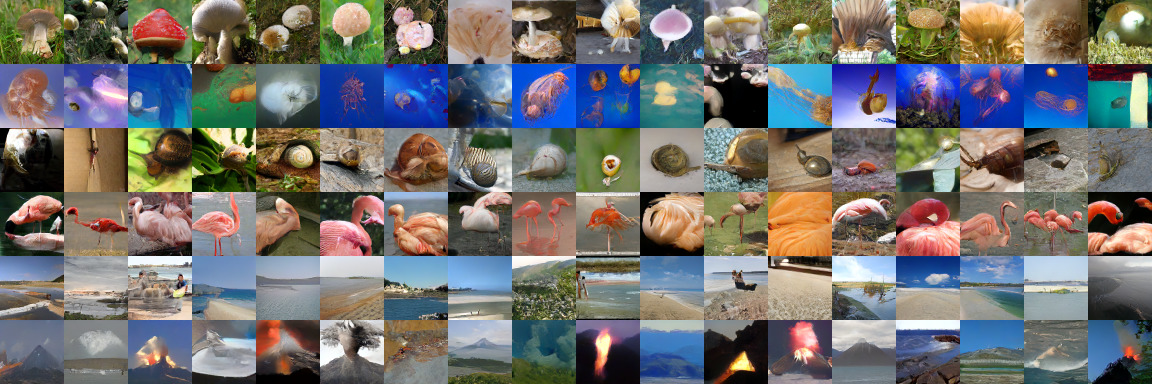}
        \caption{CT with two-step generation (FID=11.12)}
    \end{subfigure}
    \caption{Uncurated samples from ImageNet $64\times 64$. All corresponding samples use the same initial noise.}
    \label{fig:imagenet64_full}
\end{figure}

\begin{figure}
    \centering
    \begin{subfigure}[b]{\textwidth}
        \includegraphics[width=\textwidth]{figures/bedroom_edm_full.jpg}
        \caption{EDM (FID=3.57)}
    \end{subfigure}
    \begin{subfigure}[b]{\textwidth}
        \includegraphics[width=\textwidth]{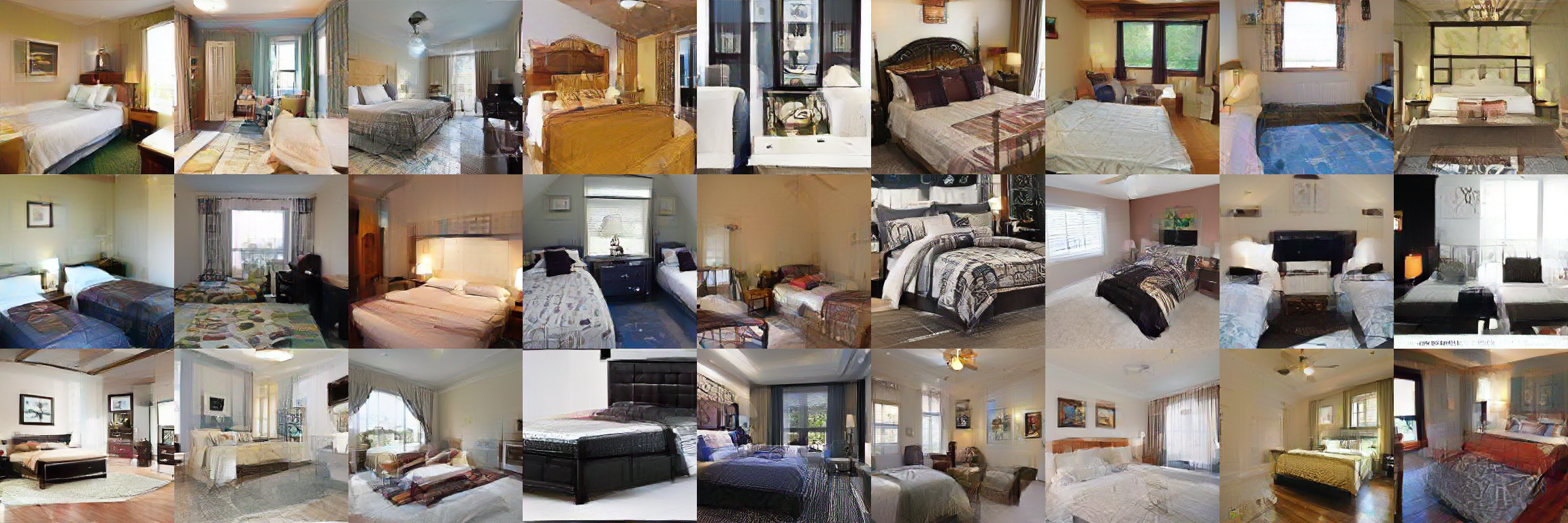}
        \caption{CT with single-step generation (FID=16.00)}
    \end{subfigure}
    \begin{subfigure}[b]{\textwidth}
        \includegraphics[width=\textwidth]{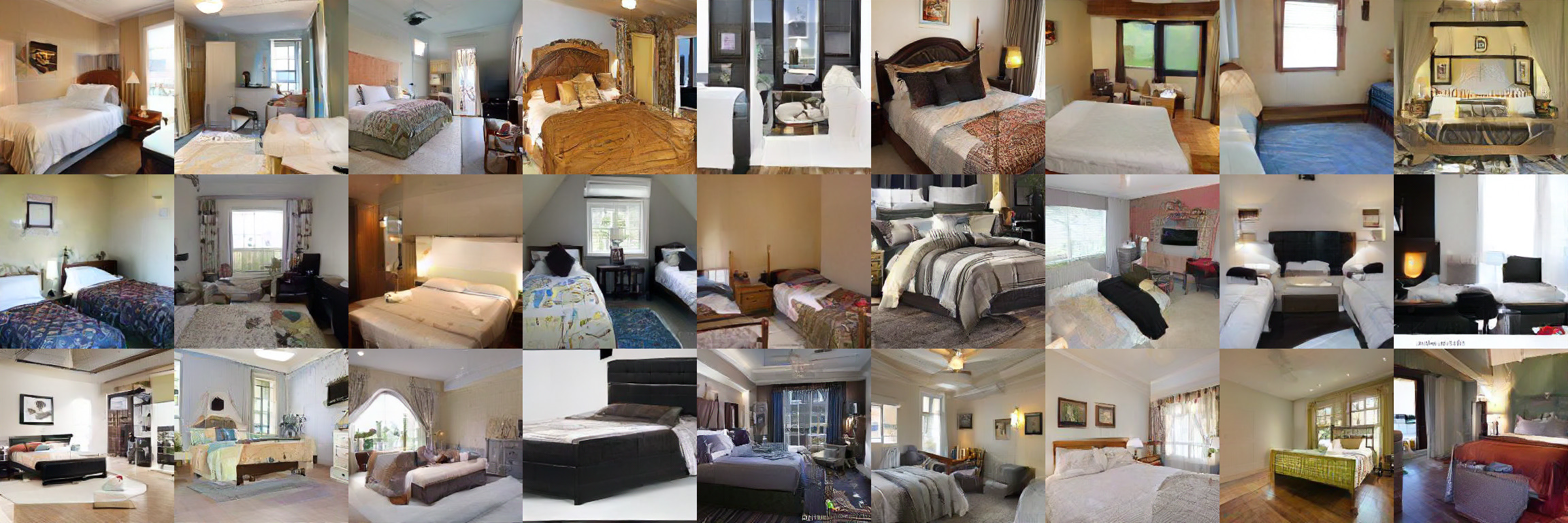}
        \caption{CT with two-step generation (FID=7.80)}
    \end{subfigure}
    \caption{Uncurated samples from LSUN Bedroom $256\times 256$. All corresponding samples use the same initial noise.}
    \label{fig:bedroom_full}
\end{figure}

\begin{figure}
    \centering
    \begin{subfigure}[b]{\textwidth}
        \includegraphics[width=\textwidth]{figures/cat_edm_full_new.jpg}
        \caption{EDM (FID=6.69)}
    \end{subfigure}
    \begin{subfigure}[b]{\textwidth}
        \includegraphics[width=\textwidth]{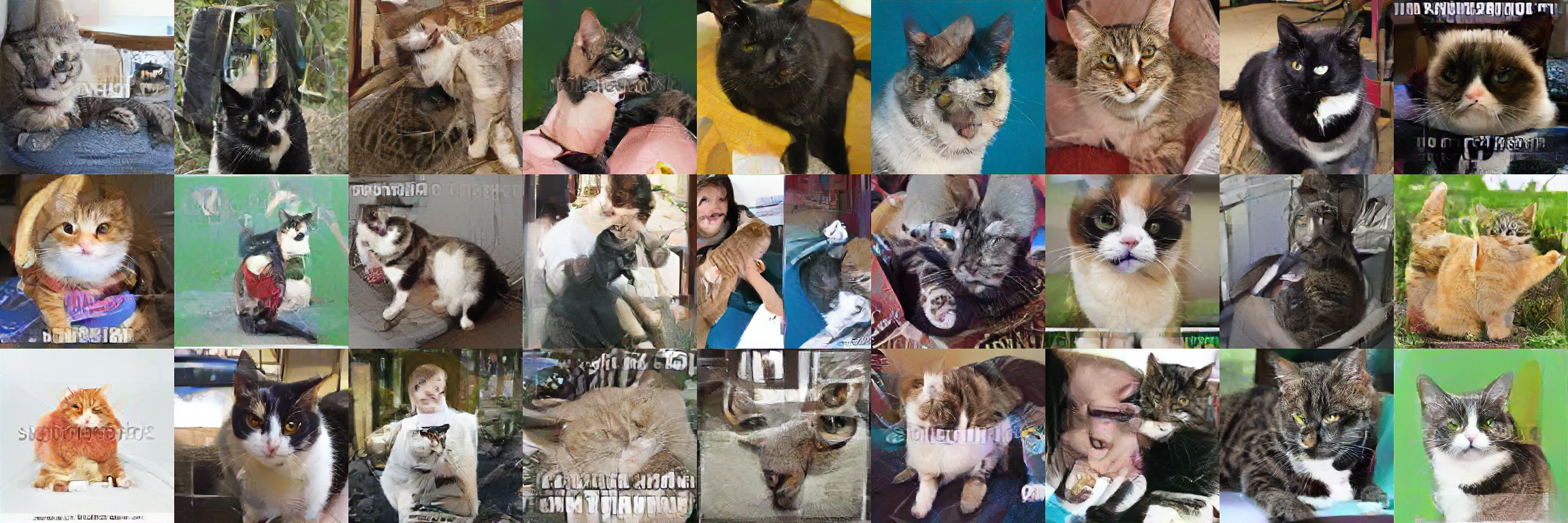}
        \caption{CT with single-step generation (FID=20.70)}
    \end{subfigure}
    \begin{subfigure}[b]{\textwidth}
        \includegraphics[width=\textwidth]{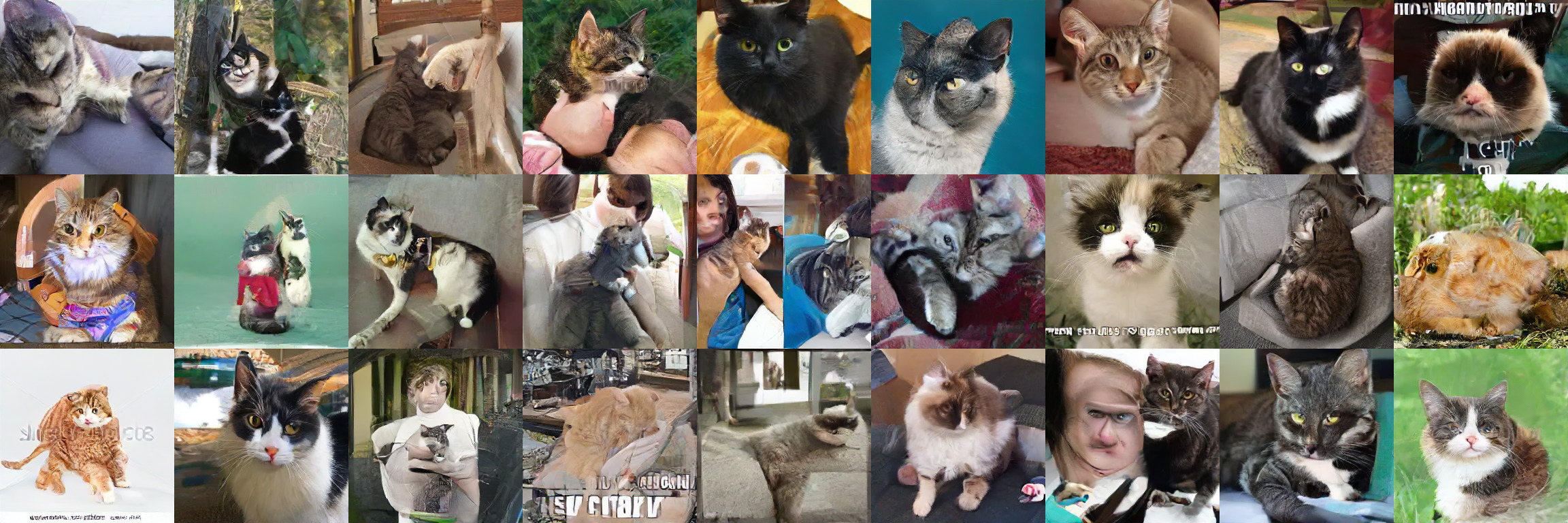}
        \caption{CT with two-step generation (FID=11.76)}
    \end{subfigure}
    \caption{Uncurated samples from LSUN Cat $256\times 256$. All corresponding samples use the same initial noise.}
    \label{fig:cat_full}
\end{figure}

\end{appendices}

\end{document}